\documentclass[twoside]{article}

\usepackage[accepted]{aistats2025}

\usepackage{newtxtext}
\usepackage{amssymb,amsthm} 
\usepackage{amsmath}
\usepackage{graphicx}

\usepackage{algorithm}
\usepackage{algpseudocode}
\usepackage{bbm}
\usepackage{hyperref}

\usepackage{tikz}
\usetikzlibrary{calc}

\newtheorem{theorem}{Theorem}
\newtheorem{lemma}[theorem]{Lemma}

\newtheorem{proposition}[theorem]{Proposition}
\newtheorem{definition}[theorem]{Definition}

\theoremstyle{remark}
\newtheorem{remark}{Remark}
\newtheorem{assumption}{Assumption}

\usepackage{thmtools} 
\usepackage{thm-restate}

\def\uu{f}
\def\xx{X}
\def\vv{V}

\def\1{\mathds{1}}

\DeclareMathOperator{\R}{\mathbb{R}}

\newcommand{\E}{\mathbb{E}}

\def\U{\mathcal{U}}
\def\B{\mathcal{B}}

\usepackage[round]{natbib}
\begin{document}

% If your paper is accepted and the title of your paper is very long,
% the style will print as headings an error message. Use the following
% command to supply a shorter title of your paper so that it can be
% used as headings.
%
%\runningtitle{I use this title instead because the last one was very long}

% If your paper is accepted and the number of authors is large, the
% style will print as headings an error message. Use the following
% command to supply a shorter version of the authors names so that
% they can be used as headings (for example, use only the surnames)
%
%\runningauthor{Surname 1, Surname 2, Surname 3, ...., Surname n}

\twocolumn[

\aistatstitle{Sampling from Bayesian Neural Network Posteriors with Symmetric Minibatch Splitting Langevin Dynamics}

\aistatsauthor{Daniel Paulin$^{\textbf{*}}$  \And  Peter A. Whalley$^{\textbf{*}}$ \And Neil K. Chada \And Benedict Leimkuhler}

\aistatsaddress{Nanyang Technological U. \And ETH Z\"{u}rich \And City University of Hong Kong  \And  University of Edinburgh} ]

\begin{abstract}
We propose a scalable kinetic Langevin dynamics algorithm for sampling parameter spaces of big data and AI applications.  Our scheme combines a symmetric forward/backward sweep over minibatches with a symmetric discretization of Langevin dynamics. For a particular Langevin splitting method (UBU), we show that the resulting Symmetric Minibatch Splitting-UBU (SMS-UBU) integrator has bias $\mathcal{O}(h^2 d^{1/2})$ in dimension $d>0$ with stepsize $h>0$, despite only using one minibatch per iteration, thus providing excellent control of the sampling bias as a function of the stepsize. We apply the algorithm to explore local modes of the posterior distribution of Bayesian neural networks (BNNs) and evaluate the calibration performance of the posterior predictive probabilities for neural networks with convolutional neural network architectures for classification problems on three different datasets (Fashion-MNIST, Celeb-A and chest X-ray). Our results indicate that BNNs sampled with SMS-UBU can offer significantly better calibration performance compared to standard methods of training and stochastic weight averaging.
\end{abstract}

\section{INTRODUCTION}
Bayesian neural networks (BNNs) are a statistical paradigm which has been widely advocated to improve the reliability of neural networks by formulating questions regarding parameter sensitivity and prediction error in terms of the structure of the posterior parameter distribution.  The ultimate goals of BNNs (already espoused in the seminal works of \cite{Neal} and \cite{mackay}) are to reduce generalization error and enhance robustness of trained models.  Due to the capacity of BNNs to quantify uncertainty, they have been proposed for use in applications such as astronomy, and modelling partial differential equations (\cite{astronomy}, \cite{Zabaras}), but their potential impact is much greater and includes topics in data-driven engineering (\cite{driving2}, \cite{driving}) and automated medical diagnosis (\cite{diabetic}, \cite{cancer}).

While reliance on a Bayesian statistics approach provides
theoretical foundation for BNNs, the high cost of treating these models has hampered uptake in practical applications.  In the modern era, datasets are often very large--a single epoch (pass through the training data) with computation of likelihoods for a large scale network may take hours or days (\cite{scale}). Optimization schemes can robustly use efficient data subsampling, whereas  BNNs require a more delicate approach, since the target distribution of the Markov chain is easily corrupted (\cite{SGHMC}, \cite{HMCBNN}, \cite{Zhang}). Scalable alternatives to MCMC based on Gaussian approximation can lead to some improvements in calibration (\cite{maddox2019simple}, \cite{blundell2015weight}), but they lack theoretical justification.

The high computational cost of deep learning is also a consequence of the size of the model parameter space, which may consist of many millions (or even billions) of parameters.  Sampling the parameters of such models undeniably adds substantial overload in comparison to optimization, which is already seen as costly,  thus a central impediment to exploiting the promise of BNNs is the need for optimal sampling strategies that scale well both in terms of the size of the dataset and parameter space dimension (\cite{rank1}, \cite{ADAM}).   The most popular schemes for high dimensional inference are Markov chain Monte Carlo (MCMC) methods
(\cite{andrieu}), e.g. Hamiltonian Monte Carlo (HMC) and the Metropolis-adjusted Langevin algorithm (MALA) (\cite{talldata}, \cite{SGHMC}, \cite{gouraud2022hmc}), with HMC typically favored in the large scale BNN setting (\cite{BNNpost}, \cite{theo2}).

In this article, we present the use of an ``unadjusted'' sampling scheme, within BNNs, which avoids the costly Metropolis test and is based on the use of efficient symmetric splitting methods for kinetic Langevin dynamics, in particular the UBU discretization, providing second order accuracy in the strong (pathwise) sense and in terms of sampling bias.
We show that second order accuracy can be maintained in combination with a carefully organized minibatch subsampling strategy resulting in a fully symmetric (SMS-UBU) integrator, which provides both efficiency for large datasets and low bias, asymptotically, at high dimension.  We implement this scheme for BNNs and demonstrate the performance in comparison with some alternative methods. Numerical tests are conducted using suitable NNs on a range of classification problems involving image data (the fashion-MNIST dataset (\cite{xiao2017fashion}), celeb-A dataset (\cite{yang2015facial}) and a chest X-ray dataset (\cite{kermany2018identifying})). 
To assess the results, we rely on well-known metrics, namely accuracy, negative log-likelihood on test data, adaptive calibration error, and ranked probability score. Our experiments clearly demonstrate the effectiveness of BNNs trained with SMS-UBU compared to standard training and stochastic weight averaging. Our code is available at \url{github.com/paulindani/SMS_Kinetic_Langevin}.

The main contributions of this work are:
(i) the presentation of a kinetic Langevin scheme incorporating symmetric minibatch subsampling (SMS-UBU) and the demonstration of its second order accuracy, the first result of its kind for a stochastic gradient scheme,  (ii) novel bias bounds for vanilla stochastic gradient kinetic Langevin dynamics as a function of dataset size, considering also the control of the bias for minibatches of the dataset drawn without replacement, and (iii) demonstration of the use of a Bayesian \textcolor{black}{uncertainty quantification (UQ)} framework for NNs by applying our efficient methods to sample BNNs near targeted posterior modes. %\textcolor{black}{By Bayesian UQ, we refer to the quantification of uncertainty that can arise, for which instead of a point estimator, we consider a posterior distribution instead.} 
In Bayesian uncertainty quantification, we estimate the uncertainties in the predictions by sampling from the Bayesian posterior distribution of the model's parameters. 

%perhaps we could remove this paragraph below. The paper is only 8 pages.
The results are organized as follows:  Section \ref{sec:method} gives an  overview of kinetic Langevin integrators, while the SMS-UBU schemes appears in  Section \ref{sec:sms-ubu} along with a summary of the key theoretical results. Our main theorem is given in Section \ref{sec:main}. Section \ref{sec:uq} takes up  discussion of \textcolor{black}{UQ} in neural networks, followed by the presentation of classification experiments in Section \ref{sec:num}. We conclude in Section \ref{sec:conc}. Detailed proofs are provided in the supplement, along with additional details about the experiments.

\section{Methodology}
In this section, we first review sampling using kinetic Langevin integrators. We begin with a brief overview, before discussing a particular splitting scheme, (UBU). We will then discuss the extension of UBU to stochastic gradients before introducing a new method referred to as SMS-UBU, which is presented in algorithmic form.
\label{sec:method}
\subsection{Kinetic Langevin dynamics and discretizations}
\label{sec:integrators}

Denote by $\pi(d x)$ the target probability measure of form $\pi(dx)\propto \exp(-\uu(x))dx$,
where $\uu(x)$ refers to the potential energy function.
Kinetic Langevin dynamics is the system
\begin{equation}\label{eq:LD}
    \begin{split}
    d\xx_{t} &= \vv_{t}dt,\\
    d\vv_{t} &= -\nabla \uu(\xx_{t}) dt - \gamma \vv_{t}dt + \sqrt{2\gamma}dW_{t},
\end{split}
\end{equation}
on $\mathbb{R}^{2d}$, where $\{W_t\}_{t\geq 0}$ is a standard $d-$dimensional Wiener process, and $\gamma>0$ is a friction coefficient. Under %fairly weak 
standard assumptions (see \cite{pavliotis2014stochastic}), the unique invariant measure of the process $\{X_t,V_t\}_{t\geq 0}$ has density  given by
\begin{align}
\label{eq:inv}
&\bar{\pi}(dx,dv) \propto \exp\left(-\uu(x)-\frac{ \|v\|^2}{2}\right)dx dv
\\ &=\pi(dx) \exp\left (-\frac{ \|v\|^2}{2}\right) dv, \nonumber
\end{align}
with respect to Lebesgue measure.
Due to the product form of the invariant measure, averages with respect to the target measure $\pi(dx)$ can be obtained as the configurations generated by approximating paths of \eqref{eq:LD}.

Bias in the invariant measure arises, principally, due to two factors: (i) finite stepsize discretization of \eqref{eq:LD}, and (ii) the use of stochastic gradients. The choice of integrator, the stepsize, and the friction $\gamma$ all affect both the convergence rate and the discretization bias (\cite{BoOw10}, \cite{LeMa13}, \cite{gouraud2022hmc}), and this is further altered by the procedure used to estimate the gradient.    While the simplest procedure is to use the Euler-Mayurama scheme for the overdamped form of \eqref{eq:LD}, together with stochastic gradients (SG), as in the SG-HMC algorithm of \cite{SGHMC}, there has been significant progress in the numerical analysis community in the development of accurate second order integrators for  kinetic Langevin dynamics \eqref{eq:LD}, see \cite{LeMa13}, \cite{sanz2021wasserstein} and \cite{BUBthesis}. Here we extend some of these integrators via the use of stochastic gradients.

An accurate splitting method was introduced in \cite{BUBthesis} and further studied in \cite{sanz2021wasserstein} and \cite{ububu}. This splitting method only requires one gradient evaluation per iteration but has strong order two. The method is based on breaking up the SDE \eqref{eq:LD} as follows 
\[
\begin{pmatrix}
dX_{t} \\
dV_{t}
\end{pmatrix} = \underbrace{\begin{pmatrix}
0 \\
-\nabla \uu(X_{t})dt
\end{pmatrix}}_{\mathcal{B}} +\underbrace{\begin{pmatrix}
V_{t}dt \\
-\gamma V_{t} dt + \sqrt{2\gamma}dW_t
\end{pmatrix}}_{\mathcal{U}},
\]
where each part can be integrated exactly over a step of size $h$. Given $\gamma > 0$, let $\eta = \exp{\left(-\gamma h/2\right)}$, and for ease of notation, define the following operators 
\begin{equation}\label{eq:Bdef}
\mathcal{B}(x,v,h) = (x,v - h\nabla \uu(x)),
\end{equation}
and
%{\tiny
\begin{align}\label{eq:Udef}
%\begin{split}
&\mathcal{U}(x,v,h/2,\xi^{(1)},\xi^{(2)}) =\Big(x + \frac{1-\eta}{\gamma}v\\
\nonumber
& + \sqrt{\frac{2}{\gamma}}\left(\mathcal{Z}^{(1)}\left(h/2,\xi^{(1)}\right) - \mathcal{Z}^{(2)}\left(h/2,\xi^{(1)},\xi^{(2)}\right) \right),\\
\nonumber
& \eta v + \sqrt{2\gamma}\mathcal{Z}^{(2)}\left(h/2,\xi^{(1)},\xi^{(2)}\right)\Big), \text{ where}\\
\label{eq:Z1def}&\mathcal{Z}^{(1)}\left(h/2,\xi^{(1)}\right) = \sqrt{\frac{h}{2}}\xi^{(1)},\\
\label{eq:Z2def}&\mathcal{Z}^{(2)}\left(h/2,\xi^{(1)},\xi^{(2)}\right) = \\
\nonumber&\sqrt{\frac{1-\eta^{2}}{2\gamma}}\Bigg(\sqrt{\frac{1-\eta}{1+\eta}\cdot \frac{4}{\gamma h}}\xi^{(1)} + \sqrt{1-\frac{1-\eta}{1+\eta}\cdot\frac{4}{\gamma h}}\xi^{(2)}\Bigg),
\end{align}
where $\xi^{(1)},\xi^{(2)} \sim \mathcal{N}(0_{d},I_{d})$ are $d$-dimensional standard Gaussian random variables.
%}

%
%\end{split}
%\end{equation}

% \begin{equation}\label{eq:Z12def}
% \begin{split}
% \mathcal{Z}^{(1)}\left(h/2,\xi^{(1)}\right) &= \sqrt{\frac{h}{2}}\xi^{(1)},\\
% \mathcal{Z}^{(2)}\left(h/2,\xi^{(1)},\xi^{(2)}\right) &= \sqrt{\frac{1-\eta^{2}}{2\gamma}}\Bigg(\sqrt{\frac{1-\eta}{1+\eta}\cdot \frac{4}{\gamma h}}\xi^{(1)} + \sqrt{1-\frac{1-\eta}{1+\eta}\cdot\frac{4}{\gamma h}}\xi^{(2)}\Bigg).
% \end{split}
% \end{equation}
%\normalsize

The UBU scheme consists of applying first a half-step ($h$ replaced by $h/2$) using \eqref{eq:Udef}, then applying an impulse based on the potential energy term, and following this by another half-step of the mapping $\mathcal{U}$. The symmetry of this scheme plays a role in its accuracy, since it induces a cancellation of the leading order term in the error expansion (in the absence of gradient noise).

Other symmetric splitting methods are possible and have been extensively studied in recent years. In particular, the BAOAB, ABOBA and OBABO schemes (\cite{OBABO}, \cite{LeMa13}, \cite{lemast2016}) are all second order in the weak (sampling bias) sense and make plausible candidates for BNN sampling.  These methods can be viewed as  breaking the ${\cal U}$ stage of the UBU algorithm into two parts and interspersing the resulting solution maps with ${\cal B}$ steps to form a numerical scheme. Several of these methods will be compared in Section \ref{sec:num}.

\subsection{Stochastic gradient algorithms}\label{sec:stochastic_gradients}

We first abstractly define stochastic gradients as unbiased estimators of the gradient of the potential, under the same set of assumptions as \cite{leimkuhler2023contractionb}.
\begin{definition} \label{def:stochastic_gradient}
A \textit{stochastic gradient approximation} of a potential $\uu$ is defined by a function $\mathcal{G}:\R^d \times \Omega \to \R^d$ and a probability distribution $\rho$ on a Polish space $\Omega$, such that for every $x\in \R^d$, $\mathcal{G}(x, \cdot)$ is measurable on $(\Omega,\mathcal{F})$, and for $\omega\sim \rho$,
\[\E(\mathcal{G}(x,\omega)) = \nabla \uu(x).\]
The function $\mathcal{G}$ and the distribution $\rho$  together define the stochastic gradient, which we denote as  $(\mathcal{G}, \rho)$.
\end{definition}

The following assumption is useful for controlling the accuracy of the stochastic gradient approximations.
\begin{assumption}\label{Assumption:Bounded_Variance}
    The Jacobian of the stochastic gradient $\mathcal{G}$, $D_x\mathcal{G}(x,\omega)$ exists and  is measurable on $(\Omega,\mathcal{F})$.
    There exists a bound $C_{G}>0$ such that, for $\omega\sim \rho$,
    \begin{equation*}
        \sup_{x\in \R^d}\mathbb{E}\|D_{x}\mathcal{G}(x,\omega) - \nabla^{2}\uu (x)\|^{2} \leq C_{G}.
    \end{equation*}
\end{assumption}

\begin{assumption}[Moments of Stochastic Gradient]\label{assum:stochastic_gradient}
Let $\mathcal{D}(y,\omega) := \mathcal{G}(y,\omega)-\nabla \uu (y)$, for all $y \in \R^{d}$ be the difference between the stochastic gradient approximation, $\mathcal{G}(y,\omega)$, and the true gradient. The following moment bound holds on $\mathcal{D}$
\begin{align*}
    \mathbb{E}\left[\|\mathcal{D}(y,\omega)\|^{2} \mid y\right]
    &\leq C_{SG}^{2}M^{2}\|y-x^{*}\|^{2}.
\end{align*}
\end{assumption}

In Bayesian inference settings, stochastic gradients are usually considered for a potential $\uu:\mathbb{R}^{d} \to \mathbb{R}$ of the form
\begin{equation}\label{eq:split_potential}
    \uu(x) = \uu_{0}(x) + \sum^{N_{D}}_{i=1}\uu_{i}(x),
\end{equation}
where $x \in \mathbb{R}^{d}$ and $N_{D}$ is the size of the dataset; the individual contributions arise as terms in a summation that defines the log-likelihood of parameters given the data. In many applications, such as the ones we consider in this work, both the dimension $d$ and the size of the dataset $N_{D}$ are large. As a result, gradient evaluations of the potential can be computationally expensive. As a consequence, a stochastic approximation of the gradient is typically used (see e.g. \cite{robbins1951stochastic}) which dramatically reduces the computational cost by relying only on a sample of the dataset of size $N_{b} \ll N_{D}$. For MCMC sampling it is natural to seek a similar approach to improve the scalability  (\cite{welling2011bayesian}, \cite{nemeth2021stochastic}).  Within \eqref{eq:split_potential}, $\uu_{0}$ can be chosen to be the negative log density of the prior distribution. Then random variables of the form $\omega \in [N_D]^{N_b}$ are constructed as uniform draws on $[N_D]=\{1,\ldots,N_{D}\}$, taken i.i.d. with replacement (\cite{baker2019control}).  At each step one uses an unbiased estimator of \eqref{eq:split_potential} instead of the full gradient evaluation, for example
\begin{align}\nonumber
    \mathcal{G}(x,\omega|\hat{x}) &= \nabla \uu_{0}(x) + \sum^{N_D}_{i=1}\nabla \uu_{i}(\hat{x}) \\&+ \frac{N_D}{N_b} \sum_{i \in \omega}[\nabla \uu_{i}(x) - \nabla \uu_{i}(\hat{x})], \label{eq:SG1}
\end{align}
such that $\mathbb{E}(\mathcal{G}(x,\omega)) = \nabla \uu(x)$, where, in case the potential is convex, $\hat{x}$ can be chosen as the minimizer of $\uu$, $x^{*} \in \mathbb{R}^{d}$. 
We define the random variables 
\begin{align}
\nonumber Z^{x}&:= \sqrt{\frac{2}{\gamma}}\left(\mathcal{Z}^{(1)}\left(h/2,\xi^{(1)}\right) - \mathcal{Z}^{(2)}\left(h/2,\xi^{(1)},\xi^{(2)}\right)\right),\\ Z^{v}&:=\sqrt{2\gamma}\mathcal{Z}^{(2)}\left(h/2,\xi^{(1)},\xi^{(2)}\right),\label{def:ZxZv}
\end{align}
based on \eqref{eq:Z1def} and \eqref{eq:Z2def}. In Algorithm \ref{alg:SG-UBU} we formulate a stochastic gradient implementation of the UBU scheme.

\begin{algorithm}[t]
    \footnotesize
    \begin{algorithmic}
		\item Initialize $\left(x_{0},v_{0}\right) \in \mathbb{R}^{2d}$, stepsize $h > 0$, \textcolor{black}{sample size $K>0$} and friction parameter $\gamma > 0$.
            % \item Sample $W_{1} \sim \rho$
            % \item $G_{0} \to \mathcal{G}(x_{0},\omega_{1})$
		\For {$k = 1,2,...,K$}
		\begin{itemize}
                \item[] Sample $Z^{x}_{k},Z^{v}_{k},\Tilde{Z}^{x}_{k},\Tilde{Z}^{v}_{k}$ according to \eqref{def:ZxZv}
			\item[(U)]
                     $(x,v) \to (x_{k-1} + \frac{1-\eta^{1/2}}{\gamma}v_{k-1}+Z^{x}_k, \eta^{1/2} v_{k-1}+Z^{v}_k)$
   			\item[] 
                    Sample $\omega_{k} \sim \rho$
                   \item[(B)] $v\to v-h  \mathcal{G}(x,\omega_{k})$
			\item[(U)] 
                     $(x_k,v_k) \to (x + \frac{1-\eta^{1/2}}{\gamma}v+\tilde{Z}^{x}_k, \eta^{1/2} v+\tilde{Z}^{v}_k)$
                     % $v \to v_{k-1} - \frac{h}{2}G_{k-1}$
                  \end{itemize}   
            \EndFor
            \item Output: Samples $(x_{k})^{K}_{k=0}$.

    \end{algorithmic}
	\caption{Stochastic Gradient UBU (SG-UBU)}
	\label{alg:SG-UBU}
\end{algorithm}

A limitation of using stochastic gradients is that it can introduce additional bias.  In particular, using the standard stochastic gradient approximation (sampling the minibatches i.i.d. at each iteration) reduces the strong order of a numerical integrator for kinetic Langevin dynamics to  $O(h^{1/2})$. SG-UBU for example will be order $1/2$, a dramatic reduction in accuracy compared to its full gradient counterpart, which has strong order two (see \cite{BUBthesis}, \cite{sanz2021wasserstein}). In terms of the weak order and bias in empirical averages stochastic gradients reduce the order of accuracy in the stepsize from two to one (\cite{Vollmer2016}, \cite{Sekkat2023}).  

A consequence of the reduced accuracy with respect to stepsize is substantially worse non-asymptotic guarantees and scaling in terms of dimension (\cite{chatterji2018theory}, \cite{DalalyanSG}, \cite{gouraud2022hmc}). \textcolor{black}{We note that some practical implementations in the literature of stochastic gradient based sampling schemes also consider sampling without replacement, as discussed in \cite{fortuna}. We do not prove theoretical results for such a version of SG-UBU, but perform some numerical experiments in Sections \ref{sec:numevalbias} and \ref{sec:num}.}
\subsection{Symmetric Minibatch Splitting UBU (SMS-UBU)}
\label{sec:sms-ubu}
We now introduce an alternative stochastic gradient approximation method that relies on a random partition 
$\omega_{1},....,\omega_{N_{m}}$ of $[N_D]=\{1,\ldots, N_D\}$, each of size $N_b$ (i.e. they are sampled uniformly without replacement from $[N_D]$, and $\omega_{1}\cup\ldots \cup \omega_{N_m}=[N_D]$). 
%which is used in a symmetric fashion (forward/backward).
%creating a random collection of minibatches of size $N_{m}$ every $2N_{m}$ iterations, that is we generate random variables of the form $\omega_{1},....,\omega_{N_{m}} \in [N_{D}]^{N_{b}}$ sampled uniformly without replacement. 
We propose to take UBU steps with gradient approximations using the index-set $\omega_{1}$ first, then $\omega_{2}$, etc., all the way up to $\omega_{N_{m}}$.  We follow this with steps using the same sequence of minibatches taken in the reverse order.   We illustrate this methodology for the UBU integrator below:
\begin{align*}
\underbrace{(\U\B_{\omega_{1}}\U)}_{\textnormal{minibatch 1}}\, &... \underbrace{(\U\B_{\omega_{N_{m}}}\U)}_{\textnormal{minibatch }N_{m}}\,\underbrace{(\U\B_{\omega_{N_{m}}}\U)}_{\textnormal{minibatch }N_{m}}\,
...\underbrace{(\U\B_{\omega_{1}}\U)}_{\textnormal{minibatch 1}},
\end{align*}
where 
\[
\mathcal{B}_{\omega_{l}}(x,v,h) = \left(x,v-hN_{m}\sum_{i \in \omega_{l}}\nabla \uu_{i}(x)\right),
\]
and $N_{m} := N_{D}/N_{b}$ is the number of minibatches. This framework is not specific to the UBU integrator and can be applied to any kinetic Langevin dynamics integrator. However, we detail precisely the SMS stochastic gradient algorithm for the UBU integrator in Algorithm \ref{alg:SMS-UBU}.

A motivation for such a splitting is that it ensures that the integrator is symmetric as an integrator on a longer time horizon $T = 2hN_{m}$ and therefore has improved weak order properties \cite{lemast2016}. It also turns out that it gains improved strong order properties, which we make rigorous in the following sections.

\begin{remark}
    A similar minibatch selection procedure was used within the context of HMC to integrate Hamiltonian dynamics based on symmetric splitting with stochastic gradient approximations for each leapfrog step (\cite{HMCBNN}). Their motivation was different than ours: as they were working in the HMC setting, they required an integrator for Hamiltonian dynamics that offered high Metropolis-Hastings acceptance rate.  They did not study the use of this type of subsampling for unadjusted Langevin. Their approach can only take samples after a complete forward/backward sweep (going through the entire dataset) when the accept/reject step is applied.  This requires going through the entire dataset between consecutive samples, which is not needed for unadjusted Langevin. We numerically evaluate such an HMC-based approach in Section \ref{sec:comparsionwithSMSHMC} of the Supplementary material.%Similar work has also been considered in \textcolor{black}{as well as \cite{Franzese}}
\end{remark}

\begin{remark}
\textcolor{black}{A similar procedure is proposed in \cite{Franzese}, also considering back and forth operators.
That paper studies a variety of numerical methods, including OABAO, a leapfrog scheme, and a higher order scheme, but as pointed out in their paper the order of accuracy is limited by the error due to gradient noise. Although they demonstrate an improved accuracy with the higher order scheme it requires multiple gradient evaluations per sample, which is considerably less efficient than our approach.}
\end{remark}

\begin{algorithm}[t]
    \footnotesize
    \begin{algorithmic}
		\State Initialize $\left(x_{0},v_{0}\right) \in \mathbb{R}^{2d}$, stepsize $h > 0$, \textcolor{black}{sample size $K>0$}, friction parameter $\gamma > 0$ and number of minibatches $N_{m}$.
		\For {$i = 1,2,...,\lceil K/2N_{m}\rceil$}\\
            Sample $\omega_{1},...,\omega_{N_{m}} \in [N_{D}]^{N_{b}}$ uniformly without replacement. \\
            \State \textbf{\underline{Forward Sweep}}\\
            \For {$k = 1,2,...,\min\{N_{m},K-(2i-2)N_{m}\}$} 
                
                \begin{itemize}
                    \item[] Sample $Z^{x}_{k},Z^{v}_{k},\Tilde{Z}^{x}_{ k},\Tilde{Z}^{v}_{k}$ according to \eqref{def:ZxZv}
                     \item[(U)] $(x,v) \to (x_{(2i-2)N_{m} + k-1} + \frac{1-\eta^{1/2}}{\gamma}v_{(2i-2)N_{m} + k-1}+Z^{x}_{k},\eta^{1/2} v_{(2i-2)N_{m} + k-1}+Z^{v}_{k})$ 
                \end{itemize}
                \begin{itemize}
                \item[(B)] $v \to v-h \mathcal{G}(x,\omega_{k})$
                \end{itemize}
			
                \begin{itemize}
                     \item[(U)] $(x_{2N_{m}i + k},v_{2N_{m}i + k}) \to (x + \frac{1-\eta^{1/2}}{\gamma}v+\tilde{Z}^{x}_{k}, \eta^{1/2} v+\tilde{Z}^{v}_{k})$
                \end{itemize}
            \EndFor\\
            \State \textbf{\underline{Backward Sweep}}\\
            \For {$k = 1,2,...,\min\{N_{m},K-(2i-1)N_{m}\}$}
			
                \begin{itemize}
                     \item[] Sample $Z^{x}_{k},Z^{v}_{k},\Tilde{Z}^{x}_{k},\Tilde{Z}^{v}_{k}$ according to \eqref{def:ZxZv}
                     \item[(U)] $(x,v) \to (x_{(2i-1)N_{m} + k-1} + \frac{1-\eta^{1/2}}{\gamma}v_{(2i-1)N_{m} + k-1}+Z^{x}_{k}, \eta^{1/2} v_{(2i-1)N_{m} + k-1}+Z^{v}_{k})$
                \end{itemize}
                \begin{itemize}
                    \item[(B)] $v \to v-h \mathcal{G}(x,\omega_{N_{m} +1 - k})$
                \end{itemize}
                \begin{itemize}
                   \item[(U)]
                    $(x_{(2i-1)N_{m} + k},v_{(2i-1)N_{m} + k}) \to (x + \frac{1-\eta^{1/2}}{\gamma}v+\tilde{Z}^{x}_{k}, \eta^{1/2} v+\tilde{Z}^{v}_{k})$
                \end{itemize}
            \EndFor
        \EndFor
        \State Output: Samples $(x_{k})^{K}_{k=0}$.
    \end{algorithmic}
	\caption{Symmetric Minibatch Splitting UBU (SMS-UBU)}
	\label{alg:SMS-UBU}
\end{algorithm}

\section{Theory}
\label{sec:main}
In this section we present theoretical results related to SMS-UBU. We provide the key assumptions to establish a  Wasserstein contraction result related to the convergence of our numerical scheme, and we give a global error estimate which establishes the $\mathcal{O}(h^2)$ strong accuracy bound, where $h$ is the stepsize used. Proofs of these results are deferred to the appendix.
\subsection{Assumptions}
\label{sec:ass}
We first state the assumptions used in our proofs.

\begin{assumption}\label{assum:convex}
$f$ is $m$-strongly convex if there exists a $m > 0$ such that for all $x,y \in \mathbb{R}^{d}$
\[
\left\langle \nabla f(x) - \nabla f(y),x-y \right\rangle \geq m \left|x-y\right|^{2}.
\]
\end{assumption}

\begin{assumption}
\label{assum:Lip}
$\uu:\R^{d} \to \R$ is $C^2$, and there exists $M>0$ such that for all $x,y \in \R^d$
$$
\| \nabla \uu(x) - \nabla \uu(y)\| \leq M\|x-y\|.
$$
\end{assumption}

\begin{assumption}\label{ass:hessian_lip}
The potential $\uu:\mathbb{R}^{d} \to \mathbb{R}$ is $C^3$ and there exists $M_{1} > 0$ such that for all $x,y \in \mathbb{R}^{d}$,
\[
\|\nabla^{2}\uu(x) - \nabla^{2}\uu(y)\| \leq M_{1}\|x-y\|,
\]
this implies that
\[
\|\nabla^{3}\uu(x)[v,v']\| \leq M_{1}\|v\|\|v'\|,
\]
as used in \cite{sanz2021wasserstein}.
\end{assumption}

The strongly Hessian Lipschitz property relies on a specific tensor norm from \cite{chen2023does}, which we use to establish improved dimension dependence.  A small strongly Hessian Lipschitz constant for $\nabla^{3}\uu$ is shown to hold for some applications of practical interest (see \cite{chen2023does}, \cite{ububu} and Lemma 7 of \cite{paulin2024}).
\begin{definition}\label{def:strongHessianLipschitznorm}
    For $A \in \mathbb{R}^{d \times d \times d}$, let
\begin{equation}\label{eq:strongHessianLipschitz:alternative:def}   
\|A\|_{\{12\}\{3\}}=\left\|\sum_{i_1} A_{i_1,\cdot,\cdot}^T\cdot A_{i_1,\cdot,\cdot}\right\|^{1/2},
\end{equation}
where $\|\cdot\|$ refers to the $L^2$ matrix norm, and $A_{i_1,\cdot,\cdot}=\left(A_{i_1,i_2,i_3}\right)_{1\le i_2\le d, 1\le i_3\le d}$ is a $d\times d$ matrix.
    % \tiny{
    % \begin{align*}
    %         &\|A\|_{\{1,2\}\{3\}} \\&= \sup_{x\in \mathbb{R}^{d\times d}, y\in \mathbb{R}^d}\left\{\left.\sum^{d}_{i,j,k=1}A_{ijk}x_{ij}y_{k} \right| \sum_{i,j=1}^{d}x^{2}_{ij}\leq 1, \sum^{d}_{k=1}y^{2}_{k}\leq 1\right\}.
    % \end{align*}}
\end{definition}
%\normalsize
% \begin{remark}
% The $\|A\|_{\{1,2\}\{3\}}$ norm in Definition \ref{def:strongHessianLipschitznorm} can be equivalently written as 
% \begin{equation}\label{eq:strongHessianLipschitz:alternative:def}   
% \|A\|_{\{12\}\{3\}}=\left\|\sum_{i_1} A_{i_1,\cdot,\cdot}^T\cdot A_{i_1,\cdot,\cdot}\right\|^{1/2},
% \end{equation}
% where $\|\cdot\|$ refers to the $L^2$ matrix norm, and $A_{i_1,\cdot,\cdot}=\left(A_{i_1,i_2,i_3}\right)_{1\le i_2\le d, 1\le i_3\le d}$ is a $d\times d$ matrix, see the proof of Lemma 7 of \cite{paulin2024}. 
% \end{remark}

\begin{assumption}[$M^{s}_{1}$-strongly Hessian Lipschitz]
\label{ass:strong_hessian_lip}
$\uu:\mathbb{R}^{d} \to \mathbb{R}$ is three times continuously differentiable and $M^{s}_{1}$-strongly Hessian Lipschitz if there exists $M^{s}_{1} > 0$ such that
    \[
    \|\nabla^{3}\uu(x)\|_{\{1,2\}\{3\}} \leq M^{s}_{1},
    \]
    for all $x \in \mathbb{R}^{d}$.
\end{assumption}

\subsection{Convergence of the numerical method}
In this subsection, we summarize Wasserstein convergence results from \cite{ububu} which used the technique developed in \cite{LePaWh24}. There have been many works on couplings of kinetic Langevin dynamics and its discretization including \cite{cheng2018underdamped}, \cite{dalalyan2020sampling}, \cite{PM21}, \cite{sanz2021wasserstein}, \cite{chakreflection} and \cite{schuh2024convergence} in the discretized setting and \cite{eberle2019couplings} and  \cite{schuh2022global} in the continuous setting. We remark that Wasserstein convergence for the UBU discretization 
%of a different formulation of the kinetic Langevin dynamics 
was first studied in \cite{sanz2021wasserstein}.
\begin{definition}[Weighted Euclidean norm]
\label{def:weightednorm}
For $z = (x,v) \in \R^{2d}$ the weighted Euclidean norm of $z$ is defined by
\[
\left|\left| z \right|\right|^{2}_{a,b} = \left|\left| x \right|\right|^{2} + 2b \left\langle x,v \right\rangle + a \left|\left| v \right|\right|^{2},
\]
for $a,b > 0$ with $b^{2}<a$. 
\end{definition}

\begin{remark}
Using the assumption $b^2<a$, we can show that this is equivalent to the Euclidean norm on $\R^{2d}$. Under the condition $b^2\le a/4$, we have
\begin{align} \nonumber
\frac{1}{2}\min(a,1) \|z\|^2&\leq \frac{1}{2}||z||^{2}_{a,0} \leq ||z||^{2}_{a,b} \leq \frac{3}{2}||z||^{2}_{a,0}\\&\leq \frac{3}{2}\max(a,1) \|z\|^2.\label{eq:normequiv}\end{align}
\end{remark}

\begin{definition}[$p$-Wasserstein distance]
\label{def:wass}
Let us define $\mathcal{P}_p(\R^{2d})$
to be the set of probability measures which have $p$-th moment for $p\in[1,\infty)$ (i.e. $\E (\|Z\|^p)<\infty$). Then the $p$-Wasserstein distance in norm $\|\cdot \|_{a,b}$ between two measures $\mu,\nu\in \mathcal{P}_p(\R^{2d})$ is defined as
\begin{equation}
\label{eq:wass_dist}
\mathcal{W}_{p,a,b}(\nu,\mu) =\Big(\inf_{\xi \in \Gamma(\nu,\mu)} \int_{\R^{2d}} \|z_1 - z_2\|^p_{a,b}d\xi(z_1,z_2)\Big)^{1/p},
\end{equation}
where $\| \cdot \|_{a,b}$ is the norm introduced before and that $\Gamma(\nu,\mu)$  is the set of measures with respective marginals of $\nu$ and $\mu$.
\end{definition}

\begin{remark}
    We use $\mathcal{W}_{2}$ to mean the standard Wasserstein-2 distance (equivalent to $\mathcal{W}_{2,1,0}$).
\end{remark}

\subsection{Error estimates}
\begin{theorem}
\label{maintheorem}
Consider the SMS-UBU scheme with friction parameter $\gamma >0$, stepsize $h>0$ and initial measure $\overline{\pi}_{0}$ and assume that $h<\frac{1}{2\gamma}$ and  $\gamma \geq \sqrt{8M}$. Let the potential $\uu$ be $M$-$\nabla$Lipschitz and                          of the form $\uu = \sum^{N_{m}}_{i=1}\uu_{i}$, where each $\uu_{i}$ is $m_{i}$-strongly convex for $i = 1,...,N_{m}$ and $M/N_{m}$-$\nabla$Lipschitz with minimizer $x^{*}\in \mathbb{R}^{d}$. Consider the measure of the position $k$-th SMS-UBU, denoted by $ \Tilde{\pi}_{k}$, which approximates the target measure $\pi$ for a potential $\uu$ that is $M_1$-Hessian Lipschitz. We have that
    \begin{align*}
    \mathcal{W}_{2}(\Tilde{\pi}_{k},\pi) &\leq  \sqrt{2}\exp{\left(-\frac{mh}{8\gamma}\left\lfloor \frac{k}{N_{m}}\right\rfloor\right)}\mathcal{W}_{2,a,b}(\overline{\pi}_{0},\overline{\pi}) \\&+ C(\gamma,m,M,M_{1},N_{m})h^{2}d,
    \end{align*}
    and if $\uu$ is $M^{s}_{1}$-strongly Hessian Lipschitz we have that
    \begin{align*}
    \mathcal{W}_{2}(\Tilde{\pi}_{k},\pi) &\leq  \sqrt{2}\exp{\left(-\frac{mh}{8\gamma}\left\lfloor \frac{k}{N_{m}}\right\rfloor\right)}\mathcal{W}_{2,a,b}(\overline{\pi}_{0},\overline{\pi}) \\&+ C(\gamma,m,M,M^{s}_{1},N_{m})h^{2}\sqrt{d}.
    \end{align*}
    If we impose the stronger stepsize restriction $h<1/(12\gamma N_{m})$, then our bounds simplify to the form
    \begin{align*}
        \mathcal{W}_{2}(\Tilde{\pi}_{k},\pi) &\leq  \sqrt{2}\exp{\left(-\frac{mh}{8\gamma}\left\lfloor \frac{k}{N_{m}}\right\rfloor\right)}\mathcal{W}_{2,a,b}(\overline{\pi}_{0},\overline{\pi}) \\&+ C(\Tilde{\gamma},\Tilde{m},\Tilde{M},\Tilde{M}^{s}_{1})h^{2}N^{5/2}_{m}\sqrt{d},
    \end{align*}
    where $\Tilde{\gamma} = \gamma/\sqrt{N_{m}}, \Tilde{m} = m/N_{m}, \Tilde{M} = M/N_{m}$ and $\Tilde{M}^{s}_{1} = M^{s}_{1}/N_{m}$, and similarly when we don't assume the potential is strongly Hessian Lipschitz.
\end{theorem}
% \begin{proof}
% For $k \in \mathbb{N}$ define $(\hat{x}_{k},\hat{v}_{k})$ to be SMS-UBU and $(X_t, V_t)_{t \geq 0}$ to be \eqref{eq:LD}, both initialized at the target measure $(\hat{x}_{0},\hat{v}_{0}) := (X_{0},V_{0}) \sim \overline{\pi}$ then,
% \begin{align*}
%     &\|x_{k} - X_{kh}\|_{L^{2}} \leq \|x_{k} - \hat{x}_{k}\|_{L^{2}} + \|\hat{x}_{k} - X_{kh}\|_{L^{2}}\\
%     &\leq \sqrt{2}\|(x_{k} - \hat{x}_{k},v_{k}-\hat{v}_{k})\|_{L^{2},a,b} + \|\hat{x}_{k} - X_{kh}\|_{L^{2}}\\
%     &\leq \sqrt{2}\exp{\left(-\frac{mh}{8\gamma}\left\lfloor \frac{k}{N_{m}}\right\rfloor\right)}\|(x_{0} - X_{0},v_{0}-V_{0})\|_{L^{2},a,b} \\&+ \|\hat{x}_{k} - X_{kh}\|_{L^{2}},
% \end{align*}
% where the final term can be bounded by Theorem \ref{theorem:second_order_bias} under appropriate assumptions.
% \end{proof}
Our next result focuses on the large stepsize regime, which has also recently been considered in \cite{shaw2025random}.
\begin{theorem}\label{theorem:SMS-UBU-large-step}
  Considering the SMS-UBU scheme with friction parameter $\gamma >0$, stepsize $h>0$, initial measure $\overline{\pi}_{0}$ and assuming that $h<\frac{1}{2\gamma}$, $\gamma \geq \sqrt{8M}$ and the stochastic gradient satisfies Assumptions \ref{assum:stochastic_gradient} with constants $C_{SG}$. Let the potential $\uu$ be $M$-$\nabla$Lipschitz, $m$-strongly convex and of the form $f = \sum^{N_{m}}_{i=1}f_{i}$, where each $f_{i}$ is $m_{i}$-strongly convex, where $m = \sum^{N_{m}}_{i=1}m_{i}$. Then at the $k$-th iteration for its measure $\pi^{k}$ we have the non-asymptotic bound 
    \begin{align*}
    &\mathcal{W}_{2,a,b}(\pi^{k},\overline{\pi}) \leq  \exp{\left(-\frac{mh}{8\gamma}\left\lfloor \frac{k}{N_{m}}\right\rfloor\right)}\mathcal{W}_{2,a,b}(\overline{\pi}_{0},\overline{\pi}) \\
    &+ C(\gamma/\sqrt{N_{m}},m/N_{m},M/N_{m})\left[\frac{C_{SG}\sqrt{h}}{N^{1/4}_{m}} + h\right]\sqrt{d}.
    \end{align*}
\end{theorem}
\begin{remark}
    When using random permutations we have an unbiased estimator of the force at each step and we can achieve similar bounds as \cite{gouraud2022hmc} and \cite{DalalyanSG} which allows us to improve upon our bounds in the large $N_{m}$, and large stepsize setting. This is the result of Theorem \ref{theorem:SMS-UBU-large-step}. We can demonstrate a corresponding result in the simpler setting with vanilla stochastic gradients for UBU (see Theorem \ref{theorem:SG-UBU}).
\end{remark}

% \begin{remark}
% \label{rem:unbiased}
%     Due to the strong order properties shown in Theorem \ref{maintheorem} this numerical method would improve the performance of the recently introduced unbiased estimation methods in \cite{ububu} and complexity guarantees can be shown for this method as a Corollary to Theorem \ref{maintheorem}. The methods of \cite{ububu} remove discretization bias via couplings, as an alternative to Metropolization.
% \end{remark}

\begin{remark}
\label{rem:nonconvex}
    These results can be extended to the non-convex setting using the results of \cite{schuh2022global} or \cite{schuh2024convergence} under appropriate assumptions (convexity outside a ball), using contraction results of the continuous or discrete dynamics. The dimension and stepsize dependence would be as in the convex case, with additional dependence on the radius of the ball of non-convexity.
\end{remark}

\subsection{Numerical evaluation of bias of integrators}
\label{sec:numevalbias}
To illustrate the $O(h^2)$ bias for SMS-UBU, we are going to first consider a simple one dimensional Gaussian target example of the form $U(x)=U_1(x)+U_2(x)$, where $U_i(x)=\frac{(x-x_i)^2}{\sigma_i^2}$ are Gaussian potentials. We have chosen the numerical values $x_1=-1$, $x_2=1$, $\sigma_1=0.5$, $\sigma_2=2$. In this example, there are only two batches when using batch size 1. Since the target is itself a Gaussian, by taking samples from the Markov chain, and the Gaussian target, we can evaluate the Wasserstein distances between them to good accuracy (it is well known that the Wasserstein distance of two one dimensional empirical distributions with the same $N$ samples equals $\frac{1}{N}\sum_{i=1}^N |x_{(i)}-y_{(i)}|$). Using step sizes $h=2^{-k}$ for $k=2,3,4,5$, and taking a large number of samples ($N=10^7\cdot 2^k$), we evaluated the Wasserstein bias between the Gaussian target and the stationary distribution of 9 samplers. SMS-UBU, SG-UBU, and SG-UBU with batches sampled without replacement, and the same 3 variants of BAOAB (\cite{leimkuhler2023contractionb}) and EM (SG-EM is called SG-HMC in the literature). The results are shown on Figure \ref{figure:wassbias} (the errors in our bias estimates seem to be negligibly small as the results did not change noticeably when increasing the number of samples $N$).
\begin{figure}[h!]
\centering
\includegraphics[width=0.48\linewidth]{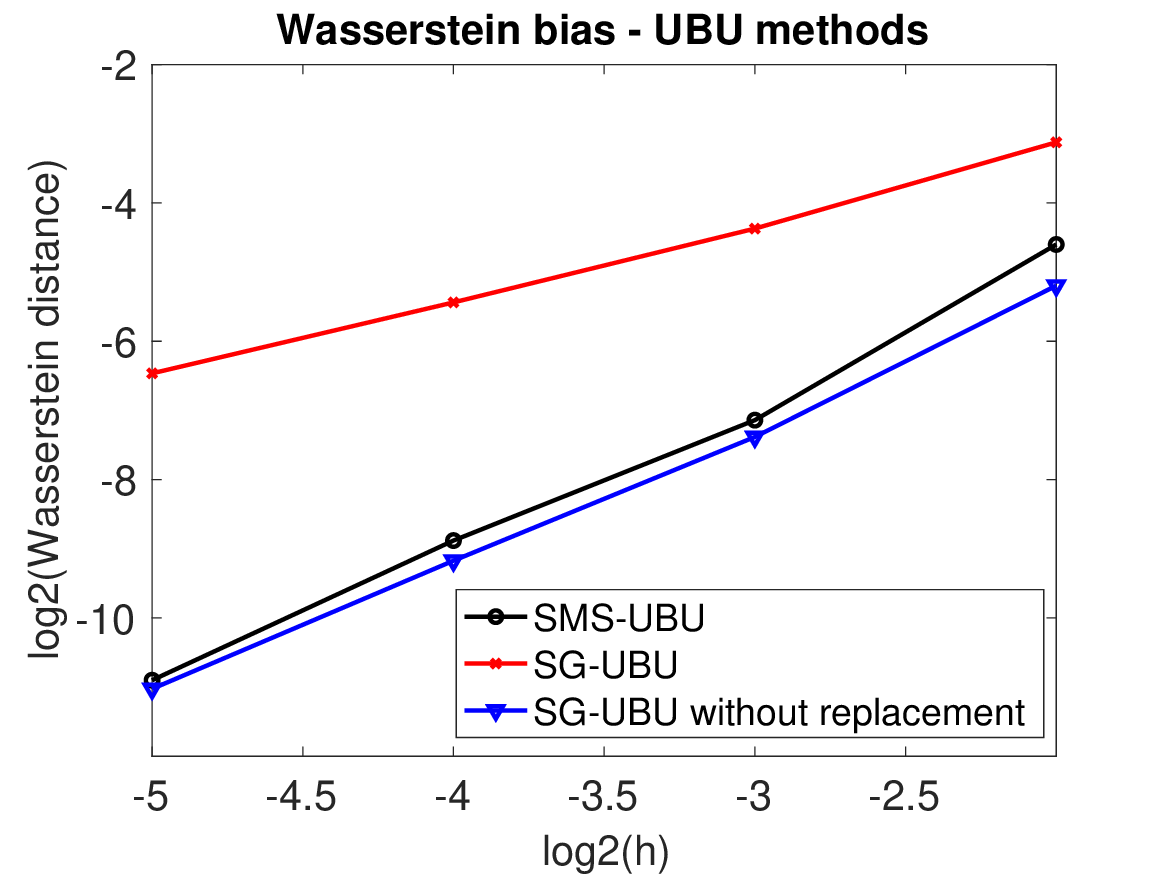}
\includegraphics[width=0.48\linewidth]{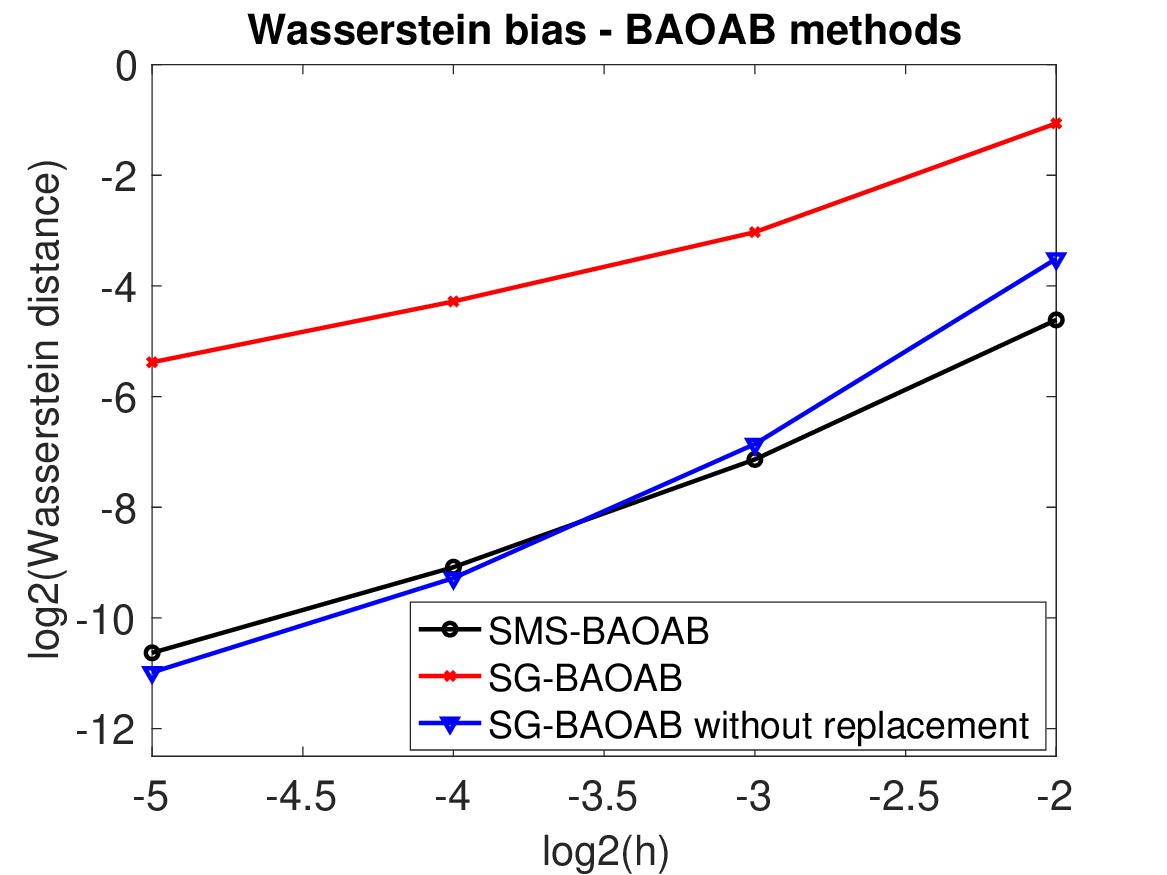}\\
\includegraphics[width=0.48\linewidth]{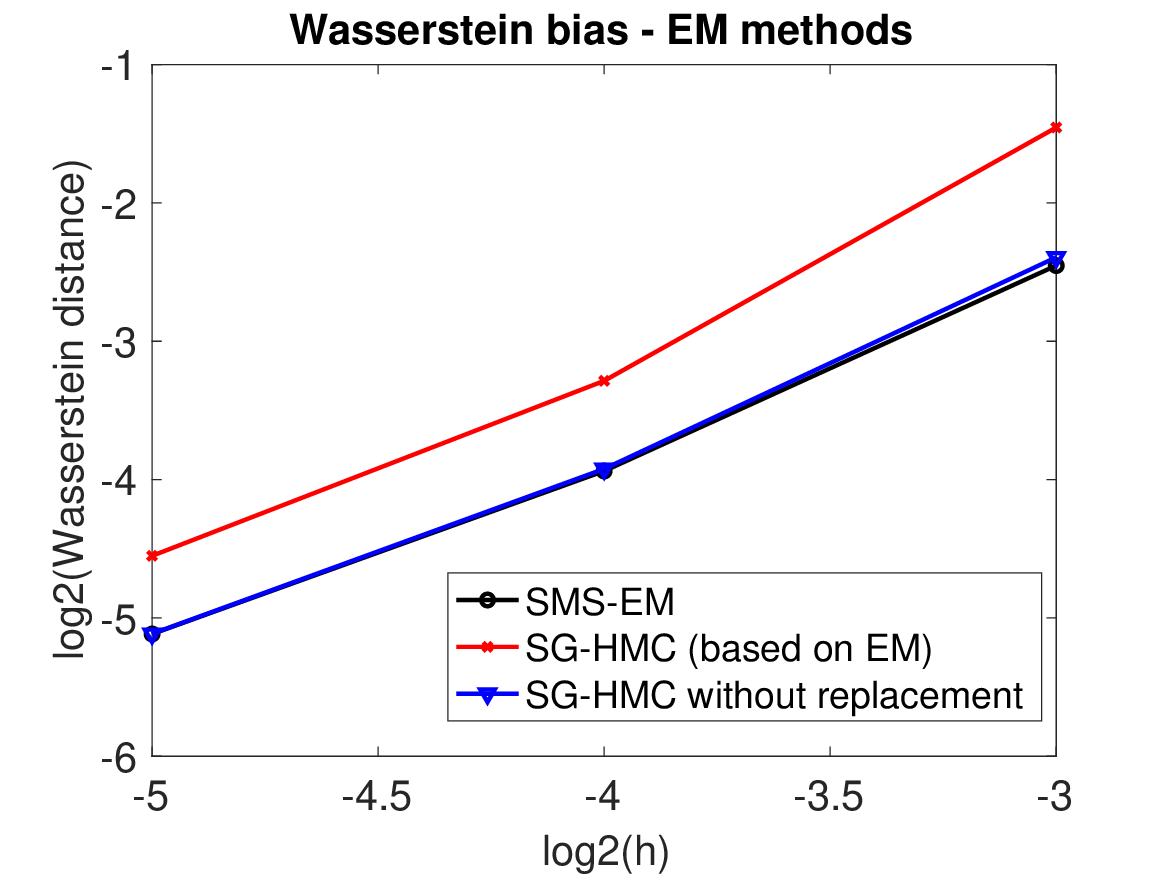}
\label{figure:wassbias}
\caption{Wasserstein bias of stochastic integrators of kinetic Langevin dynamics for a 1D Gaussian target}
\end{figure}

These results show that SMS-UBU and SG-UBU without replacement perform similarly, and they also slightly outperform all other methods. SMS-UBU, SG-UBU without replacement, SMS-BAOAB and SG-BAOAB without replacement appear to have a bias of $O(h^2)$ (weak order 2), while the other methods seem to have a bias of $O(h)$ (weak order 1). This is consistent with our theoretical results. SG-HMC and the other EM-based methods seem to have much larger bias than the methods based on higher order integrators (and EM based methods diverged at the largest step size $h=2^{-2}$).

We have conducted another experiment in Bayesian multinomial logistic regression (\cite{bishop2006pattern}) on the Fashion-MNIST dataset (\cite{xiao2017fashion}). We have used a isotropic Gaussian prior (quadratic regularizer) of the form $p(x)dx\propto \exp(-\|x\|^2/(2\sigma^2))dx$ with $\sigma=50^{-1/2}$.
The potential is strongly convex and log-concave. Using the L-BFGS algorithm, we found the global minimizer $x^*$, and initiated sampling algorithms from this point. 
We have considered five sampling algorithms: SG-HMC based on Euler-Mayurama discretization of kinetic Langevin (\cite{SGHMC}), as well as SG-BAOAB, SMS-BAOAB, SG-UBU, and SMS-UBU. Batch sizes were chosen as 200 (out of training size 60000), and we have implemented variance reduction \eqref{eq:SG1} with respect to $x^*$.
As test functions, we have used the posterior probabilities on the correct class on the test dataset (10000 images, so 10000 probabilities). To evaluate the bias accurately, we have constructed synchronous couplings of chains at stepsizes $h$ and $h/2$ (see Section \ref{sec:biasexplanation} in the Appendix). \textcolor{black}{One could use alternative splitting schemes, such as ABOBA as mentioned in \cite{chenMCMC}, but we would expect it to have worse performance compared to the SG-BAOAB integrator, since BAOAB is known to have improved bias compared to all other A,B and O splittings.}

In addition to the biases, we have also evaluated the accuracy of the estimators based on posterior mean predictive probabilities of each class, and also evaluated calibration performance in terms of average negative log-likelihood on test set (NLL), adaptive calibration error (ACE) \cite{nixon2019measuring}, and Ranked Probability Score (RPS) \cite{constantinou2012solving}. 

As Figure \ref{figure:bias} shows, SMS-UBU outperforms alternative methods in terms of calibration performance and accuracy at the largest stepsizes, and it has a significantly smaller bias compared to the BAOAB and EM integrators. The plots agree with our theoretical results, showing the bias to be $O(\sqrt{d}h^2)$. The dimensional dependency seems even better for these test functions, in line with the results in \cite{chen2024convergence}.
\begin{figure}[h!]
\centering
\includegraphics[width=0.48\linewidth]{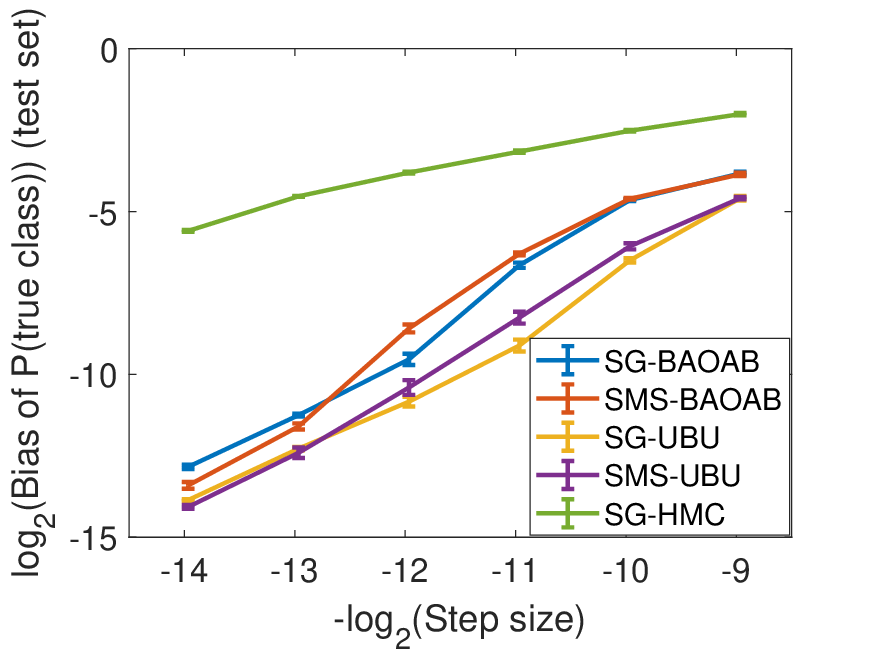}\\
\includegraphics[width=0.48\linewidth]{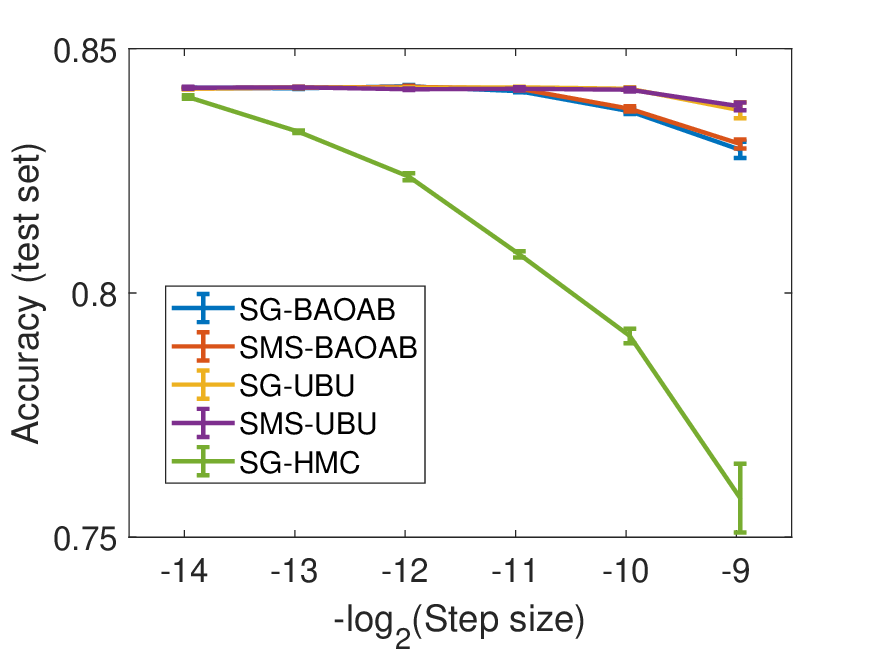}
\includegraphics[width=0.48\linewidth]{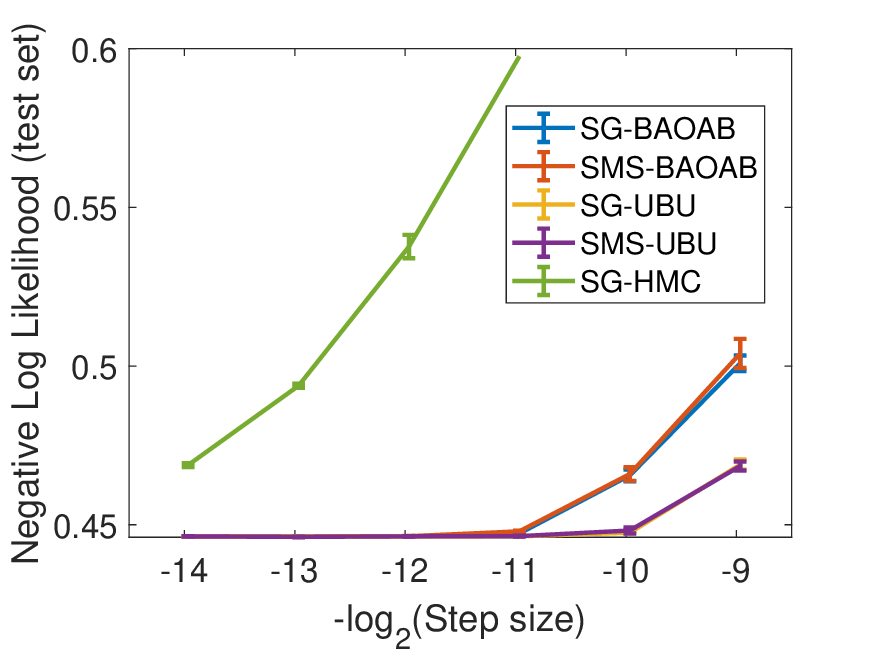}\\
\includegraphics[width=0.48\linewidth]{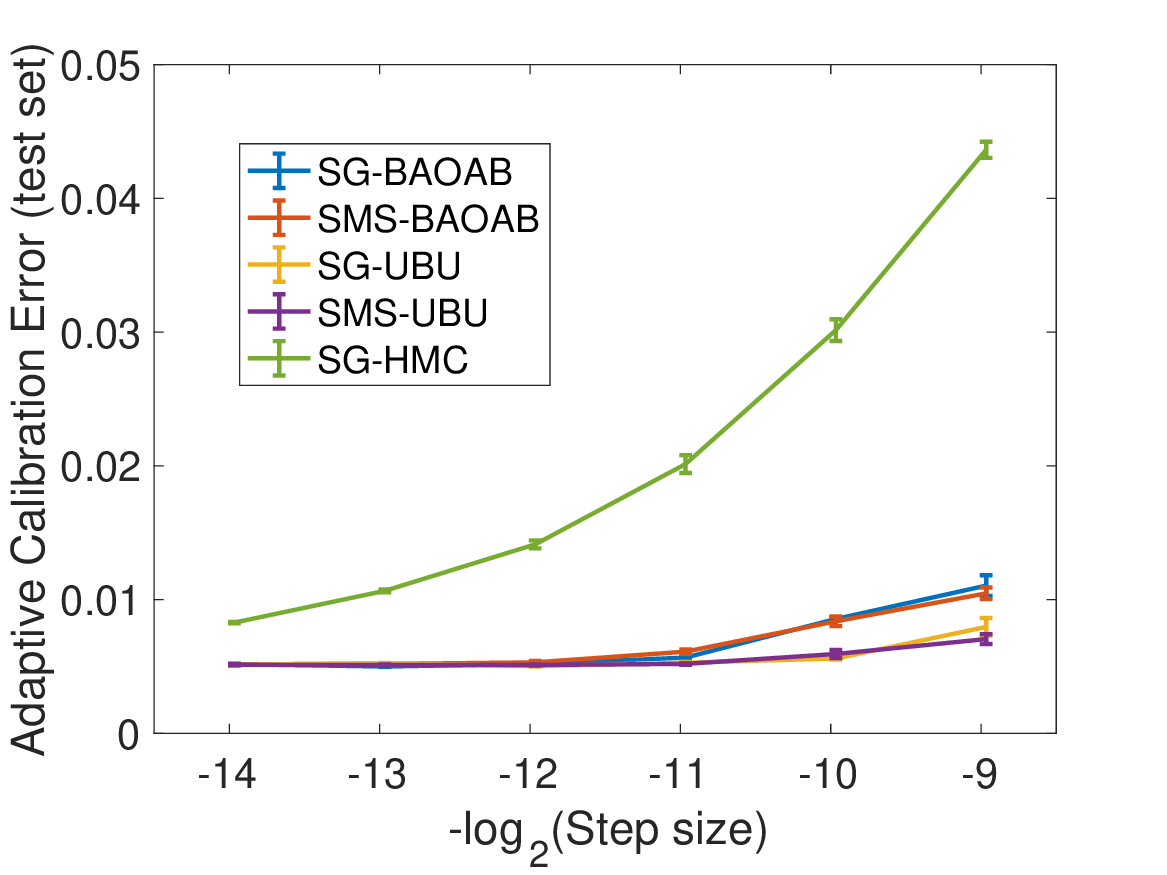}
\includegraphics[width=0.48\linewidth]{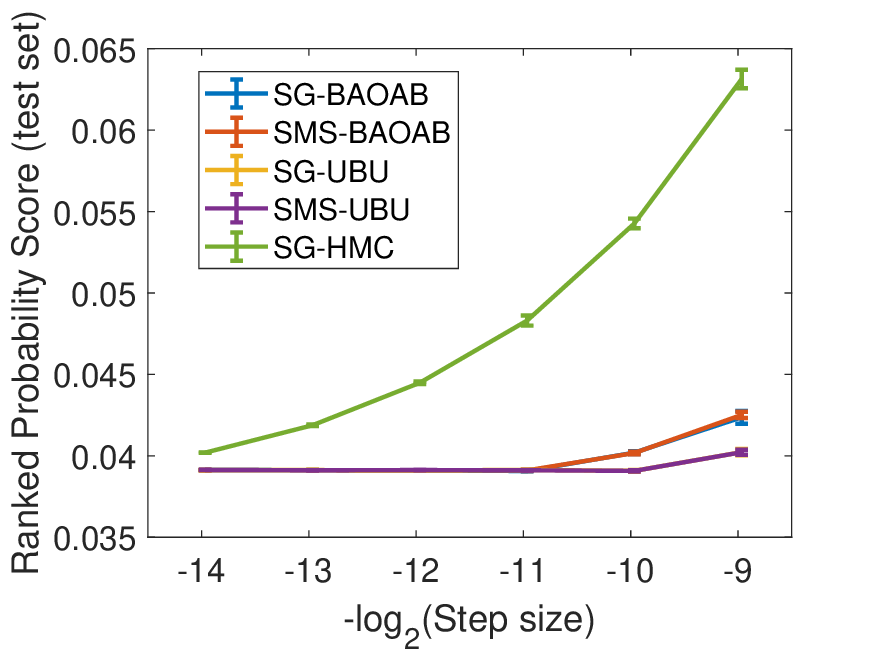}
\label{figure:bias}
\caption{Top: average bias of probability of true class on test dataset for five integrators as a function of stepsize. The other plots show the effect of stepsize on accuracy and calibration performance (NLL, ACE, and RPS).}
\end{figure}

\section{Bayesian uncertainty quantification in neural networks via sampling}
\label{sec:uq}
In this section, we explain the use of our sampling methods for exploring the posterior distributions of the parameters of neural networks. The theory is based on a unimodality assumption (log-concave target distribution)  whose potential has a bounded second derivative. It is well known that the loss functions of neural networks are highly multimodal and can have  minimizers in areas of high curvature, i.e. the norm of the Hessian may be large. Sampling from such multimodal distributions is challenging. Mixing can be slow even if full gradients are used (\cite{BNNpost}). 

There has been a significant amount of effort in the literature to find minimizers that lie in flat regions with low curvature, see \cite{baldassi2020shaping} and \cite{foret2021sharpnessaware}. In the experiments of this paper, we have used a slowly decreasing stepsize combined with Stochastic Weight Averaging (SWA) (\cite{SWA}) to find flat minimizers. We can assess the flatness of an area near a point via the norm of the Hessian of the log-posterior at the point. Figure \ref{fig:Hessiannorm} in Section \ref{sec:Hessiannorm} of the Appendix shows the typical norm of the Hessian of log-posterior for a CNN-based neural network trained for classification on the Fashion-MNIST dataset. The Hessian is much better behaved in the relatively smooth area in the neighbourhood of the SWA network optima.

Let $x^*$ the SWA weights found by the algorithm. Instead of exploring the multimodal target $\pi(d x)\propto \exp(-\uu(x))dx$, we propose the localized distribution $\pi^{*}(x)$ 
and potential energy $\uu^*(x)$ defined by
\begin{align*}
    &\pi^*(dx)\propto \exp(-\uu^*(x))dx \text{ for }\|x-x^*\|_{\infty}<\rho_{\max}, \\ &\uu^*(x)=\uu(x)+\frac{1}{2\rho^2} \|x-x^*\|^2,
\end{align*}
where $\rho, \rho_{\max}>0$ are so-called localization parameters.
$\uu^*(x)$ will be strongly convex due to inclusion of the quadratic regularizer term $\frac{1}{2\rho^2} \|x-x^*\|^2$ (for sufficiently small $\rho$).  We further restrict the parameters to the hypercube $\Omega=\{x: \|x-x^*\|_{\infty}<\rho_{\max}\}$ which ensures that the norm of the Hessian $\|\nabla^2(f^*(x))\|$ is relatively small within the domain $\Omega$ (see Figure \ref{fig:Hessiannorm}). Although  Algorithm \ref{alg:SMS-UBU} was stated for an unconstrained domain, it is straightforward to adapt it to the hypercube  $\Omega$ by implementing elastic bounces independently in each component whenever the $x$ component exits the domain (recent theoretical results for numerical integrators on constrained spaces suggest that such bounces do not change the order of accuracy, see \cite{leimkuhler2024numerical}). We chose $\rho_{\infty}=6\rho$ in our experiments, which ensured that bounces rarely happened.
Figure \ref{fig:pathsGR} shows the average training loss function for four independent SMS-UBU paths over 40 epochs started from $x^*$ plus a multivariate Gaussian with standard deviation $\rho=50^{-1/2}$. We have repeated this experiment 16 times, and plotted the $\hat{R}$ values (Gelman-Rubin diagnostics, computed after 10 epochs of burn-in). These plots indicate that SMS-UBU is able to approximate the localized distribution $\pi^*$ efficiently, in contrast to earlier experiments exploring the whole target $\pi$ (\cite{BNNpost}). The additional computational cost for sampling (40 epochs) is only twice the cost of optimization (20 epochs).

The local approximation $\pi^*$ can still be far away from the original multimodal distribution $\pi$. To obtain a better approximation, we can use ensembles of independently obtained SWA points $x^*_1,\ldots, x^*_N$, and run $N$ independent SMS-UBU chains initiated from these points. The full details of our approach are stated in Algorithm \ref{alg:ensSMSUBU}.

\begin{algorithm}[t]
    \footnotesize
    \begin{algorithmic}
		\Require Stepsize $h > 0$, friction parameter $\gamma > 0$. Number of sampling steps $K$, number of training epochs $T_{\mathrm{train}}$, number of SWA epochs $T_{\mathrm{SWA}}$. Localization parameters $\rho$, $\rho_{\max}$.
            \For {$n = 1,2,...,N$}
		\begin{enumerate}
                \item Randomly initialize the network weights.
                \item Train network for $T_{\mathrm{train}}$ epochs at decreasing stepsize.
                \item Continue training network with a fixed stepsize for $T_{\mathrm{SWA}}$ epochs, and accumulate the average weights over this period in variable $x^*_{(n)}$.
                \item Obtain $K$ samples from the distribution $\pi^{*}_{(n)}(x)\propto \exp\left(-\uu(x)-\frac{\|x-x^*_{(n)}\|^2}{2\rho^2}\right)$ using SMS-UBU initiated at $x_0=x^*_{(n)}$, $v_0\sim \mathcal{N}(0_d,I_d)$, with elastic bounces when exiting hypercube $\Omega_{(n)}=\{x: \|x-x^*_{(n)}\|_{\infty}<\rho_{\max}\}$.
            \end{enumerate}
            \EndFor
            \item Output: Samples $(x_{k}^{(n)})_{k\in 1:K, n\in 1:N}$.
	            
    \end{algorithmic}
	\caption{Ensemble SMS-UBU with SWA centred local approximation}
	\label{alg:ensSMSUBU}
\end{algorithm}

 %              \begin{itemize}
   %              \item[] Sample $Z^{x}_{k},Z^{v}_{k},\Tilde{Z}^{x}_{k},\Tilde{Z}^{v}_{k}$ according to \eqref{def:ZxZv}
			% \item[(U)]
   %                   $(x,v) \to (x_{k-1} + \frac{1-\eta^{1/2}}{\gamma}v_{k-1}+Z^{x}_k, \eta^{1/2} v_{k-1}+Z^{v}_k)$
   % 			\item[] 
   %                  Sample $\omega_{k} \sim \rho$
   %                 \item[(B)] $x\to x-h  \mathcal{G}(x,\omega_{k})$
			% \item[(U)] 
   %                   $(x_k,v_k) \to (x + \frac{1-\eta^{1/2}}{\gamma}v+\tilde{Z}^{x}_k, \eta^{1/2} v+\tilde{Z}^{v}_k)$
   %                   % $v \to v_{k-1} - \frac{h}{2}G_{k-1}$
   %                \end{itemize} 

\begin{figure}
\includegraphics[width=0.48\linewidth]{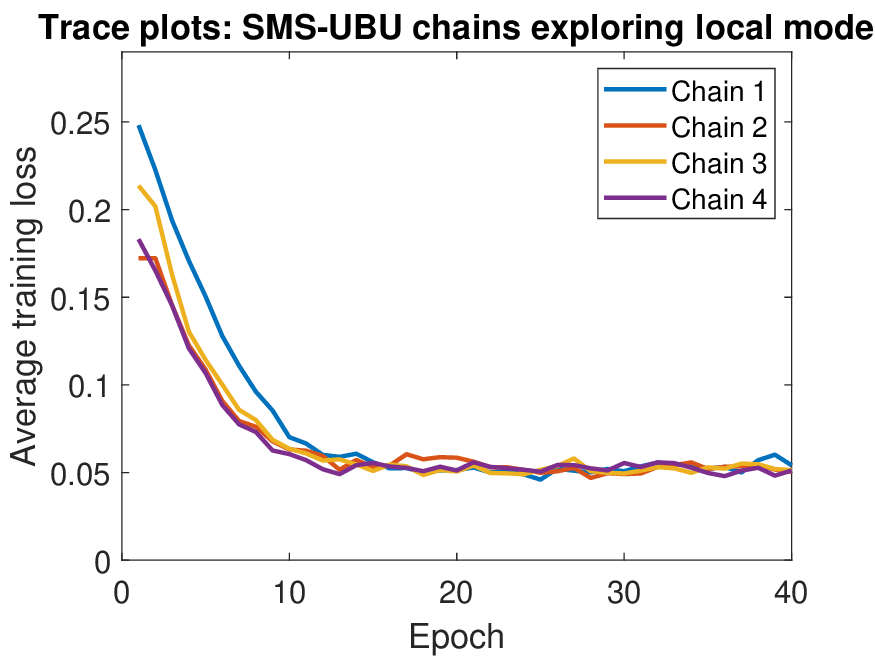}
\includegraphics[width=0.48\linewidth]{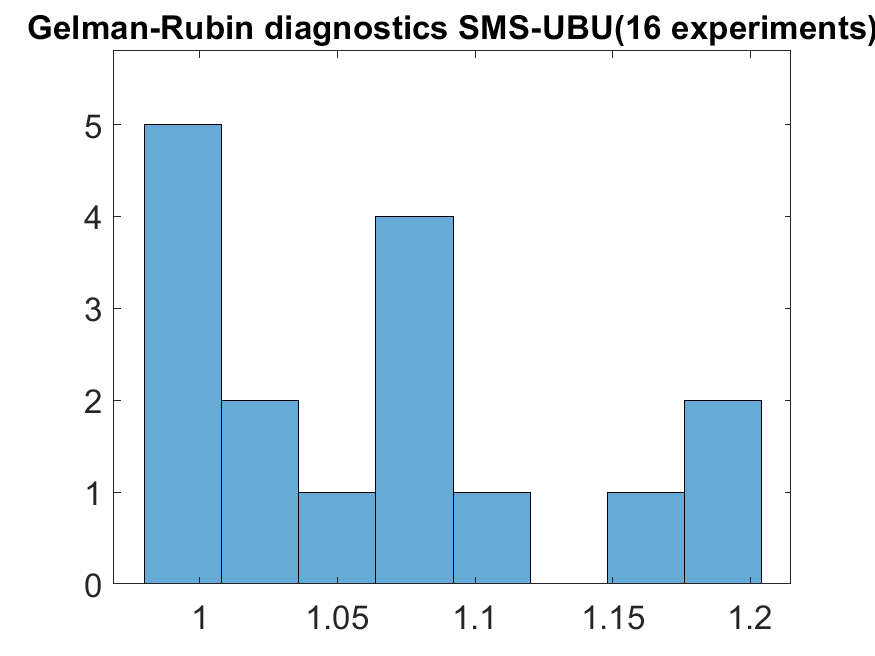}
\caption{Left: trace plots for 4 SMS-UBU chains initialized at Gaussian perturbations of SWA weights. Right: Gelman-Rubin diagnostic $\hat{R}$ of average training loss computed for 16 SWA weights $x^*$ obtained independently (four parallel chains each). The four chains converge to the same level, and $\hat{R}$ values are close to unity, indicating excellent mixing.}
\label{fig:pathsGR}
\end{figure}
\section{Experiments}
\label{sec:num}
In this section, we evaluate the accuracy and calibration performance of  ensembles of deep Bayesian neural networks on three datasets: Fashion-MNIST (\cite{xiao2017fashion}), Celeb-A (\cite{yang2015facial}), and chest X-ray (\cite{kermany2018identifying}). In the Celeb-A, we considered classification between blonde and brown hair colours. The networks had six convolutional layers combined with batch normalization and max pooling layers, followed by some MLP layers (these were chosen as low-rank to reduce the memory use). The Pytorch code of the them is included in Section \ref{sec:furtherdetailsnumerics} of the Appendix. We used a single RTX 4090 GPU for training.

The basic setup in all experiments was the same: 15 epochs of initial training with Adam at a step polynomially decaying stepsize schedule with power equal to 1, followed by five epochs of \textcolor{black}{SWA}, and 40 epochs of SMS-UBU (10 epoch was discarded as burn-in, and 30 epochs were used as samples). See Section \ref{sec:hyperparameterchoices} for further details.

We performed 64 independent runs, and compared the performance of ensemble methods (ensembles after 15 epochs of training, ensembles of SWA networks, and ensembles of Bayesian samples according to Algorithm \ref{alg:ensSMSUBU}). The results are shown in Figures \ref{fig:Fashion}, \ref{fig:CelebA}, and \ref{fig:chestXray}. Standard deviations were computed based on the independent runs (for example, by pooling the runs into ensembles of size four, there were 16 independent such ensembles, etc.) Bayesian networks offer small gains in accuracy over ensembles of the same size, and significant gains in calibration performance (NLL, ACE, and RPS scores). The Bayesian approach based on SMS-UBU seems to be able to provide a better modelling of uncertainty in the weights, which seems to persists even with the local $\pi^*$ used in Algorithm \ref{alg:ensSMSUBU}.
On Figure \ref{fig:Fashion}, we have also included experiments with four additional integrators: SG-UBU, SG-UBU with sampling without replacement, SG-HMC (\cite{SGHMC}) and cyclical SG-HMC (\cite{Zhang}). As proposed in the literature, SG-HMC and cyclical SG-HMC do not use local approximation $\pi^*$, but aim to explore the whole target (using same initialization and number of epochs for all methods). These do not perform as well the UBU-based methods relying on local approximations.

\begin{figure}[h!]
\begin{center}
\includegraphics[width=0.57\linewidth]{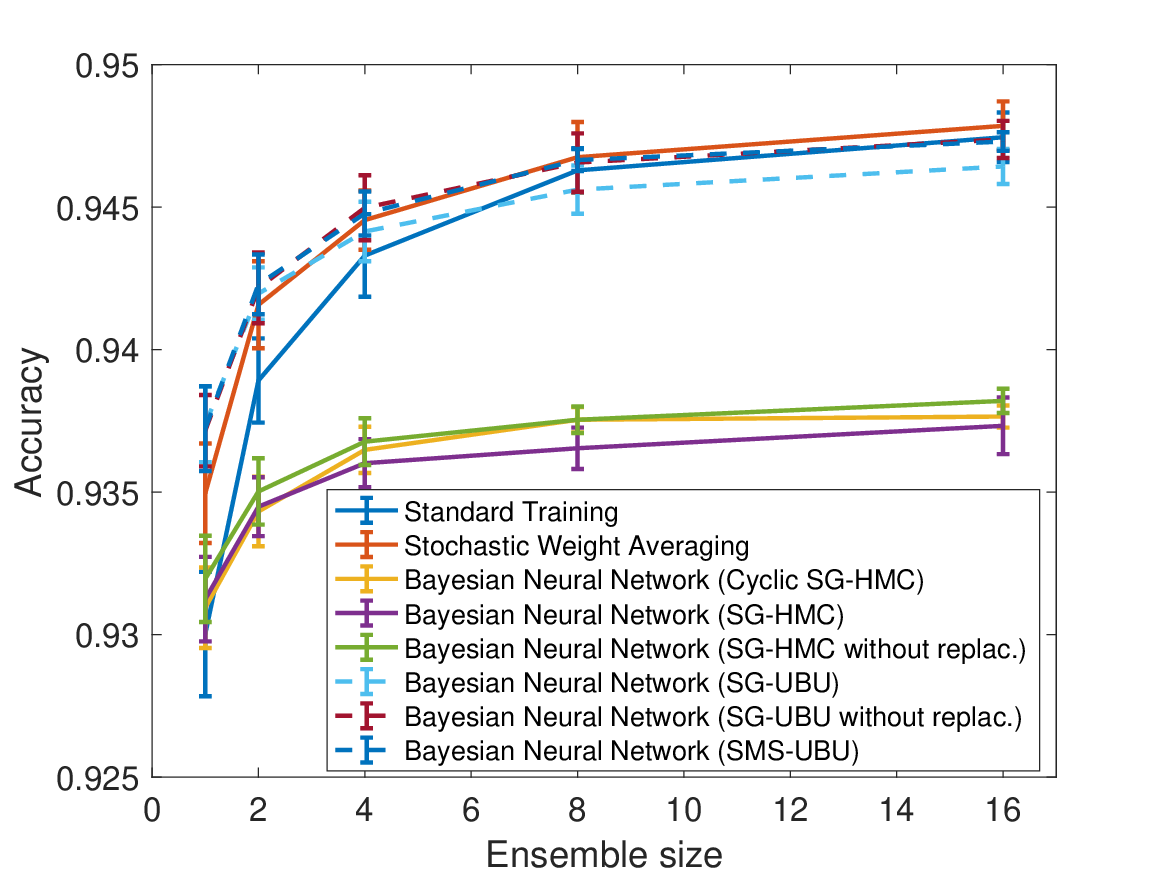}\\
\includegraphics[width=0.57\linewidth]{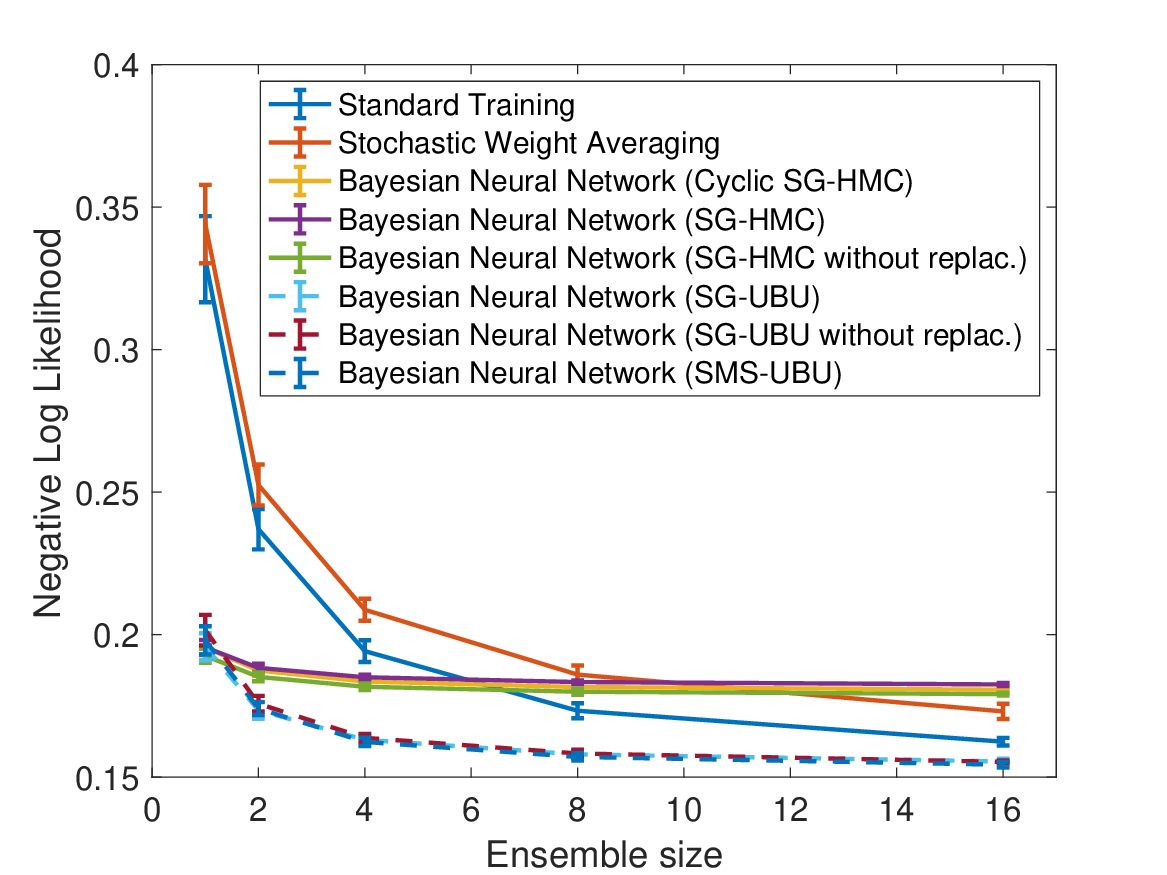}\\
\includegraphics[width=0.57\linewidth]{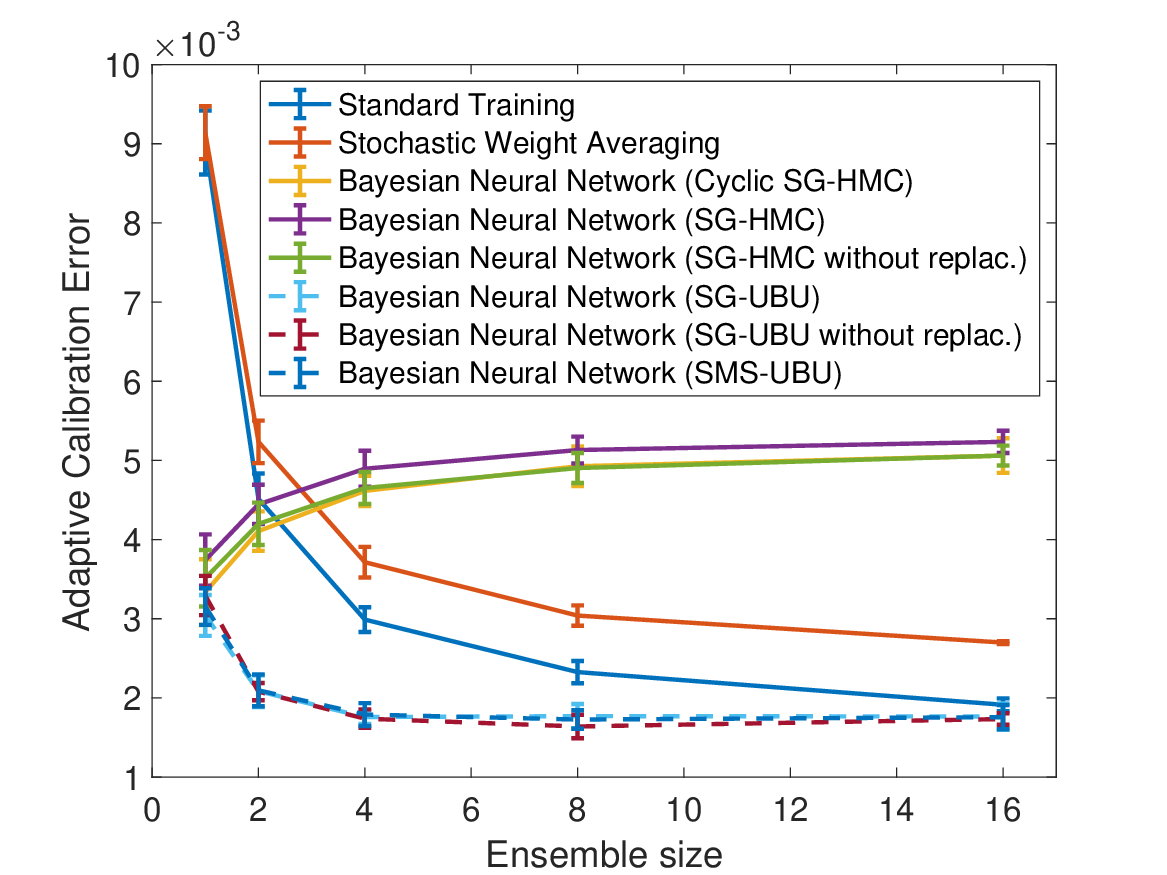}\\
\includegraphics[width=0.57\linewidth]{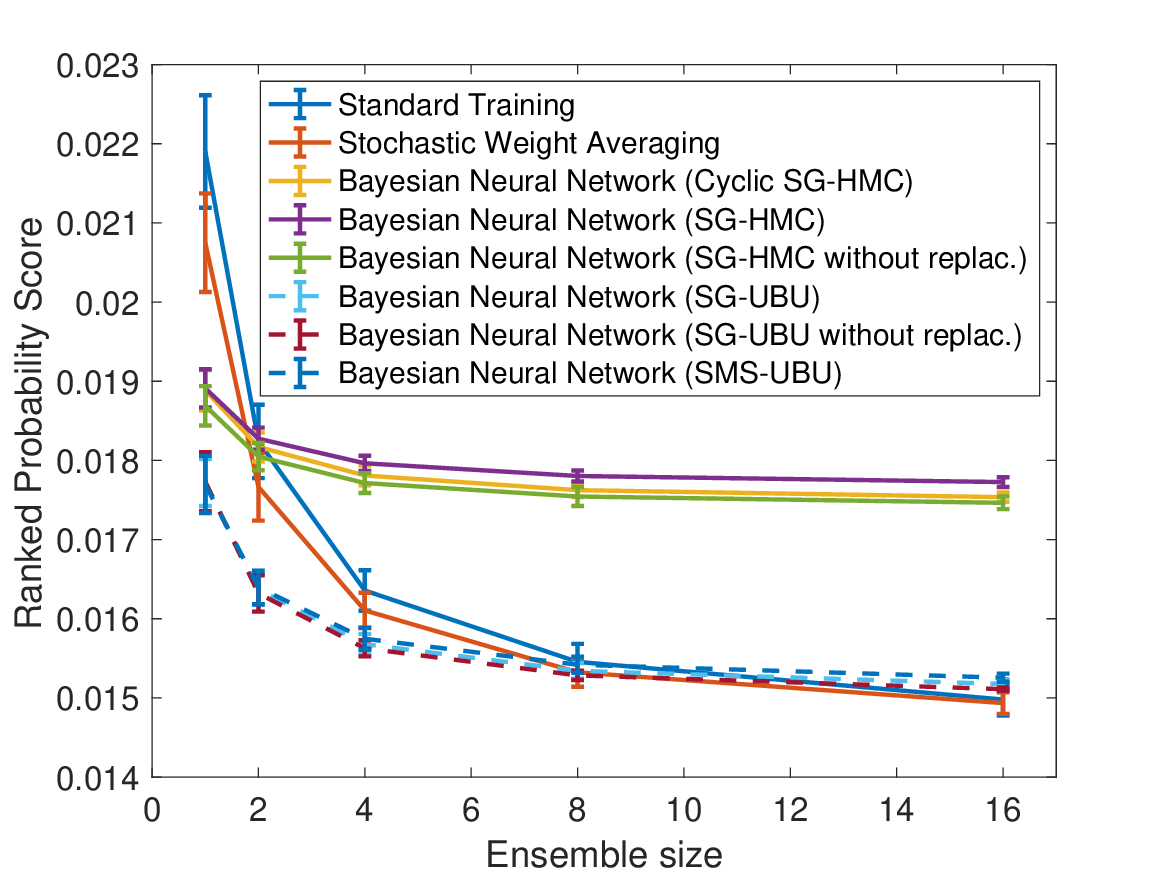}
\end{center}
\caption{Accuracy and calibration results for a CNN-based network on Fashion-MNIST.}
\label{fig:Fashion}
\end{figure}

%\subsection{Celeb-A dataset}

\begin{figure}[h!]
\includegraphics[width=0.48\linewidth]{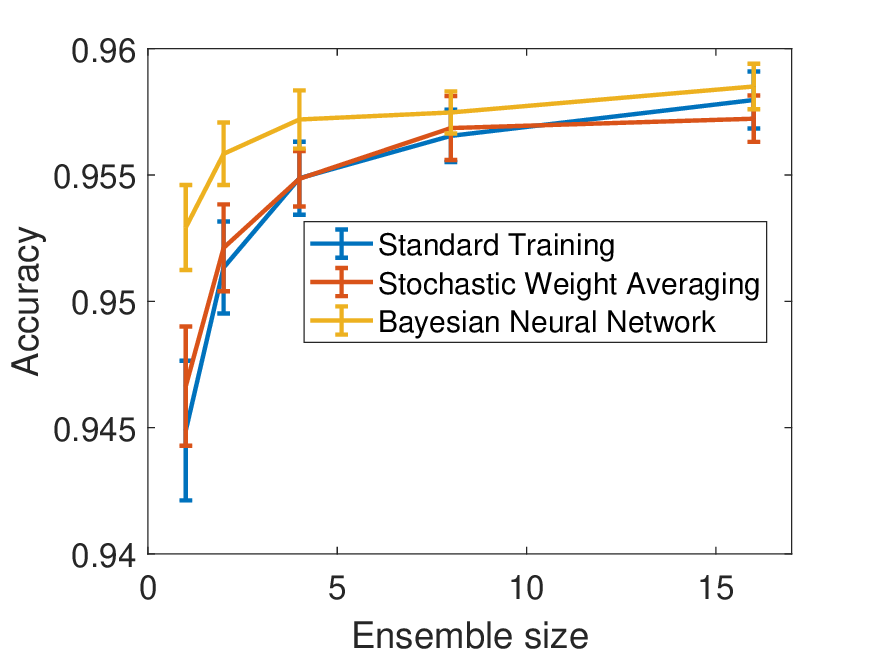}
\includegraphics[width=0.48\linewidth]{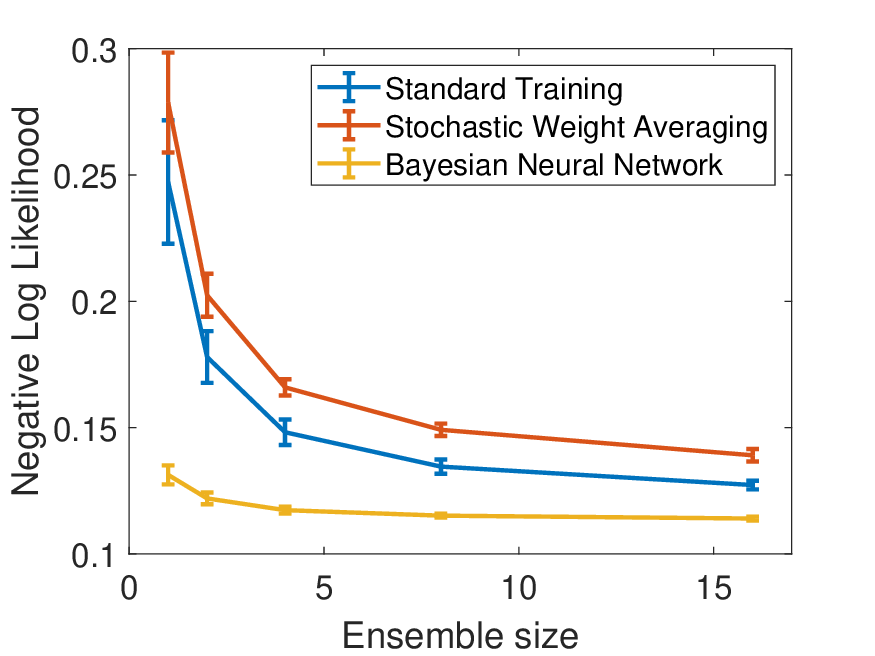}\\
\includegraphics[width=0.48\linewidth]{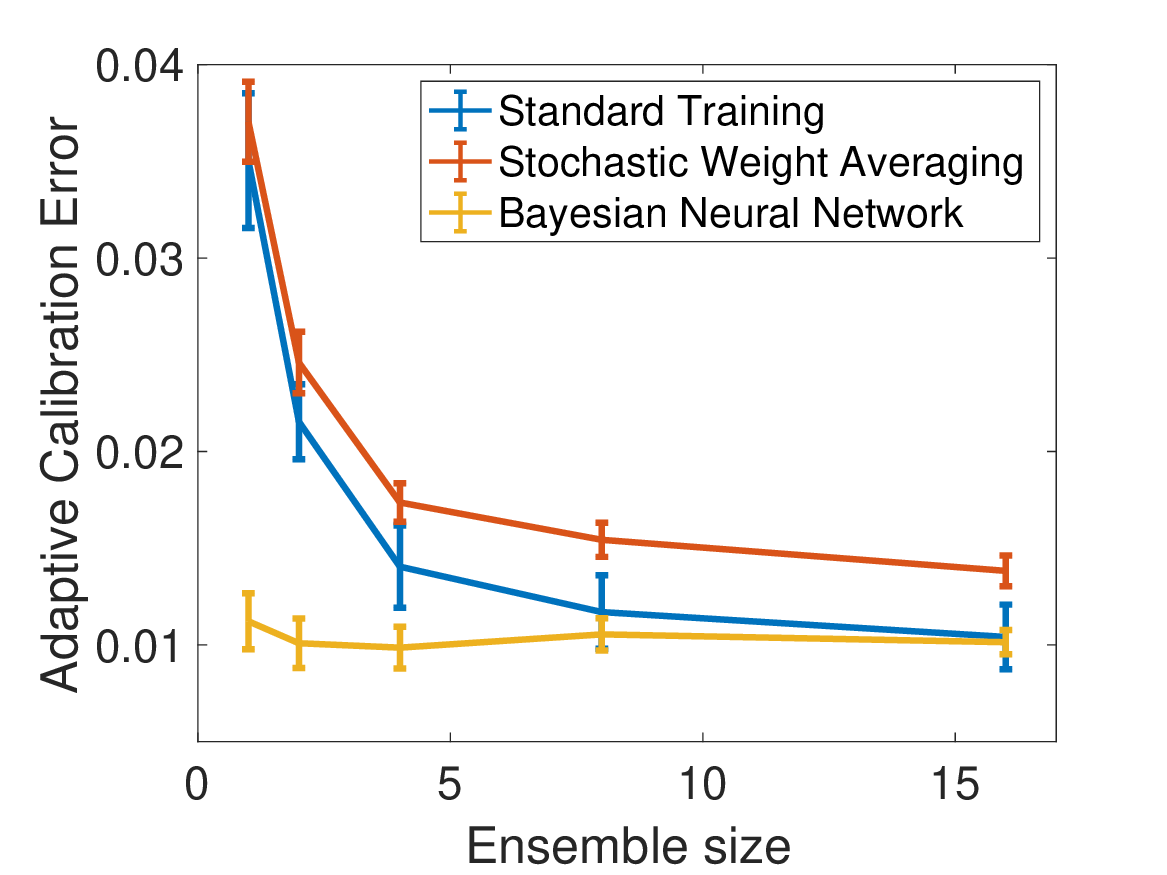}
\includegraphics[width=0.48\linewidth]{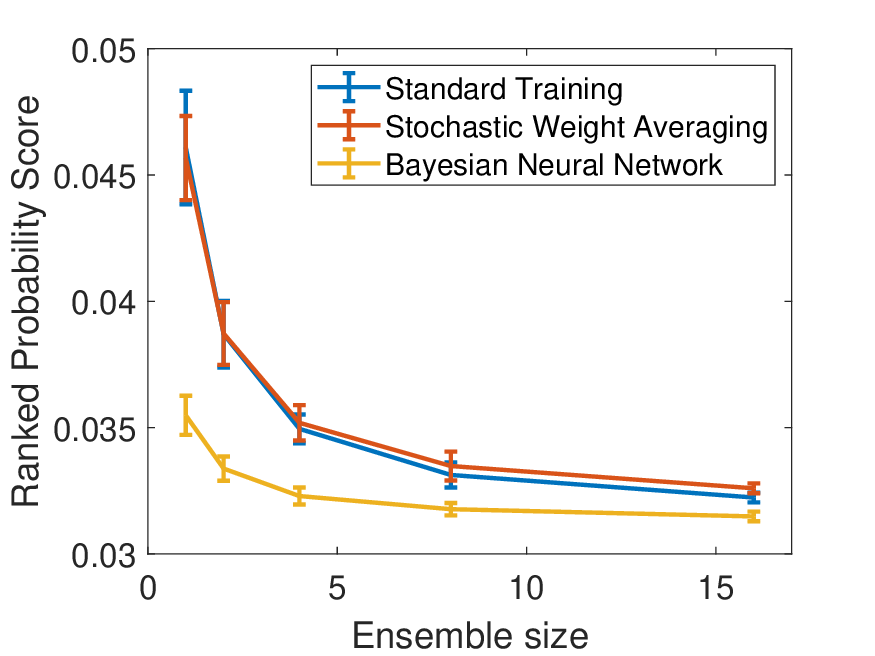}
\caption{Accuracy and calibration results for a CNN-based network for classifying brown/blonde hair colour on the Celeb-A dataset.}
\label{fig:CelebA}
\end{figure}

%\subsection{Chest X-ray dataset}

\begin{figure}[h!]
\includegraphics[width=0.48\linewidth]{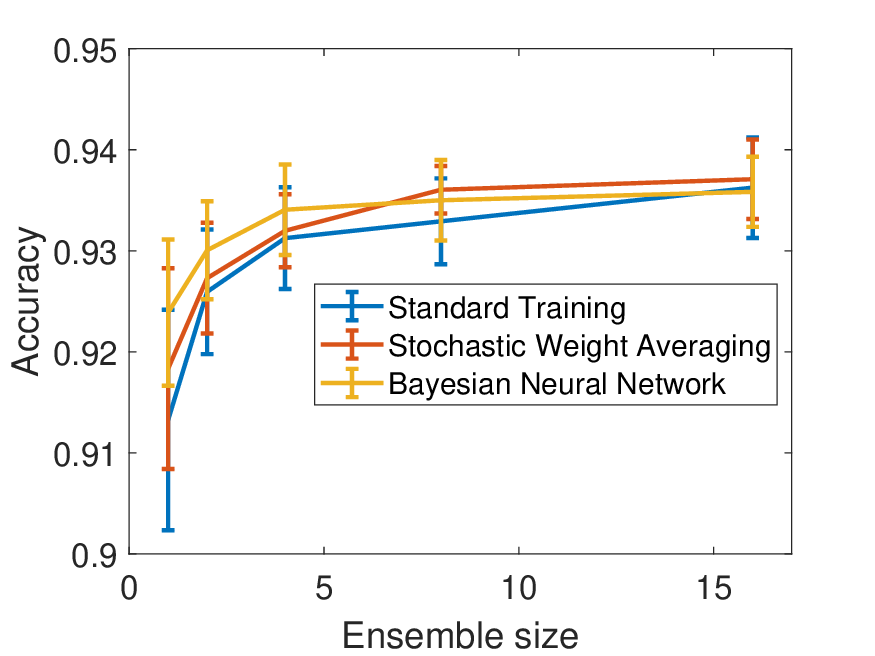}
\includegraphics[width=0.48\linewidth]{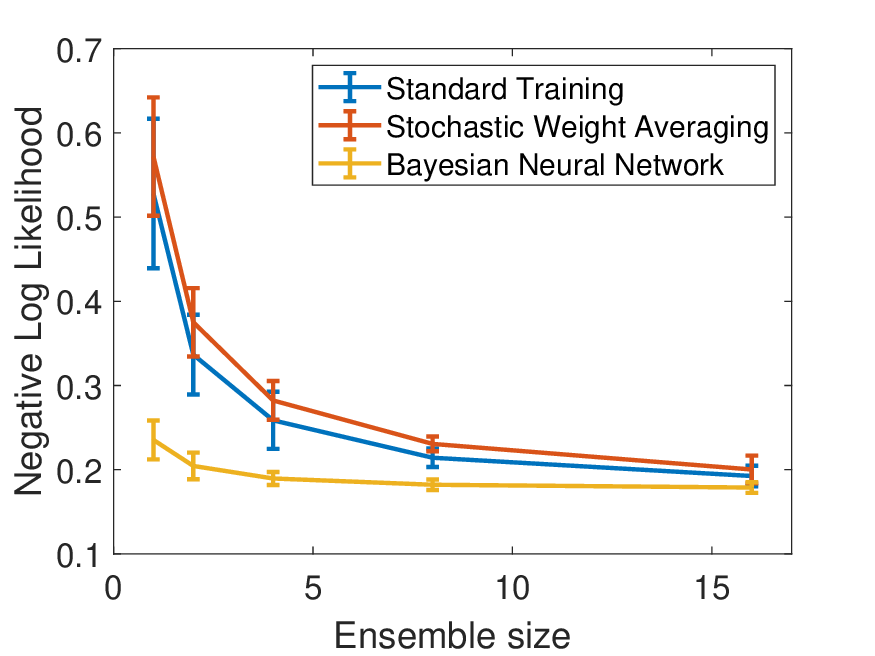}\\
\includegraphics[width=0.48\linewidth]{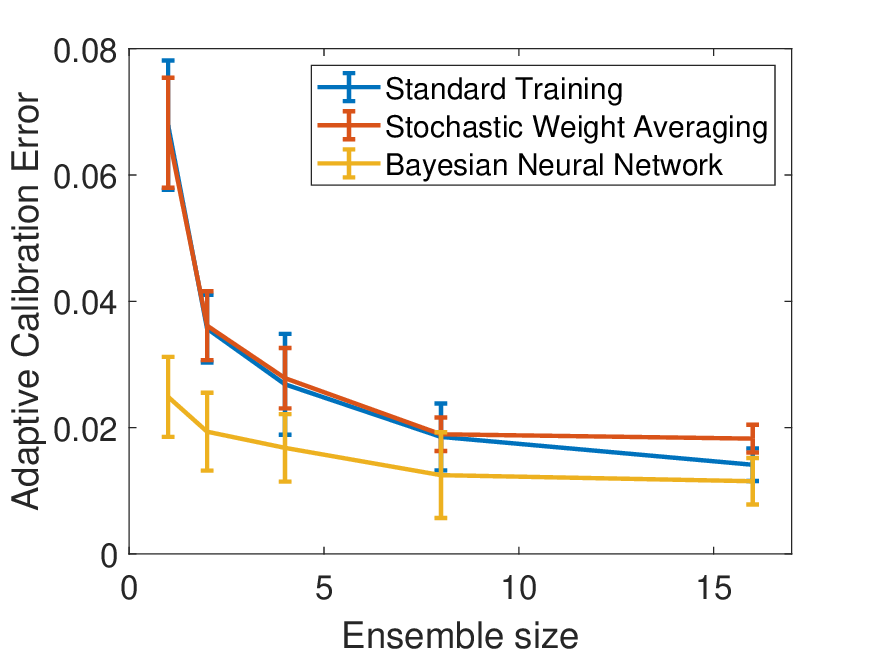}
\includegraphics[width=0.48\linewidth]{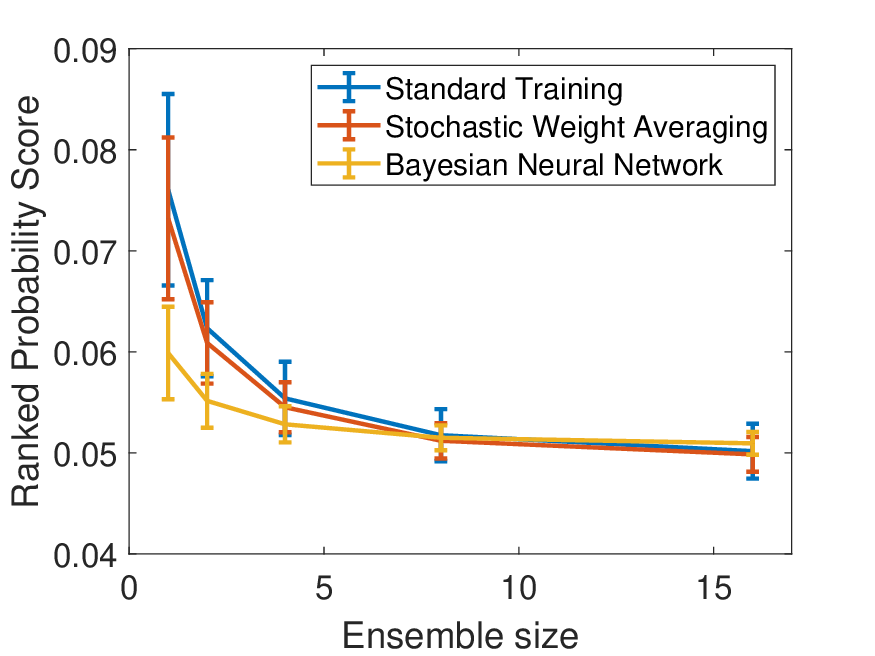}
\caption{Accuracy and calibration results for a CNN-based network for detecting pneumonia on a chest X-ray dataset.}
\label{fig:chestXray}
\end{figure}

%\ben{Numerics obviously need a lot of discussion.}

\section{Conclusion}
\label{sec:conc}

We have proposed a new sampling algorithm for Bayesian neural network posteriors, based on a kinetic Langevin integrator combined with symmetric minibatching.  We demonstrated a number of important results related to convergence of the numerical scheme, in terms of Wasserstein contraction, and also in terms of the derived estimates for weak error, establishing the accuracy to be $\mathcal{O}(h^2)$. Several numerical examples demonstrate the performance of SMS-UBU against other stochastic gradient schemes. A comparison was provided to both stochastic weight averaging and standard training (optimization). We obtained substantial improvements in calibration error.
In terms of future work, one possibility is to explore
unbiased sampling methods for BNN posteriors (see \cite{ububu} and \cite{Giles2020}), which could provide an alternative to Metropolized methods such as \cite{HMCBNN}. \textcolor{black}{A second direction could be to combine the mode connectivity framework \cite{surface,Frankle,permutation} with SMS-UBU, potentially improving efficiency}.

\subsubsection*{Acknowledgements}
D.P. and P.A.W. contributed equally to this work.
% All acknowledgments go at the end of the paper, including thanks to reviewers who gave useful comments, to colleagues who contributed to the ideas, and to funding agencies and corporate sponsors that provided financial support. 
% To preserve the anonymity, please include acknowledgments \emph{only} in the camera-ready papers.

% \subsubsection*{References}

% References follow the acknowledgements.  Use an unnumbered third level
% heading for the references section.  Please use the same font
% size for references as for the body of the paper---remember that
% references do not count against your page length total.

% \begin{thebibliography}{}
% \setlength{\itemindent}{-\leftmargin}
% \makeatletter\renewcommand{\@biblabel}[1]{}\makeatother
% \bibitem{} J.~Alspector, B.~Gupta, and R.~B.~Allen (1989).
%     \newblock Performance of a stochastic learning microchip.
%     \newblock In D. S. Touretzky (ed.),
%     \textit{Advances in Neural Information Processing Systems 1}, 748--760.
%     San Mateo, Calif.: Morgan Kaufmann.

% \bibitem{} F.~Rosenblatt (1962).
%     \newblock \textit{Principles of Neurodynamics.}
%     \newblock Washington, D.C.: Spartan Books.

% \bibitem{} G.~Tesauro (1989).
%     \newblock Neurogammon wins computer Olympiad.
%     \newblock \textit{Neural Computation} \textbf{1}(3):321--323.
% \end{thebibliography}

%\bibliographystyle{apalike}
\bibliography{sample}

%%%%%%%%%%%%%%%%%%%%%%%%%%%%%%%%%%%%%%%%%%%%%%%%%%%%%%%%%%%%
\section*{Checklist}

% % %%% BEGIN INSTRUCTIONS %%%
% The checklist follows the references. For each question, choose your answer from the three possible options: Yes, No, Not Applicable.  You are encouraged to include a justification to your answer, either by referencing the appropriate section of your paper or providing a brief inline description (1-2 sentences). 
% Please do not modify the questions.  Note that the Checklist section does not count towards the page limit. Not including the checklist in the first submission won't result in desk rejection, although in such case we will ask you to upload it during the author response period and include it in camera ready (if accepted).

% \textbf{In your paper, please delete this instructions block and only keep the Checklist section heading above along with the questions/answers below.}
% % %%% END INSTRUCTIONS %%%

\begin{enumerate}

\item For all models and algorithms presented, check if you include:
 \begin{enumerate}
   \item A clear description of the mathematical setting, assumptions, algorithm, and/or model. Yes
   \item An analysis of the properties and complexity (time, space, sample size) of any algorithm. Yes
   \item (Optional) Anonymized source code, with specification of all dependencies, including external libraries. No
 \end{enumerate}

 \item For any theoretical claim, check if you include:
 \begin{enumerate}
   \item Statements of the full set of assumptions of all theoretical results. Yes
   \item Complete proofs of all theoretical results. Yes
   \item Clear explanations of any assumptions. Yes
 \end{enumerate}

 \item For all figures and tables that present empirical results, check if you include:
 \begin{enumerate}
   \item The code, data, and instructions needed to reproduce the main experimental results (either in the supplemental material or as a URL). Yes
   \item All the training details (e.g., data splits, hyperparameters, how they were chosen). Yes
         \item A clear definition of the specific measure or statistics and error bars (e.g., with respect to the random seed after running experiments multiple times). Yes
         \item A description of the computing infrastructure used. Yes
 \end{enumerate}

 \item If you are using existing assets (e.g., code, data, models) or curating/releasing new assets, check if you include:
 \begin{enumerate}
   \item Citations of the creator If your work uses existing assets. Yes
   \item The license information of the assets, if applicable. Not Applicable
   \item New assets either in the supplemental material or as a URL, if applicable. Not Applicable
   \item Information about consent from data providers/curators. Not Applicable
   \item Discussion of sensible content if applicable, e.g., personally identifiable information or offensive content. Not Applicable
 \end{enumerate}

 \item If you used crowdsourcing or conducted research with human subjects, check if you include:
 \begin{enumerate}
   \item The full text of instructions given to participants and screenshots. Not Applicable
   \item Descriptions of potential participant risks, with links to Institutional Review Board (IRB) approvals if applicable. Not Applicable
   \item The estimated hourly wage paid to participants and the total amount spent on participant compensation. Not Applicable
 \end{enumerate}

 \end{enumerate}

\appendix

% If your paper is accepted and the title of your paper is very long,
% the style will print as headings an error message. Use the following
% command to supply a shorter title of your paper so that it can be
% used as headings.
%
%\runningtitle{I use this title instead because the last one was very long}

% If your paper is accepted and the number of authors is large, the
% style will print as headings an error message. Use the following
% command to supply a shorter version of the authors names so that
% they can be used as headings (for example, use only the surnames)
%
%\runningauthor{Surname 1, Surname 2, Surname 3, ...., Surname n}

% Supplementary material: To improve readability, you must use a single-column format for the supplementary material.
\onecolumn
\aistatstitle{Sampling from Bayesian Neural Network Posteriors with Symmetric Minibatch Splitting Langevin Dynamics: \\
Supplementary Materials}
\section{Algorithms}
\begin{algorithm}[h!]
    \footnotesize
    
	\begin{itemize}
		\item Initialize $\left(x_{0},v_{0}\right) \in \mathbb{R}^{2d}$, stepsize $h > 0$ and friction parameter $\gamma > 0$.
            % \item Sample $\omega_{0} \sim \rho$
            % \item $G_{0} \to \mathcal{G}(x_{0},\omega_{0})$ 
		\item for $k = 1,2,...,K$ do
		\begin{itemize}
                \item[] Sample $\omega_{k} \sim \rho$
                \item[] Sample $\xi_{k} \sim \mathcal{N}(0_{d},I_{d})$
                \item[] $x_{k} \to x_{k-1} + h v_{k-1}$
                \item[] $v_{k} \to v_{k-1} - h \mathcal{G}(x_{k-1},\omega_{k}) - h \gamma v_{k-1} + \sqrt{2\gamma h}\xi_{k}$

		\end{itemize}
            \item Output: Samples $(x_{k})^{K}_{k=0}$.
	\end{itemize}
            
	\caption{Stochastic Gradient Euler-Maruyama (SG-HMC)}
	\label{alg:SG-EM}
\end{algorithm}

\begin{algorithm}[h]
    \footnotesize
    
	\begin{itemize}
		\item Initialize $\left(x_{0},v_{0}\right) \in \mathbb{R}^{2d}$, stepsize $h > 0$ and friction parameter $\gamma > 0$.
            \item Sample $\omega_{1} \sim \rho$
            \item $G_{0} \to \mathcal{G}(x_{0},\omega_{1})$
		\item for $k = 1,2,...,K$ do
		\begin{itemize}
			\item[(B)] $v \to v_{k-1} - \frac{h}{2}G_{k-1}$
                \item[(A)] $x \to x_{k-1} + \frac{h}{2}v$
                \item[] Sample $\xi_{k} \sim \mathcal{N}(0_{d},I_{d})$
                \item[(O)] $v \to \eta v + \sqrt{1-\eta^{2}}\xi_{k}$
                \item[(A)] $x_{k} \to x + \frac{h}{2}v$
                \item[] Sample $\omega_{k+1} \sim \rho$
                \item[] $G_{k} \to \mathcal{G}(x_{k},\omega_{k+1})$
                \item[(B)] $v_{k} \to v - \frac{h}{2}G_{k}$
		\end{itemize}
            \item Output: Samples $(x_{k})^{K}_{k=0}$.
	\end{itemize}
	\caption{Stochastic Gradient BAOAB (SG-BAOAB)}
	\label{alg:SG-BAOAB}
\end{algorithm}

\begin{algorithm}[h!]
    \footnotesize
    \begin{algorithmic}
		\State Initialize $\left(x_{0},v_{0}\right) \in \mathbb{R}^{2d}$, stepsize $h > 0$, friction parameter $\gamma > 0$ and number of minibatches $N_{m}$.
            \For {$i = 1,2,...,\lceil K/2N_{m}\rceil$}
            \State Sample $\omega_{1},...,\omega_{N_{m}} \in [N_{D}]^{N_{b}}$ uniformly without replacement. 
            \State Define $\omega_{N_m+1}:=\omega_1$.
            \If {i=1}\State $G_{0} \to \mathcal{G}(x_{0},\omega_{1})$.
            \Else \State $G_0\to G_{N_m}$.
            \EndIf
            \State \textbf{\underline{Forward Sweep}}\\
            \For {$k = 1,2,...,\min\{N_{m},K-(2i-2)N_{m}\}$}                 
                \begin{itemize}
                \item[(B)] $v \to v_{(2i-2)N_{m} + k-1} - \frac{h}{2}G_{k-1}$
                \item[(A)] $x \to x_{(2i-2)N_{m} +k-1} + \frac{h}{2}v$
                \item[] Sample $\xi_{k} \sim \mathcal{N}(0_{d},I_{d})$
                \item[(O)] $v \to \eta v + \sqrt{1-\eta^{2}}\xi_{k}$
                \item[(A)] $x_{(2i-2)N_{m} +k} \to x + \frac{h}{2}v$
                \item[] $G_{k} \to \mathcal{G}(x_{k},\omega_{k+1})$
                \item[(B)] $v_{(2i-2)N_{m} +k} \to v - \frac{h}{2}G_{k}$
                \end{itemize}
            \EndFor\\
            \State \textbf{\underline{Backward Sweep}}\\
            \For {$k = 1,2,...,\min\{N_{m},K-(2i-1)N_{m}\}$}			
                \begin{itemize}
                \item[(B)] $v \to v_{(2i-1)N_{m}+k-1} - \frac{h}{2}G_{N_m+1-k}$
                \item[(A)] $x \to x_{(2i-1)N_{m}+k-1} + \frac{h}{2}v$
                \item[] Sample $\xi_{k} \sim \mathcal{N}(0_{d},I_{d})$
                \item[(O)] $v \to \eta v + \sqrt{1-\eta^{2}}\xi_{k}$
                \item[(A)] $x_{(2i-1)N_{m}+k} \to x + \frac{h}{2}v$
                \item[] $G_{k} \to \mathcal{G}(x_{k},\omega_{N_m+1-k})$
                \item[(B)] $v_{(2i-1)N_{m}+k} \to v - \frac{h}{2}G_{k}$                
                \end{itemize}
                % \begin{itemize}
                   % \item[(U)]
                   %  $(x_{(2i+1)N_{m} + j},v_{(2i+1)N_{m} + j}) \to (x + \frac{1-\eta^{1/2}}{\gamma}v+\tilde{Z}^{x}_{(2i+1)N_{m} + j}, \eta^{1/2} v+\tilde{Z}^{v}_{(2i+1)N_{m} + j})$
                % \end{itemize}
            \EndFor
        \EndFor
        \State Output: Samples $(x_{k})^{K}_{k=0}$.
    \end{algorithmic}
	\caption{Symmetric Minibatch Splitting BAOAB (SMS-BAOAB)}
	\label{alg:SMS-BAOAB}
\end{algorithm}

\section{Convergence of the UBU scheme}

% \begin{remark}
% The $\|A\|_{\{1,2\}\{3\}}$ norm in Definition \ref{def:strongHessianLipschitznorm} can be equivalently written as 
% \begin{equation}\label{eq:strongHessianLipschitz:alternative:def}   
% \|A\|_{\{12\}\{3\}}=\left\|\sum_{i_1} A_{i_1,\cdot,\cdot}^T\cdot A_{i_1,\cdot,\cdot}\right\|^{1/2},
% \end{equation}
% where $\|\cdot\|$ refers to the $L^2$ matrix norm, and $A_{i_1,\cdot,\cdot}=\left(A_{i_1,i_2,i_3}\right)_{1\le i_2\le d, 1\le i_3\le d}$ is a $d\times d$ matrix, see the proof of Lemma 7 of \cite{paulin2024}. 
% \end{remark}

\begin{proposition}[Proposition C.6 of \cite{ububu}]
\label{prop:Wasserstein}
Suppose that $\uu$ satisfies Assumptions \ref{assum:convex} and \ref{assum:Lip}.
Let 
$a=\frac{1}{M},\, b=\frac{1}{\gamma},\, c(h)=\frac{mh}{8\gamma}$ and $P_{h}$ denote the transition kernel for a step of UBU with stepsize $h>0$. For all $\gamma \geq \sqrt{8M}$, $h < \frac{1}{2\gamma}$, $1 \leq p \leq \infty$, $\mu,\nu \in \mathcal{P}_{p}(\R^{2d})$ and for all  $n \in \mathbb{N}$,
\[
\mathcal{W}_{p,a,b}\left(\nu P^n_{h} ,\mu P^n_{h} \right) \leq \left(1 - c(h)\right)^{n} \mathcal{W}_{p,a,b}\left(\nu,\mu\right).
\]
Further to this, $P_{h}$ has a unique invariant measure $\pi_{h}$ satisfying $\pi_{h} \in \mathcal{P}_{p}(\R^{2d})$ for all $1 \leq p \leq \infty$.
\end{proposition}

\section{Wasserstein bias bounds}
We use the following formulation of kinetic Langevin dynamics in the analysis of the UBU scheme in the full gradient setting (as in \cite{sanz2021wasserstein}) and alternative schemes with the inexact gradient splitting methods.

It is derived via It\^o's formula on the product $e^{\gamma t}V_{t}$, we then have for initial condition $(X_0,V_0) \in \R^{2d}$:
\begin{align}
V_{t} &= \mathcal{E}(t)V_{0} - \int^{t}_{0}\mathcal{E}(t-s)\nabla \uu(X_{s})ds + \sqrt{2\gamma}\int^{t}_{0}\mathcal{E}(t-s)dW_s, \label{eq:cont_v}\\
X_{t} &= X_0 + \mathcal{F}(t)V_{0} - \int^{t}_{0}\mathcal{F}(t-s)\nabla \uu(X_{s})ds + \sqrt{2\gamma}\int^{t}_{0}\mathcal{F}(t-s)dW_{s} \label{eq:cont_x},
\end{align}
where
\begin{equation}\label{eq:EFdef}
\mathcal{E}(t) = e^{-\gamma t}\qquad \mathcal{F}(t) = \frac{1-e^{-\gamma t}}{\gamma}.
\end{equation}
Then the UBU scheme (as in \cite{sanz2021wasserstein}) can be expressed as 
\begin{align}
v_{k+1} &= \mathcal{E}(h)v_k - h\mathcal{E}(h/2)\nabla \uu(y_{k}) + \sqrt{2\gamma}\int^{(k+1)h}_{kh}\mathcal{E}((k+1)h-s)dW_s,\label{eq:disc_v}\\
y_{k} &= x_{k} + \mathcal{F}(h/2)v_{k} + \sqrt{2\gamma}\int^{(k+1/2)h}_{kh}\mathcal{F}((k+1/2)h-s)dW_s, \label{eq:disc_y}\\ 
x_{k+1} &= x_{k} + \mathcal{F}(h)v_{k} - h\mathcal{F}(h/2)\nabla \uu(y_{k}) + \sqrt{2\gamma}\int^{(k+1)h}_{kh} \mathcal{F}((k+1)h-s)dW_s,\label{eq:disc_x}
\end{align}
for comparison with the true dynamics via \eqref{eq:cont_v} and \eqref{eq:cont_x}. 

We now define the stochastic gradient scheme (see Definition \ref{def:stochastic_gradient}) by
\begin{align}
\Tilde{v}_{k+1} &= \mathcal{E}(h)\Tilde{v}_k - h\mathcal{E}(h/2)\mathcal{G}(\Tilde{y}_{k},\omega_{k+1}) + \sqrt{2\gamma}\int^{(k+1)h}_{kh}\mathcal{E}((k+1)h-s)dW_s,\label{eq:disc_v_stoch}\\
\Tilde{y}_{k} &= \Tilde{x}_{k} + \mathcal{F}(h/2)\Tilde{v}_{k} + \sqrt{2\gamma}\int^{(k+1/2)h}_{kh}\mathcal{F}((k+1/2)h-s)dW_s, \label{eq:disc_y_stoch}\\ 
\Tilde{x}_{k+1} &= \Tilde{x}_{k} + \mathcal{F}(h)\Tilde{v}_{k} - h\mathcal{F}(h/2)\mathcal{G}(\Tilde{y}_{k},\omega_{k+1}) + \sqrt{2\gamma}\int^{(k+1)h}_{kh} \mathcal{F}((k+1)h-s)dW_s.\label{eq:disc_x_stoch}
\end{align}

First, considering the full gradient scheme, we have the following $L^2$ error bound to the true diffusion.

\begin{proposition}\label{prop:full_gradient}
 Assume that $h <1/2\gamma$ and $\gamma \geq \sqrt{M}$ and consider the kinetic Langevin dynamics and the UBU scheme (with full gradients and iterates $(x_{n},v_{n})_{n \in \mathbb{N}}$) with synchronously coupled Brownian motion initialized from the target measure $(x_{0},v_{0}) = (X_{0},V_{0}) \sim \pi$. Suppose that Assumption \ref{assum:Lip} and Assumption \ref{ass:hessian_lip} are satisfied, then for $k \in \mathbb{N}$ we have that for $(\Delta^{k}_{x},\Delta^{k}_{v}) := (x_{k} - X_{kh},v_{k} - V_{kh})$
    \begin{align*}
        &\left\|\left(\Delta^{k}_{x},\Delta^{k}_{v}\right)\right\|_{L^{2},a,b} \leq \frac{3\gamma^{2}}{M}e^{3\max\left\{\gamma,\frac{2M}{\gamma}\right\}hk}\left(\mathbf{C} + \frac{5h^{3}}{48}(k+1)\left(5M + \gamma M^{1/2}\right)d^{1/2}\right),
    \end{align*}
    where
    \begin{align*}
        \mathbf{C} &= \gamma^{-1}(k+1)\frac{h^{3}\sqrt{d}}{24}\left(3M_{1}\sqrt{d} + M^{3/2} + \gamma M\right) + \sqrt{2\gamma^{-1}}\sqrt{\frac{(k+1)h^{5}M^{2}d}{192}},
    \end{align*}
    and if Assumption \ref{ass:strong_hessian_lip} is satisfied this is refined to
    \begin{equation*}
        \mathbf{C} = \gamma^{-1}(k+1)\frac{h^{3}\sqrt{d}}{24}\left(3M^{s}_{1} + M^{3/2} + \gamma M\right) + \sqrt{2\gamma^{-1}}\sqrt{\frac{(k+1)h^{5}M^{2}d}{192}}.
    \end{equation*}
\end{proposition}
\begin{proof}
    The result follows using the same argument as was used to establish \cite[Lemma 13]{schuh2024convergence} with the choice of $\alpha = 5/2$ in $L^{2}$ rather than $L^{1}$ and using the equivalence of norms.
\end{proof}

\begin{lemma}\label{lemma:moment_bounds}
   Considering a potential of the form $\uu(x)=\sum^{N_{m}}_{i=1}\uu_{i}(x)$, where we assume that $\uu$ and each $\uu_{i}$ for $i = 1,...,N_{m}$ share the same minimizer $x^{*} \in \mathbb{R}^{2d}$ and each $\nabla^{2}\uu_{i} \prec M/N_{m}$ has a $M/N_{m}$-Lipschitz gradient, where $\uu$ is $M$-$\nabla$Lipschitz and $m$-strongly convex we have the following moment bound
   \[
   \int_{\mathbb{R}^{d}}\|\nabla \uu_{i}(x)\|e^{-\uu(x)}dx \leq \frac{M}{N_{m}}\sqrt{\frac{d}{m}}.
   \]
\end{lemma}
\begin{proof}
    This result follows from \cite[Lemma 30]{PM21} and the fact that $\nabla \uu_{i}$ is Lipschitz with minimizer $x^{*} \in \mathbb{R}^{d}$.
\end{proof}
\begin{remark}
    We do not need to make the assumption that all the $\uu_{i}$'s have the same minimizer, but we do this in order to simplify the estimates.
\end{remark}

\begin{lemma}\label{lem:SwithoutR}
    Let $f:\mathbb{R}^{d} \to \mathbb{R}$ of the form $f = \sum^{N_{D}}_{i=1}f_{i}$ have an unbiased estimator $\mathcal{G}(\cdot, \omega) = Nf_{\omega}$, where the random variable $\omega \in [N_{D}]^{N_{b}}$, then we consider a sequence of random variables $(\omega_{i})^{N_{m}}_{i=1}$ which are sampled uniformly without replacement, where $N_{m} = N_{D}/N_{b} \in \mathbb{N}$. $(x_{i})^{N_{m}}_{i=0}$ are a sequence of random variables which are independent of $\omega$. In this setting, we have the estimate
    \begin{equation}
    \mathbb{E}\left[\left\|\sum^{N_{m}}_{i=1}\left(\mathcal{G}(x_{i},\omega_{i}) - f(x_{i})\right)\right\|^{2}\right] \leq \frac{5}{2}\sum^{N_{m}}_{i=1}\mathbb{E}\left[\left\|\left(\mathcal{G}(x_{i},\omega_{i}) - f(x_{i})\right)\right\|^{2}\right].
    \end{equation}
\end{lemma}

\begin{proof}
    Firstly, we expand the square and we consider the cross terms which are of the form $\mathbb{E}\left[\left(f(x_{i}) - \mathcal{G}(x_{i},\omega_{i})\right)\left(f(x_{j}) - \mathcal{G}(x_{j},\omega_{j})\right)\right]$ for $i\neq j$. Let $\omega'_{j}=\omega_j$ with probability $1-\frac{1}{N_m}$, and $\omega_j'=\omega_i$ with probability $\frac{1}{N_m}$. Then it is easy to see that $\omega_j'$ is uniformly distributed on the set $\{1,\ldots, N_{m}\}$, independently of $\omega_{i}$. As a consequence
    \begin{align*}
    &\mathbb{E}\left[\left(f(x_{i}) - \mathcal{G}(x_{i},\omega_{i})\right)\left(f(x_{j}) - \mathcal{G}(x_{j},\omega_{j})\right)\right] = \mathbb{E}\left[\left(f(x_{i}) - \mathcal{G}(x_{i},\omega_{i})\right)\left(f(x_{j}) - \mathcal{G}(x_{j},\omega'_{j})\right)\right] \\
    &+ \mathbb{E}\left[\left(f(x_{i}) - \mathcal{G}(x_{i},\omega_{i})\right)\left(\left(f(x_{j}) - \mathcal{G}(x_{j},\omega_{j})\right)- \left(f(x_{j}) - \mathcal{G}(x_{j},\omega'_{j})\right)\right)\right],
    \end{align*}
    then the first term has zero expectation. The second term can be written as
    \begin{align*}
        &\mathbb{E}\left[\left(f(x_{i}) - \mathcal{G}(x_{i},\omega_{i})\right)\left(\left(f(x_{j}) - \mathcal{G}(x_{j},\omega_{j})\right)- \left(f(x_{j}) - \mathcal{G}(x_{j},\omega'_{j})\right)\right)\mathbbm{1}[\omega_{j}\neq \omega_{j'}]\right]\leq\\
        &\frac{1}{2}\mathbb{E}\left[\left(f(x_{i}) - \mathcal{G}(x_{i},\omega_{i})\right)^{2}\mathbbm{1}[\omega_{j}\neq \omega_{j'}]\right] \\
        &+ \frac{1}{2}\mathbb{E}\left[\left(\left(f(x_{j}) - \mathcal{G}(x_{j},\omega_{j})\right)\mathbbm{1}[\omega_{j}\neq \omega_{j'}]- \left(f(x_{j}) - \mathcal{G}(x_{j},\omega'_{j})\right)\mathbbm{1}[\omega_{j}\neq \omega_{j'}]\right)^{2}\right]\\
        &\leq \frac{1}{2N_{m}}\mathbb{E}\left[\left(f(x_{i}) - \mathcal{G}(x_{i},\omega_{i})\right)^{2}\right] + \frac{1}{N_{m}}\mathbb{E}\left[\left(f(x_{j}) - \mathcal{G}(x_{j},\omega_{j})\right)^{2}+ \left(f(x_{j}) - \mathcal{G}(x_{j},\omega'_{j})\right)^{2}\right],
    \end{align*}
    and summing up the terms we have the required result.
\end{proof}

\begin{proof}[Proof of Theorem \ref{theorem:SMS-UBU-large-step}]

We estimate the difference between the full-gradient UBU scheme and the stochastic gradient scheme UBU scheme as follows. We use the notation $\Delta^{k}_{x}$ and $\Delta^{k}_{v}$ to be the difference in position and velocity at iteration $k \in \mathbb{N}$ respectively. We also use synchronously coupled Brownian motion.

Using \eqref{eq:disc_v} and \eqref{eq:disc_v_stoch} we have
\begin{align*}
    \Delta^{k}_{v} &= \mathcal{E}(h)^{k}\Delta^{0}_{v} - h\mathcal{E}(h/2)\sum^{k}_{i=1}\mathcal{E}(h)^{k-(i-1)}\left(\nabla \uu(y_{i-1}) - N_{m}\nabla \uu_{{i+1}}(\Tilde{y}_{i-1})\right)\\
    &= \mathcal{E}(h)^{k}\Delta^{0}_{v} - h\mathcal{E}(h/2)\sum^{\lfloor \frac{k}{2N_{m}}\rfloor -1}_{i=0}\sum^{2N_{m}-1}_{j=0}\mathcal{E}(h)^{k-1-2iN_{m} - j}\left(\nabla \uu(y_{2N_{m}i + j}) - N_{m}\nabla \uu_{{2iN_{m} + j +1}}(\Tilde{y}_{2iN_{m} + j})\right)\\
    &- h\mathcal{E}(h/2)\sum^{k}_{i=2N_{m}\lfloor k/2N_{m} \rfloor}\mathcal{E}(h)^{k-i}\left(\nabla \uu(y_{i-1}) - N_{m}\nabla \uu_{{i+1}}(\Tilde{y}_{i-1})\right),
\end{align*}
and therefore using the fact that the $\omega_{i}$ random variables are independent between each block of size $2N_{m}$ and applying Lemma \ref{lem:SwithoutR} to the forward and backward sweeps individually we have the following $L^{2}$ estimate
\begin{align*}
    &\|\Delta^{k}_{v}\|_{L^{2}} \leq \|\Delta^{0}_{v}\|_{L^{2}} + \sum^{k-1}_{i=0} hM\|y_{i} - \Tilde{y}_{i}\|_{L^{2}} + \sqrt{5}h\mathcal{E}(h/2)\sqrt{\sum^{k-1}_{i=0}\left\|\nabla \uu(y_{i}) - N_{m}\nabla \uu_{{i+1}}(y_{i})\right\|^{2}_{L^{2}}}.
\end{align*}
Considering $x$ we similarly have 
\begin{align*}
    &\|\Delta^{k}_{x}\|_{L^{2}} \leq \|\Delta^{0}_{x}\|_{L^{2}} + \sum^{k-1}_{i=0}\left[h^{2}M\|y_{i} - \Tilde{y}_{i}\|_{L^{2}} + h\|\Delta^{i}_{v}\|_{L^{2}}\right] +\sqrt{5}h^2 \sqrt{\sum^{k-1}_{i=0}\left\|\nabla \uu(y_{i}) - N_{m}\nabla \uu_{{i+1}}(y_{i})\right\|^{2}_{L^{2}}},
\end{align*}
and therefore if they are initialized at the same point we can combine the estimates and get 
\begin{align*}
    &\|(\Delta^{k}_{x},\Delta^{k}_{v})\|_{L^{2},a,0} \leq 4h\sqrt{M}\sum^{k-1}_{i=0}\|(\Delta^{i}_{x},\Delta^{i}_{v})\|_{L^{2},a,0}+\frac{2\sqrt{5}h}{\sqrt{M}}\sqrt{\sum^{k -1}_{i=0}\left\|\nabla \uu(y_{i}) - N_{m}\nabla \uu_{{i+1}}(y_{i})\right\|^{2}_{L^{2}}}\\
    &\leq 4h\sqrt{M}\sum^{k-1}_{i=0}\|(\Delta^{i}_{x},\Delta^{i}_{v})\|_{L^{2},a,0}+2\sqrt{5}h\sqrt{M}C_{SG} \sqrt{\sum^{k -1}_{i=0}\left\|y_{i} - X_{(i+1/2)h}\right\|^{2}_{L^{2}} + k\frac{d}{m}},\\
    \intertext{where we have used that $\|x-x^*\|_{L^{2}} \leq \sqrt{d/m}$ for $x\sim \pi$. By using that $\|y_{i} - X_{(i+1/2)}\|_{L^{2}} \leq C\frac{M}{m}h\sqrt{d}$ (see \cite[Proposition H.3]{ububu}) we have }
    &\leq 4h\sqrt{M}\sum^{k-1}_{i=0}\|(\Delta^{i}_{x},\Delta^{i}_{v})\|_{L^{2},a,0}+Ch\sqrt{M}C_{SG} \sqrt{k\frac{M^{2}}{m^{2}}h^{2}d + k\frac{d}{m}}\\
    &\leq Ch\sqrt{M}C_{SG}e^{4hk\sqrt{M}}\sqrt{kd\left(\frac{M^{2}}{m^{2}}h^{2} + \frac{1}{m}\right)}.
\end{align*}
Now using an interpolation argument, which is presented in detail later in the proof of Theorem \ref{theorem:SG-UBU} with blocks of size $\lfloor 1/4h\sqrt{M}\rfloor$ we have the required estimate by combining the result with \cite[Proposition H.3]{ububu} or Proposition \ref{prop:full_gradient} and using contraction of the UBU scheme at each step with different convexity constants $(m_{i})^{N_{m}}_{i=1}$.
\end{proof}

\begin{lemma}\label{lem:non_asym_local} Under the same conditions as Theorem \ref{theorem:SMS-UBU-large-step} we have that for $i \leq 2N_{m}-1$
\[
\|\Tilde{\Delta}^{i}_{x}\|_{L^{2}} + h\|\Tilde{\Delta}^{i}_{v}\|_{L^{2}}  \leq \sqrt{2}\|\Tilde{\Delta}^{0}\|_{L^{2},a,b} + \mathcal{O}\Bigg(h^{2}N_{m}M(C_{SG}+1)e^{8hN_{m}\gamma}\sqrt{N_{m}d\left(\frac{\gamma^{4}h^{2}}{m^{2}}+\frac{1}{m}\right)}\Bigg).
\]
\end{lemma}
\begin{proof}
    Assuming that $\Tilde{\Delta}^{0} = 0$, we have that for all $i = 0,...,2N_{m}-1$ (using \cite[Proposition H.3]{ububu} or Theorem \ref{theorem:SMS-UBU-large-step}) that
\begin{align*}
\|\Tilde{\Delta}^{i}_{v}\|_{L^{2}}&\leq\mathcal{O}\left( hM(C_{SG}+1) e^{8hN_{m}\gamma}\sqrt{N_{m}d\left(\frac{\gamma^{4}h^{2}}{m^{2}}+\frac{1}{m}\right)}\right)
\end{align*}
and using  \eqref{eq:cont_x}, \eqref{eq:disc_x_stoch} and discrete Gr\"{o}nwall inequality we have
\begin{align*}
&\|\Tilde{\Delta}^{i}_{x}\|_{L^{2}} \leq \mathcal{O}\Bigg(h^{2}N_{m}\sqrt{M}C_{SG}e^{8hN_{m}\gamma}\sqrt{N_{m}d\left(\frac{\gamma^{4}h^{2}}{m^{2}}+\frac{1}{m}\right)} \\
&+ \sum^{i-1}_{j=0}\int^{h}_{0}\|\mathcal{F}(h/2)N_{m}\nabla f_{j}(\Tilde{y}_{j}) - \mathcal{F}(h-s)\nabla f(X_{s + jh})\|_{L^{2}}ds\Bigg)\\
&\leq \mathcal{O}\Bigg(h^{2}N_{m}M(C_{SG}+1)e^{8hN_{m}\gamma}\sqrt{N_{m}d\left(\frac{\gamma^{4}h^{2}}{m^{2}}+\frac{1}{m}\right)} + \sum^{i-1}_{j = 0}h^{2}M(\|\Tilde{\Delta}^{j}_{x}\|_{L^{2}} + h\|\Tilde{\Delta}^{j}_{v}\|_{L^{2}})\Bigg)\\
&\leq \mathcal{O}\Bigg(h^{2}N_{m}M(C_{SG}+1)e^{8hN_{m}\gamma}\sqrt{N_{m}d\left(\frac{\gamma^{4}h^{2}}{m^{2}}+\frac{1}{m}\right)}\Bigg).
\end{align*}
Then for $\Tilde{\Delta}^{0} \neq 0$ we can use Proposition \ref{prop:Wasserstein} and the triangle inequality to achieve the required result.

\end{proof}

\begin{proposition}\label{prop:SMS-UBU}
Suppose we have an SMS-UBU discretization with a potential that is $M$-$\nabla$ Lipschitz and $M_{1}$-Hessian Lipschitz and of the form $\uu = \sum^{N_{m}}_{i=1} \uu_{i}$, where $\nabla^{2}\uu_{i} \prec MI_{D}/N_{m}$ for all $i = 1,...,N_{m}$, $h<\min\left\{\frac{1}{2\gamma},\frac{1}{2\sqrt{M}}\right\}$, for $l \in \mathbb{N}$
$(\Tilde{\Delta}^{2lN_{m}}_{x},\Tilde{\Delta}^{2lN_{m}}_{v}) := (\Tilde{x}_{2lN_{m}} - X_{2lN_{m}h},\Tilde{v}_{2lN_{m}} - V_{2lN_{m}h})$, where the discretization and continuous dynamics have synchronously coupled Brownian motion and they are all initialized from the target measure, $(X_{0},V_{0})\sim \pi$. Then if $\uu$ is $M_{1}$-Hessian Lipschitz we have that
    \begin{align*}
    \|(\Tilde{x}_{2lN_{m}} - X_{2lN_{m}h},\Tilde{v}_{2lN_{m}} - V_{2lN_{m}h})\|_{L^{2},a,b} &\leq  e^{16h\gamma l N_{m}}C(\gamma,m,M,M_{1},N_{m},C_{SG})\left(lh^{3}d + h^{5/2}\sqrt{ld}\right),
    \end{align*}
    and if $\uu$ is $M^{s}_{1}$-strongly Hessian Lipschitz
    \begin{align*}
    \|(\Tilde{x}_{2lN_{m}} - X_{2lN_{m}h},\Tilde{v}_{2lN_{m}} - V_{2lN_{m}h})\|_{L^{2},a,b} &\leq  e^{16h\gamma l N_{m}}C(\gamma,m,M,M^{s}_{1},N_{m})\left(lh^{3} + h^{5/2}\sqrt{l}\right)\sqrt{d}.
    \end{align*}
    If we impose the stronger stepsize restriction $h<1/(2\gamma N_{m})$, then our bounds simplify to the form
    \begin{align*}
        \|(\Tilde{x}_{2lN_{m}} - X_{2lN_{m}h},\Tilde{v}_{2lN_{m}} - V_{2lN_{m}h})\|_{L^{2},a,b} &\leq e^{16h\gamma l N_{m}}C(\Tilde{\gamma},\Tilde{m},\Tilde{M},\Tilde{M}^{s}_{1},C_{SG})\left(lh^{3}N^{4}_{m} + h^{5/2}\sqrt{l}N^{13/4}_{m}\right)\sqrt{d},
    \end{align*}
    where $\Tilde{\gamma} = \gamma/\sqrt{N_{m}}, \Tilde{m} = m/N_{m}, \Tilde{M} = M/N_{m}$ and $\Tilde{M}^{s}_{1} = M^{s}_{1}/N_{m}$.
\end{proposition}
\begin{remark}
    We could instead assume that $\nabla^{2}\uu_{i} \prec M_{i}I_{D}$, for all $i = 1,...,N_{m}$, each $\uu_{i}$ having a different gradient Lipschitz constant, but to simplify notation in the argument we assume the stronger assumption $\nabla^{2}\uu_{i} \prec MI_{D}/N_{m}$.
\end{remark}
\begin{proof}
We consider the position and velocity components separately.\\
\textbf{Velocity component}\\
    Firstly, we write 
    \[
    \Tilde{v}_{2N_{m}} - V_{2N_{m}h} =  \underbrace{\Tilde{v}_{2N_{m}} -  v_{2N_{m}}}_{\textnormal{(I)}} +  \underbrace{v_{2N_{m}} - V_{2N_{m}h}}_{\textnormal{(II)}},
    \]
    and similarly, in $x$, we can bound the (II) using the full gradient scheme bounds in Proposition \ref{prop:full_gradient}. Then we consider (I), the distance to full gradient discretization. We define $(\delta^{k}_{x},\delta^{k}_{v}) := (\Tilde{x}_{k} - x_{k},\Tilde{v}_{k} - v_{k})$, for $k = 0,....,2N_{m}$, then
    \begin{align*}
        \delta^{2N_{m}}_{v} &= \mathcal{E}(h)\delta^{2N_{m}-1}_{v} - h\mathcal{E}(h/2)(N_{m}\nabla \uu_{1}(\Tilde{y}_{2N_{m}-1}) - \nabla \uu(y_{2N_{m}-1}))\\
        &= \mathcal{E}^{2N_{m}}(h)\delta^{0}_{v} -h\mathcal{E}(h/2)\Bigg[\sum^{N_{m}-1}_{i=0}\mathcal{E}(h)^{N_{m}-i-1}\left(N_{m}\nabla \uu_{N_{m} - i}(\Tilde{y}_{N_{m} + i}) - \nabla \uu(y_{N_{m}+i})\right) \\
        &+ \sum^{N_{m}-1}_{i=0}\mathcal{E}(h)^{i+ N_{m}}\left(N_{m}\nabla \uu_{N_{m} - i}(\Tilde{y}_{N_{m} - i - 1}) - \nabla \uu(y_{N_{m}-i-1})\right)\Bigg]\\
        &= \mathcal{E}^{2N_{m}}(h)\delta^{0}_{v} + (\star),
    \end{align*}
    where without loss of generality we have assumed that they are ordered $1,...,N_{m}$ and allow for random permutations of this ordering.

    We now consider
    \begin{align*}
        &(\star)= -h\sum^{N_{m}-1}_{i=0}\Bigg[C^{i}_{N_{m}-i} + N_{m}\mathcal{E}(h)^{N_{m}-i-1/2}\left(\nabla \uu_{N_{m}-i}(\Tilde{y}_{N_{m}+i}) - \nabla \uu_{N_{m}-i}(X_{(N_{m}+i+1/2)h})\right) \\
        &+  N_{m}\mathcal{E}(h)^{N_{m}+i+1/2}\left(\nabla \uu_{N_{m}-i}(\Tilde{y}_{N_{m}-i-1}) - \nabla \uu_{N_{m}-i}(X_{(N_{m}-i-1/2)h})\right) - C^{i} \\
        &- \mathcal{E}(h)^{N_{m}-i-1/2}\left(\nabla \uu(y_{N_{m}+i}) - \nabla \uu(X_{(N_{m}+i+1/2)h})\right) \\
        &- \mathcal{E}(h)^{N_{m}+i+1/2}\left(\nabla \uu(y_{N_{m}-i-1}) - \nabla \uu(X_{(N_{m}-i-1/2)h})\right)\Bigg],
    \end{align*}
where 
\[
C^{j}_{i} := \mathcal{E}(h)^{N_{m}+j+1/2}\nabla \uu_{i}(X_{(N_{m}-j-1/2)h}) - 2\mathcal{E}(h)^{N_{m}}\nabla \uu_{i}(X_{N_{m}h}) + \mathcal{E}(h)^{N_{m}-j-1/2}\nabla \uu_{i}(X_{(N_{m} + j + 1/2)h}),
\]
 and
\[
C^{j} := \mathcal{E}(h)^{N_{m}+j+1/2}\nabla \uu(X_{(N_{m}-j-1/2)h}) - 2\mathcal{E}(h)^{N_{m}}\nabla \uu(X_{N_{m}h}) + \mathcal{E}(h)^{N_{m}-j-1/2}\nabla \uu(X_{(N_{m} + j + 1/2)h}),
\]
for $i,j \in \{1,...,N_{m}\}$. 

Considering the terms excluding the $C^{j}_{i}$ and $C^{i}$, they can be bounded in $L^{2}$ by 
\begin{align*}
  &hM\sum^{N_{m}-1}_{i=0}\Bigg[\|\Tilde{y}_{N_{m} + i} - X_{(N_{m} + i + 1/2)h}\|_{L^{2}} + \|\Tilde{y}_{N_{m} - i - 1} - X_{(N_{m} - i - 1/2)h}\|_{L^{2}} \\
   &+ \|y_{N_{m} + i} - X_{(N_{m} + i + 1/2)h}\|_{L^{2}} + \|y_{N_{m} - i - 1} - X_{(N_{m} - i - 1/2)h}\|_{L^{2}} \Bigg]\\
   &\leq 2hM\sum^{N_{m}-1}_{i=0}\Bigg[\left(\|\Tilde{\Delta}^{N_{m} + i}_{x}\|_{L^{2}} + h\|\Tilde{\Delta}^{N_{m} + i}_{v}\|_{L^{2}} + h^{2}\sqrt{Md}\right)  \\
   &+\left(\|\Delta^{N_{m} -i-1}_{x}\|_{L^{2}} + h\|\Delta^{N_{m} - i-1}_{v}\|_{L^{2}} + h^{2}\sqrt{Md}\right)\Bigg].
\end{align*}
By Lemma \ref{lem:non_asym_local} (noting that in the full gradient case we can consider Lemma \ref{lem:non_asym_local} $C_{SG} = 0$ and $N_{m} = 1$) we have that the terms, excluding  $C^{j}_{i}$ and $C^{i}$,  can be bounded in $L^{2}$ by 
\begin{align*}
  &hM\sum^{N_{m}-1}_{i=0}\Bigg[\|\Tilde{y}_{N_{m} + i} - X_{(N_{m} + i + 1/2)h}\|_{L^{2}} + \|\Tilde{y}_{N_{m} - i - 1} - X_{(N_{m} - i - 1/2)h}\|_{L^{2}} \\
   &+ \|y_{N_{m} + i} - X_{(N_{m} + i + 1/2)h}\|_{L^{2}} + \|y_{N_{m} - i - 1} - X_{(N_{m} - i - 1/2)h}\|_{L^{2}} \Bigg]\\
   &\leq \mathcal{O}\Bigg( \sqrt{2}hMN_{m}\Bigg[\|\Delta^{0}\|_{L^{2},a,b} + \|\Tilde{\Delta}^{0}\|_{L^{2},a,b}\\
   &+ e^{8hN_{m}\gamma}\left(h^{2}N_{m}M(C_{SG}+1)\sqrt{N_{m}d\left(\frac{\gamma^{4}h^{2}}{m^{2}}+\frac{1}{m}\right)}+ \frac{\gamma^{2}}{M}(\mathbf{C} + h^{3}N_{m}\gamma\sqrt{M})\right)\Bigg).
\end{align*}

We can then use the It\^{o}-Taylor expansion applied to $\mathcal{G}(t) = \mathcal{E}(2N_{m}h - t)\nabla f_{i}(X_{t})$ to bound the $C^{j}_{i}$ as follows
\begin{align*}
    \mathcal{G}(N_{m}h) &= \mathcal{G}((N_{m}-j-1/2)h) + (j+1/2)h \mathcal{G}'((N_{m}-j-1/2)h)\\
    &+ \int^{(j+1/2)h}_{0}\int^{s}_{0}d(\mathcal{G}'(s' + (N_{m}-j-1/2)h))ds,
\end{align*}
and applying It\^{o}'s formula as in \cite{sanz2021wasserstein} we have 
\[
\int^{(j+1/2)h}_{0}\int^{s}_{0}d(\mathcal{G}'(s' + (N_{m}-j-1/2)h))ds = I^{i,j}_{1}(h) + I^{i,j}_{2}(h) + I^{i,j}_{3}(h) + I^{i,j}_{4}(h) + I^{i,j}_{5}(h)
\]
where 
\begin{align*}
&I^{i,j}_{1}(h) := \gamma^{2}\int^{(j+1/2)h}_{0}\int^{s}_{0}\mathcal{E}(2N_{m}-(s' + (N_{m}-j-1/2)h))\nabla \uu_{i}(X_{s' + (N_{m}-j-1/2)h})ds'ds\\
&I^{i,j}_{2}(h) := \gamma\int^{(j+1/2)h}_{0}\int^{s}_{0}\mathcal{E}(2N_{m}-(s' + (N_{m}-j-1/2)h))\nabla^{2} \uu_{i}(X_{s' + (N_{m}-j-1/2)h})V_{s' + (N_{m}-j-1/2)h}ds'ds\\
&I^{i,j}_{3}(h) := \\
&\int^{(j+1/2)h}_{0}\int^{s}_{0}\mathcal{E}(2N_{m}-(s' + (N_{m}-j-1/2)h))\nabla^{3} \uu_{i}(X_{s' + (N_{m}-j-1/2)h})[V_{s' + (N_{m}-j-1/2)h},V_{s'+(N_{m}-j-1/2)h}]ds'ds\\
&I^{i,j}_{4}(h) := -\int^{(j+1/2)h}_{0}\int^{s}_{0}\mathcal{E}(2N_{m}-(s' + (N_{m}-j-1/2)h))\nabla^{2} \uu_{i}(X_{s' + (N_{m}-j-1/2)h})\nabla \uu(X_{s' + (N_{m}-j-1/2)h})ds'ds\\
&I^{i,j}_{5}(h) :=\sqrt{2\gamma}\int^{(j+1/2)h}_{0}\int^{s}_{0}\mathcal{E}(2N_{m}-(s' + (N_{m}-j-1/2)h))\nabla^{2} \uu_{i}(X_{s' + (N_{m}-j-1/2)h})dB_{s'} ds.
\end{align*}
Then  
\[
\mathcal{G}(N_{m}-j-1/2)h) - 2\mathcal{G}(N_{m}h) + \mathcal{G}((N_{m} + j + 1/2)h) = \sum^{5}_{k=1}\left[I^{i,j}_{k}(2h) - 2I^{i,j}_{k}(h)\right],
\]
and using Lemma \ref{lemma:moment_bounds} we have 
\begin{align*}
&\|I^{i,j}_{1}(h)\|_{L^{2}} \leq \mathcal{O}\left(h^{2}N_{m}\gamma^{2}M\sqrt{\frac{d}{m}}\right),\\ &\|I^{i,j}_{2}(h)\|_{L^{2}} \leq  \mathcal{O}\left(h^{2}N_{m}\gamma M  \sqrt{d}\right),\\
&\|I^{i,j}_{3}(h)\|_{L^{2}} \leq \mathcal{O}\left(h^{2}N_{m}M_{1}d\right),\\
&\|I^{i,j}_{4}(h)\|_{L^{2}} \leq \mathcal{O}\left(h^{2}N_{m}M\sqrt{Md}\right),\\
&\|I^{i,j}_{5}(h)\|_{L^{2}} \leq \mathcal{O}\left(h^{3/2}M\sqrt{\gamma N_{m} d}\right),
\end{align*}
where we remove precise constants for readability.  We apply the same bounds for $C^{i}$ with $I^{i}_{j}$ for $j = 1,2,3,4,5$. We then need to be careful with the $I^{i,j}_{5}(h)$ terms which are lower order in $h$ individually, but we can use the fact that they have zero expectation conditional on past events to improve the global error bound as in \cite{sanz2021wasserstein}.

Combining all the previous estimates we have that 
\begin{align*}
    \delta^{2N_{m}}_{v} &= \mathcal{E}^{2N_{m}}(h)\delta^{0}_{v} + \\&\left[(\star) +h\sum^{N_{m}-1}_{i=0}\left(N_{m}I^{N_{m}-i,i}_{5}(2h) - 2N_{m}I^{N_{m}-i,i}_{5}(h)- I^{i}_{5}(2h) + 2I^{i}_{5}(h)\right)\right]\\
    &-h\sum^{N_{m}-1}_{i=0}\left(N_{m}I^{N_{m}-i,i}_{5}(2h) - 2N_{m}I^{N_{m}-i,i}_{5}(h)- I^{i}_{5}(2h) + 2I^{i}_{5}(h)\right).
\end{align*}

Then after $2lN_{m}$ steps, where $l \in \mathbb{N}$, and is initialized at the target measure we have that 
\begin{align*}
&\|\delta^{2lN_{m}}_{v}\|_{L^{2}} \leq \mathcal{O}\Bigg(hMN_{m}\sum^{l-1}_{k=0}\Bigg[\|\Delta^{2kN_{m}}\|_{L^{2},a,b} +\|\Tilde{\Delta}^{2kN_{m}}\|_{L^{2},a,b} \\
&+ h^{2}N_{m}M(C_{SG}+1)e^{8hN_{m}\gamma}\sqrt{N_{m}d\left(\frac{\gamma^{4}h^{2}}{m^{2}}+\frac{1}{m}\right)}\Bigg] \\
&+ l h^{3}N^{3}_{m}\sqrt{d}\left(\gamma^{2} \frac{M}{\sqrt{m}} + M_{1}\sqrt{d}\right) + h^{5/2}N^{5/2}_{m}M\sqrt{l\gamma d}\Bigg).
\end{align*}
where we have used that $\mathbb{E}[I^{ij_{1}}_{5}I^{ij_{2}}_{5}] = 0$ for $j_{1} \neq j_{2}$.\\
\textbf{Position component}\\
Now considering position we write 
    \[
    \Tilde{x}_{2N_{m}} - X_{2N_{m}h} =  \underbrace{\Tilde{x}_{2N_{m}} -  x_{2N_{m}}}_{\textnormal{(I)}} +  \underbrace{x_{2N_{m}} - X_{2N_{m}h}}_{\textnormal{(II)}},
    \]
    and we can bound the (II) using the full gradient scheme bounds in Proposition \ref{prop:full_gradient}. To make computations easier we consider $\delta_{x} + \gamma^{-1}\delta_{v}$ with the previous $\delta_{v}$ bounds.  Then we consider (I), the distance to full gradient discretization and we have 
    \begin{align*}
        \delta^{2N_{m}}_{x} + \gamma^{-1}\delta^{2N_{m}}_{v} &= \delta^{2N_{m}-1}_{x} + \gamma^{-1}\delta^{2N_{m}-1}_{v} - h(N_{m}\nabla \uu_{1}(\Tilde{y}_{2N_{m}-1}) - \nabla \uu(y_{2N_{m}-1}))\\
        &= \delta^{0}_{x} + \gamma^{-1}\delta^{0}_{v}  -h\gamma^{-1}\Bigg[\sum^{N_{m}-1}_{i=0}\left(N_{m}\nabla \uu_{N_{m} - i}(\Tilde{y}_{N_{m} + i}) - \nabla \uu(y_{N_{m}+i})\right) \\
        &+ \sum^{N_{m}-1}_{i=0}\left(N_{m}\nabla \uu_{N_{m} - i}(\Tilde{y}_{N_{m} - i - 1}) - \nabla \uu(y_{N_{m}-i-1})\right)\Bigg]
    \end{align*}
    we can bound this similarly to $\delta_{v}$ with the key step being using an It\^o-Taylor expansion, but for $\mathcal{G}(t) = \nabla f_{i}(X_{t})$ and we have for $l \in \mathbb{N}$, $\|\delta^{2lN_{m}}_{x} + \gamma^{-1}\delta^{2lN_{m}}_{v}\|_{L^{2}}$ can be upper bounded by
    \begin{align*}
        &\mathcal{O}\Bigg(\frac{hMN_{m}}{\gamma}\sum^{l-1}_{k=0}\Bigg[\|\Tilde{\Delta}^{k}\|_{L^{2},a,b}\\
        &+ e^{8hN_{m}\gamma}\left(h^{2}N_{m}M(C_{SG}+1)\sqrt{N_{m}d\left(\frac{\gamma^{4}h^{2}}{m^{2}}+\frac{1}{m}\right)}+ \frac{\gamma^{2}}{M}(\mathbf{C} + h^{3}N_{m}\gamma\sqrt{M})\right)\Bigg] \\
&+ \gamma^{-1}l h^{3}N^{3}_{m}\sqrt{d}\left(\gamma^{2} \frac{M}{\sqrt{m}} + M_{1}\sqrt{d}\right) + \gamma^{-1}h^{5/2}N^{5/2}_{m}M\sqrt{l\gamma d}\Bigg).
    \end{align*}
    Therefore
    \begin{align*}
        &\|\delta^{2lN_{m}}_{x}\|_{L^{2}} \leq \|\delta^{2lN_{m}}_{x} + \gamma^{-1}\delta^{2lN_{m}}_{v}\|_{L^{2}} + \gamma^{-1}\|\delta^{2lN_{m}}_{v}\|_{L^{2}}\\
        &\leq \mathcal{O}\Bigg(\frac{hMN_{m}}{\gamma}\sum^{l-1}_{k=0}\Bigg[\left(\|\Tilde{\Delta}^{2kN_{m}}\|_{L^{2},a,b} + \|\Delta^{2kN_{m}}\|_{L^{2},a,b}\right) \\
        &+ e^{8hN_{m}\gamma}\left(h^{2}N_{m}M(C_{SG}+1)\sqrt{N_{m}d\left(\frac{\gamma^{4}h^{2}}{m^{2}}+\frac{1}{m}\right)}+ \frac{\gamma^{2}}{M}(\mathbf{C} + h^{3}N_{m}\gamma\sqrt{M})\right)\Bigg] \\
&+ \gamma^{-1}l h^{3}N^{3}_{m}\sqrt{d}\left(\gamma^{2} \frac{M}{\sqrt{m}} + M_{1}\sqrt{d}\right) + \gamma^{-1}h^{5/2}N^{5/2}_{m}M\sqrt{l\gamma d}\Bigg).
    \end{align*}
    \textbf{Combining components}\\
    We now assume that $h<1/2\gamma N_{m}$
    Therefore combining terms we have that for $l  \in \mathbb{N}$ and using the equivalence of norms relation \eqref{eq:normequiv} and Proposition \ref{prop:full_gradient} we have 
    \begin{align*}
        &\|(\Tilde{\Delta}^{2lN_{m}}_{x},\Tilde{\Delta}^{2lN_{m}}_{v})\|_{L^{2},a,b} + \|(\Delta^{2lN_{m}}_{x},\Delta^{2lN_{m}}_{v})\|_{L^{2},a,b} \leq \|(\delta^{2lN_{m}}_{x},\delta^{2lN_{m}}_{v})\|_{L^{2},a,b} + 2\|(\Delta^{2lN_{m}}_{x},\Delta^{2lN_{m}}_{v})\|_{L^{2},a,b}\\
        &\leq 2\|\delta^{2lN_{m}}_{x}|_{L^{2}} + \frac{2}{\sqrt{M}}\|\delta^{2lN_{m}}_{v}\|_{L^{2}} + \mathcal{O}\left(\frac{\gamma^{2}}{M}e^{6\max\left\{\gamma,\frac{2M}{\gamma}\right\}hlN_{m}}\left(\mathbf{C} + h^{3}lN_{m}\gamma M^{1/2}d^{1/2}\right)\right)\\
        &\leq \|\delta^{2lN_{m}}_{x} + \gamma^{-1}\delta^{2lN_{m}}_{v}\|_{L^{2}} + \gamma^{-1}\|\delta^{2lN_{m}}_{v}\|_{L^{2}}\\
        &\leq \mathcal{O}\Bigg(h\sqrt{M}N_{m}\sum^{l-1}_{k=0}\left(\|\Tilde{\Delta}^{2kN_{m}}\|_{L^{2},a,b} + \|\Delta^{2kN_{m}}\|_{L^{2},a,b}\right) \\
        &+lh^{3}N^{2}_{m}M^{3/2}(C_{SG}+1) e^{8hN_{m}\gamma}\sqrt{N_{m}d\left(\frac{\gamma^{4}h^{2}}{m^{2}}+\frac{1}{m}\right)} \\
&+ \frac{lh^{3}N^{3}_{m}\sqrt{d}}{\sqrt{M}}\left(\gamma^{2} \frac{M}{\sqrt{m}} + M_{1}\sqrt{d}\right) + h^{5/2}N^{5/2}_{m}\sqrt{M}\sqrt{l\gamma d}  \\
&+ \frac{\gamma^{2}}{M}e^{8hlN_{m}\sqrt{M}}\left(\mathbf{C} + h^{3}lN_{m}\gamma M^{1/2}d^{1/2}\right)\Bigg).
\end{align*}
Similarly, if we relax the assumption to $h<1/2\gamma$ we have the required result.
\end{proof}
\begin{lemma}\label{lem:SMSUBUsweepbounds}
    Assuming that $h<\frac{1}{2\gamma}$ and $\uu$ is $M$-$\nabla$Lipschitz and is of the form $\uu = \sum^{N_{m}}_{i=1}\uu_{i}$, $\gamma \geq \sqrt{8M}$ and each $\uu_{i}$ is $m_{i}$-strongly convex for $i = 1,...,N_{m}$. Then considering the SMS-UBU scheme with iterates $(x_{i},v_{i})_{i \in \mathbb{N}}$ defined by Algorithm \ref{alg:SMS-UBU}, approximating the continuous dynamics $(X_{t},V_{t}) \in \mathbb{R}^{2d}$ both initialized from the target measure and having synchronously coupled Brownian motion. Then we have if $\uu$ is $M_{1}$-Hessian Lipschitz that
    \begin{align*}
    \|(x_{2lN_{m}} - X_{2lN_{m}h},v_{2lN_{m}} - V_{2lN_{m}h})\|_{L^{2},a,b} &\leq  C(\gamma,m,M,M_{1},N_{m},C_{SG})h^{2}d,
    \end{align*}
    and if $\uu$ is $M^{s}_{1}$-strongly Hessian Lipschitz
    \begin{align*}
    \|(x_{2lN_{m}} - X_{2lN_{m}h},v_{2lN_{m}} - V_{2lN_{m}h})\|_{L^{2},a,b} &\leq  C(\gamma,m,M,M^{s}_{1},N_{m},C_{SG})h^{2}\sqrt{d}.
    \end{align*}
    If we impose the stronger stepsize restriction $h<1/(12\gamma N_{m})$, then our bounds simplify to the form
    \begin{align*}
        \|(x_{2lN_{m}} - X_{2lN_{m}h},v_{2lN_{m}} - V_{2lN_{m}h})\|_{L^{2},a,b} &\leq C(\Tilde{\gamma},\Tilde{m},\Tilde{M},\Tilde{M}^{s}_{1},C_{SG})h^{2}N^{5/2}_{m}\sqrt{d},
    \end{align*}
    where $\Tilde{\gamma} = \gamma/\sqrt{N_{m}}, \Tilde{m} = m/N_{m}, \Tilde{M} = M/N_{m}$ and $\Tilde{M}^{s}_{1} = M^{s}_{1}/N_{m}$.
\end{lemma}
\begin{proof}
    Inspired by the interpolation argument used in \cite{LePaWh24} and also used in \cite{schuh2024convergence} we define     $(x_{2lN_{m}},v_{2lN_{m}})$ as $2lN_{m}$ steps of the SMS-UBU scheme and $(X_{2lhN_{m}},V_{2lhN_{m}})$ is defined by \eqref{eq:LD} at time $2lhN_{m} \geq 0$, where these are both initialized at $ (X_{0},V_{0}) = (x_{0},v_{0}) \sim \pi$ and have synchronously coupled Brownian motion. We further define a sequence of interpolating variants $(\mathbf{X}^{(k)}_{2lN_{m}},\mathbf{V}^{(k)}_{2lN_{m}})$ for every $k = 0,...,l$ all initialized $(\mathbf{X}^{(k)}_{0},\mathbf{V}^{(k)}_{0}) = (\mathbf{X}_{0},\mathbf{V}_{0})$, where we define $(\mathbf{X}^{(k)}_{2iN_{m}},\mathbf{V}^{(k)}_{2iN_{m}})^{k}_{i=1} := (X_{2ihN_{m}},V_{2ihN_{m}})^{k}_{i=1}$ and  $(\mathbf{X}^{(k)}_{2iN_{m}},\mathbf{V}^{(k)}_{2iN_{m}})^{l}_{i=k+1}$ by SMS-UBU steps (each step being a full forward backward sweep) and for $k = l$ we have, simply, the continuous diffusion \eqref{eq:LD}. Using Proposition \ref{prop:SMS-UBU} we split up the steps into blocks of size $\Tilde{l}$ as 
    \begin{align*}
        &\|(x_{2lN_{m}} - X_{2lN_{m}h},v_{2lN_{m}} - V_{2lN_{m}h})\|_{L^{2},a,b}\leq 
        \|(\mathbf{X}^{(\lvert l/\Tilde{l}\rvert \Tilde{l})}_{2lN_{m}} - \mathbf{X}^{(l)}_{2lN_{m}h},\mathbf{V}^{(\lvert l/\Tilde{l}\rvert \Tilde{l})}_{2lN_{m}} - \mathbf{V}^{(l)}_{2lN_{m}h})\|_{L^{2},a,b} \\
        &+ \sum^{\lfloor l/\Tilde{l}\rfloor - 1}_{j=0}\|(\mathbf{X}^{(j\Tilde{l})}_{2lN_{m}} - \mathbf{X}^{((j+1)\Tilde{l})}_{2lN_{m}h},\mathbf{V}^{(j\Tilde{l})}_{2lN_{m}} - \mathbf{V}^{((j+1)\Tilde{l})}_{2lN_{m}h})\|_{L^{2},a,b}.
    \end{align*}
    Next, using the fact that the continuous dynamics preserves the invariant measure,  we have contraction of each SMS-UBU forward-backward sweep by Proposition \ref{prop:Wasserstein} and contraction of each UBU step within each forward-backwards sweep (with different convexity constants). Then we have, by considering each term in the previous summation, 
    \begin{align*}
    &\|(\mathbf{X}^{(j\Tilde{l})}_{2lN_{m}} - \mathbf{X}^{((j+1)\Tilde{l})}_{2lN_{m}h},\mathbf{V}^{(j\Tilde{l})}_{2lN_{m}} - \mathbf{V}^{((j+1)\Tilde{l})}_{2lN_{m}h})\|_{L^{2},a,b} \leq\\
    &e^{16h\gamma \Tilde{l} N_{m}}\prod^{N_{m}}_{i=1}\left(1-\frac{hm_{i}N_{m}}{8\gamma}\right)^{2(l- (j+1)\Tilde{l})}C(\Tilde{\gamma},\Tilde{m},\Tilde{M},\Tilde{M}^{s}_{1},C_{SG})\left(lh^{3}N^{4}_{m} + h^{5/2}\sqrt{l}N^{13/4}_{m}\right)\sqrt{d}.
    \end{align*}
     Summing up the terms we get
    \begin{align*}
     &\|(x_{2lN_{m}} - X_{2lN_{m}h},v_{2lN_{m}} - V_{2lN_{m}h})\|_{L^{2},a,b} \leq \frac{e^{16h\gamma \Tilde{l} N_{m}}C(\Tilde{\gamma},\Tilde{m},\Tilde{M},\Tilde{M}^{s}_{1},C_{SG})\left(lh^{3}N^{4}_{m} + h^{5/2}\sqrt{l}N^{13/4}_{m}\right)\sqrt{d}}{1 - \prod^{N_{m}}_{i=1}\left(1-\frac{hm_{i}N_{m}}{8\gamma}\right)^{2\Tilde{l}}}\\
     &\leq \frac{e^{16h\gamma \Tilde{l} N_{m}}C(\Tilde{\gamma},\Tilde{m},\Tilde{M},\Tilde{M}^{s}_{1},C_{SG})\left(lh^{3}N^{4}_{m} + h^{5/2}\sqrt{l}N^{13/4}_{m}\right)\sqrt{d}}{1 - e^{-\sum^{N_{m}}_{i=1}\frac{hm_{i}N_{m}\Tilde{l}}{4\gamma}}}\\
     &\leq e^{16h\gamma \Tilde{l} N_{m}}C(\Tilde{\gamma},\Tilde{m},\Tilde{M},\Tilde{M}^{s}_{1},C_{SG})\left(lh^{3}N^{4}_{m} + h^{5/2}\sqrt{l}N^{13/4}_{m}\right)\sqrt{d}\left(1 + \frac{4\gamma}{\sum^{N_{m}}_{i=1}hm_{i}N_{m}\Tilde{l}}\right),
    \end{align*}
    % \textcolor{red}{to do from here} due to the fact that $h\gamma N_{m} \leq 1/2$. 
    % 
    % \begin{align*}
    % \|(x_{2lN_{m}} - X_{2lN_{m}h},v_{2lN_{m}} - V_{2lN_{m}h})\|_{L^{2},a,b} &\leq 2e^{12h\gamma\Tilde{l}N_{m}\left(h\sqrt{M}N_{m} + 6\right)}\left(\mathbf{A}_{1} + \mathbf{A}_{2}\right)\left(1 + \frac{4\gamma^{2}}{\sum^{N_{m}}_{i=1}m_{i}N_{m}}\right)\\
    % &\leq C(\gamma,m,M,M^{s}_{1},N_{m},C_{SG})h^{2}\sqrt{d},
    % \end{align*}
    % or with a $d$ dependence for potentials which are Hessian Lipschitz, but not strongly-Hessian Lipschitz.
    If we consider smaller stepsizes, i.e. $h<1/(2\gamma N_{m})$ and choose $$\Tilde{l} = \left\lceil 1/16hN_{m}\gamma\right\rceil$$ and using the fact that $\Tilde{l} \leq 1/8hN_{m}\gamma$ for the considered stepsizes. The right-hand side simplifies to 
    \begin{align*}
        \|(x_{2lN_{m}} - X_{2lN_{m}h},v_{2lN_{m}} - V_{2lN_{m}h})\|_{L^{2},a,b} \leq C(\gamma/\sqrt{N_{m}},m/N_{m},M/N_{m},M^{s}_{1}/N_{m})h^{2}N^{5/2}_{m}\sqrt{d},
    \end{align*}
    or $d$ rather than $\sqrt{d}$ if the target is Hessian Lipschitz, but not strongly Hessian Lipschitz. These estimates are uniform in $l$. If we consider the larger stepsize regime $h\gamma N_{m} \leq 1/2$ we can consider the same argument with $\Tilde{l} = \lceil 1/ h\gamma \rceil$, but with an exponential dependence on $h\gamma N_{m}$.
\end{proof}

\begin{theorem}\label{theorem:second_order_bias}
      Assuming that $h<\frac{1}{2\gamma}$ and $\uu$ is $M$-$\nabla$Lipschitz and is of the form $\uu = \sum^{N_{m}}_{i=1}\uu_{i}$, $\gamma \geq \sqrt{8M}$ and each $\uu_{i}$ is $m_{i}$-strongly convex for $i = 1,...,N_{m}$ with minimizer $x^{*}\in \mathbb{R}^{d}$. We consider a SMS-UBU scheme with iterates $(x_{k},v_{k})_{k \in \mathbb{N}}$ defined by Algorithm \ref{alg:SMS-UBU} with stepsize $h$, approximating the continuous dynamics $(X_{t},V_{t}) \in \mathbb{R}^{2d}$ both with friction parameter $\gamma$, initialized at the target measure and having synchronously coupled Brownian motion. We further assume that the stochastic gradients satisfy assumption \ref{assum:stochastic_gradient} Then we have if $\uu$ is $M_{1}$-Hessian Lipschitz we have that
    \begin{align*}
    \|x_{k} - X_{kh}\|_{L^{2}} &\leq  C(\gamma,m,M,M_{1},N_{m},C_{SG})h^{2}d,
    \end{align*}
    and if $\uu$ is $M^{s}_{1}$-strongly Hessian Lipschitz we have that
    \begin{align*}
    \|x_{k} - X_{kh}\|_{L^{2}} &\leq  C(\gamma,m,M,M^{s}_{1},N_{m},C_{SG})h^{2}\sqrt{d}.
    \end{align*}
    If we impose the stronger stepsize restriction $h<1/(12\gamma N_{m})$, then our bounds simplify to the form
    \begin{align*}
        \|x_{k} - X_{k}\|_{L^{2}} &\leq C(\Tilde{\gamma},\Tilde{m},\Tilde{M},\Tilde{M}^{s}_{1},C_{SG})h^{2}N^{5/2}_{m}\sqrt{d},
    \end{align*}
    where $\Tilde{\gamma} = \gamma/\sqrt{N_{m}}, \Tilde{m} = m/N_{m}, \Tilde{M} = M/N_{m}$ and $\Tilde{M}^{s}_{1} = M^{s}_{1}/N_{m}$, and similarly when we don't assume the potential is strongly Hessian Lipschitz.
\end{theorem}

\begin{proof}
Using Lemma \ref{lemma:moment_bounds} and defining $(\Tilde{\Delta}_{x},\Tilde{\Delta}_{v})$ to be the differences in position and velocity between SMS-UBU and the continuous diffusion, one can show that for $k \in \mathbb{N}$
    \begin{align*}
        \|\Tilde{\Delta}^{k}_{v}\|_{L^{2}} &\leq \|\Tilde{\Delta}^{k-1}_{v}\|_{L^{2}} + hM\|y_{k-1} - X_{(k-1/2)h}\|_{L^{2}} + h\sqrt{Md} + hM\sqrt{\frac{d}{m}}\\
        &\leq (1+h^{2}M)\|\Tilde{\Delta}^{k-1}_{v}\|_{L^{2}} + hM\|\Tilde{\Delta}^{k-1}_{x}\|_{L^{2}} + 3hM\sqrt{\frac{d}{m}},
    \end{align*}
    and similarly
    \begin{align*}
        \|\Tilde{\Delta}^{k}_{x}\|_{L^{2}} &\leq \|\Tilde{\Delta}^{k-1}_{x}\|_{L^{2}} + h\|\Tilde{\Delta}^{k-1}_{v}\|_{L^{2}} +  h^{2}M\|y_{k-1} - X_{(k-1/2)h}\|_{L^{2}} + h^{2}\sqrt{Md} + h^{2}M\sqrt{\frac{d}{m}}\\
        &\leq  (1 + h^{2}M)\|\Tilde{\Delta}^{k-1}_{x}\|_{L^{2}} + h(1 + h^{2}M)\|\Tilde{\Delta}^{k-1}_{v}\|_{L^{2}} + 3h^{2}M\sqrt{\frac{d}{m}}.
    \end{align*}
    We then have that 
    \begin{align*}
        &\|\Tilde{\Delta}^{k}_{x}\|_{L^{2}} + \frac{1}{\sqrt{M}}\|\Tilde{\Delta}^{k}_{v}\|_{L^{2}} \leq  (1 +  2h\sqrt{M})\left(\|\Tilde{\Delta}^{k-1}_{x}\|_{L^{2}} + \frac{1}{\sqrt{M}}\|\Tilde{\Delta}^{k-1}_{v}\|_{L^{2}}\right) + 5h\sqrt{\frac{Md}{m}}\\
        &\leq e^{2h\sqrt{M}k}\left(\|\Tilde{\Delta}^{0}_{x}\|_{L^{2}} + \frac{1}{\sqrt{M}}\|\Tilde{\Delta}^{0}_{v}\|_{L^{2}}\right) + 5e^{2h\sqrt{M}k}hk\sqrt{\frac{Md}{m}}
    \end{align*}
    and therefore $\|\Tilde{\Delta}_{v}\|_{L^{2}} \leq \sqrt{M}\left[e^{2h\sqrt{M}k}\left(\|\Tilde{\Delta}^{0}_{x}\|_{L^{2}} + \frac{1}{\sqrt{M}}\|\Tilde{\Delta}^{0}_{v}\|_{L^{2}}\right) + 5e^{2h\sqrt{M}k}hk\sqrt{\frac{Md}{m}}\right]$.

    Therefore 
    \begin{align*}
        &\|\Tilde{\Delta}^{k}_{x}\|_{L^{2}} \leq  (1 + h^{2}M)\|\Tilde{\Delta}^{k-1}_{x}\|_{L^{2}} + h(1 + h^{2}M)\sqrt{M}\left[e^{2h\sqrt{M}k}\left(\|\Tilde{\Delta}^{0}_{x}\|_{L^{2}} + \frac{1}{\sqrt{M}}\|\Tilde{\Delta}^{0}_{v}\|_{L^{2}}\right)\right]\\
        &+8h^{2}ke^{2h\sqrt{M}k}(1 + h^{2}M)M\sqrt{\frac{d}{m}}\\
        &\leq e^{h^{2}Mk}\|\Tilde{\Delta}^{0}_{x}\|_{L^{2}} + 2e^{3kh\sqrt{M}}\left[hk\sqrt{M}\left(\|\Tilde{\Delta}^{0}_{x}\|_{L^{2}} + \frac{1}{\sqrt{M}}\|\Tilde{\Delta}^{0}_{v}\|_{L^{2}}\right) +8(hk)^{2}M\sqrt{\frac{d}{m}}\right]\\
        &\leq \sqrt{2}e^{h^{2}Mk}\|(\Tilde{\Delta}^{0}_{x},\Tilde{\Delta}^{0}_{v})\|_{L^{2},a,b} + 4e^{3kh\sqrt{M}}\left[hk\sqrt{M}\|(\Tilde{\Delta}^{0}_{x},\Tilde{\Delta}^{0}_{v})\|_{L^{2},a,b} +8(hk)^{2}M\sqrt{\frac{d}{m}}\right].
    \end{align*}

    Combining these estimates with Lemma \ref{lem:SMSUBUsweepbounds} (when considering the iterates of SMS-UBU $(x_{i},v_{i})_{i \in \mathbb{N}}$ approximating \eqref{eq:LD} $(X_{t},V_{t})_{t\geq 0}$) we have for any $i \in \mathbb{N}$ that we can apply the above estimate for $k = i-\lceil\frac{i}{2N_{D}}\rceil \leq 2N_{D}$ and then apply Lemma \ref{lem:SMSUBUsweepbounds} for the remainder and we have the required result. 
\end{proof}

\begin{proof}[Proof of Theorem \ref{maintheorem}]
For $k \in \mathbb{N}$ define $(\hat{x}_{k},\hat{v}_{k})$ to be SMS-UBU and $(X_t, V_t)_{t \geq 0}$ to be \eqref{eq:LD}, both initialized at the target measure $(\hat{x}_{0},\hat{v}_{0}) := (X_{0},V_{0}) \sim \overline{\pi}$ then,
\begin{align*}
    &\|x_{k} - X_{kh}\|_{L^{2}} \leq \|x_{k} - \hat{x}_{k}\|_{L^{2}} + \|\hat{x}_{k} - X_{kh}\|_{L^{2}}\\
    &\leq \sqrt{2}\|(x_{k} - \hat{x}_{k},v_{k}-\hat{v}_{k})\|_{L^{2},a,b} + \|\hat{x}_{k} - X_{kh}\|_{L^{2}}\\
    &\leq \sqrt{2}\exp{\left(-\frac{mh}{8\gamma}\left\lfloor \frac{k}{N_{m}}\right\rfloor\right)}\|(x_{0} - X_{0},v_{0}-V_{0})\|_{L^{2},a,b} \\&+ \|\hat{x}_{k} - X_{kh}\|_{L^{2}},
\end{align*}
where the final term can be bounded by Theorem \ref{theorem:second_order_bias} under appropriate assumptions.
\end{proof}

\begin{theorem}
\label{theorem:SG-UBU}
Considering the SG-UBU scheme with friction parameter $\gamma >0$, stepsize $h>0$, Markov kernel $P$ and initial measure $\pi_{0}$ and assuming that $h<\frac{1}{2\gamma}$, $\gamma \geq \sqrt{8M}$ and the stochastic gradient satisfies Assumptions \ref{Assumption:Bounded_Variance} and \ref{assum:stochastic_gradient} with constants $C_{G}$ and $C_{SG}$ respectively. Let the potential $\uu$ be $M$-$\nabla$Lipschitz, $m$-strongly convex and of the form $f = \sum^{N_{m}}_{i=1}f_{i}$. Then at the $k$-th iteration we have the non-asymptotic bound
    \begin{align*}
    \mathcal{W}_{2,a,b}(\pi_{0} P^{k},\overline{\pi}) &\leq  \left(1-\frac{mh}{4\gamma} + \frac{5h^{2}C_{G}}{M}\right)^{k/2}\mathcal{W}_{2,a,b}(\pi_{0},\overline{\pi}) \\
    &+ C(\gamma/\sqrt{N_{m}},m/N_{m},M/N_{m},C_{G}/N^{2}_{m})\left[\frac{C_{SG}\sqrt{h}}{N^{1/4}_{m}} + h\right]\sqrt{d}.
    \end{align*}
\end{theorem}

\begin{proof}
We estimate the difference between the full-gradient UBU scheme and the stochastic gradient scheme UBU scheme initialized at the invariant measure as follows. We use the notation $\Delta^{k}_{x}$ and $\Delta^{k}_{v}$ to be the difference in position and velocity at iteration $k \in \mathbb{N}$ respectively. We also use synchronously coupled Brownian motion.

Using \eqref{eq:disc_v} and \eqref{eq:disc_v_stoch} we have
\begin{align*}
    \Delta^{k}_{v} &=  - h\mathcal{E}(h/2)\sum^{k}_{i=1}\mathcal{E}(h)^{k-(i-1)}\left(\nabla \uu(y_{i-1}) - \mathcal{G}(\Tilde{y}_{i-1},\omega_{i})\right),\\
    \Delta^{k}_{x} &= \mathcal{F}(h)\sum^{k-1}_{i=1}\Delta^{i}_{v}- h\mathcal{F}(h/2)\sum^{k}_{i=1}\left(\nabla \uu(y_{i-1}) - \mathcal{G}(\Tilde{y}_{i-1},\omega_{i})\right).
\end{align*}
Then using independence of the stochastic gradient at each iteration we have, from the fact that $\mathcal{G}$ is $M$-Lipschitz,
\begin{align*}
    \|(\Delta^{k}_{x},\Delta^{k}_{v})\|_{L^{2},a,b} &\leq 4h\sqrt{M}\sum^{k-1}_{i=1}\|(\Delta^{i}_{x},\Delta^{i}_{v})\|_{L^{2},a,b} + \frac{2h}{\sqrt{M}}\sqrt{\sum^{k}_{i=1}\left\|\nabla \uu(y_{i-1}) - \mathcal{G}(y_{i-1},\omega_{i})\right\|^{2}},\\
    \intertext{and using Assumption \ref{assum:stochastic_gradient} we have}
    &\leq 4h\sqrt{M}\sum^{k-1}_{i=1}\|(\Delta^{i}_{x},\Delta^{i}_{v})\|_{L^{2},a,b} + \frac{2h}{\sqrt{M}}\sqrt{\sum^{k}_{i=1}C^{2}_{SG}M^{2}\|y_{i-1}-x^{*}\|^{2}}.\\
    \intertext{Then using a triangle inequality, the fact that $\|x-x^{*}\|_{L^{2}} \leq \sqrt{d/m}$ for $x \sim \pi$, and $\|y_{i} - X_{(i+1/2)}\|_{L^{2}} \leq C\frac{M}{m}h\sqrt{d}$ by \cite[Proposition H.3]{ububu} we have}
    &\leq 4h\sqrt{M}\sum^{k-1}_{i=1}\|(\Delta^{i}_{x},\Delta^{i}_{v})\|_{L^{2},a,b} + \frac{Ch}{\sqrt{M}}\sqrt{kC^{2}_{SG}M^{2}\left(\frac{h^{2}M^{2}d}{m^{2}} + \frac{d}{m}\right)}\\
    &\leq \frac{Ch}{\sqrt{M}}e^{4hk\sqrt{M}}\sqrt{kC^{2}_{SG}M^{2}\left(\frac{h^{2}M^{2}d}{m^{2}} + \frac{d}{m}\right)}.
\end{align*}
Using interpolation with blocks of size $\left\lfloor \frac{1}{4h\sqrt{M}}\right\rfloor$ as we did in Lemma \ref{lem:SMSUBUsweepbounds} with the Wasserstein convergence result of \cite[Proposition B.3.10]{PeterThesis} we have
\begin{align*}
 \|(\Delta^{k}_{x},\Delta^{k}_{v})\|_{L^{2},a,b} &\leq C\frac{\gamma \sqrt{M}}{m - h\gamma C_{G}/M}\frac{\sqrt{h}}{\sqrt{M}}\sqrt{\frac{C^{2}_{SG}M^{2}}{\sqrt{M}}\left(\frac{h^{2}M^{2}d}{m^{2}} + \frac{d}{m}\right)}\\
 &\leq C(\gamma/\sqrt{N_{m}},m/N_{m},M/N_{m},C_{G}/N^{2}_{m})\frac{C_{SG}\sqrt{hd}}{N^{1/4}_{m}},
\end{align*}
but we have only bounded the difference between the two schemes and not the diffusion. To conclude the bias estimate we use a triangle inequality with the global error bound for the full-gradient scheme which is a consequence of \cite[Proposition H.3]{ububu}. We also use the contraction results of \cite[Proposition B.3.10]{PeterThesis} in combination with the preceeding bias bounds to achieve the non-asymptotic guarantees.

\end{proof}

\section{Further details on numerical experiments}\label{sec:furtherdetailsnumerics}
\subsection{ Evaluation of bias}\label{sec:biasexplanation}
We have evaluated the bias of numerical integrators by synchronously coupling the integrator at step sizes $h$ and $h/2$. For example, if the SG-HMC chain in Algorithm \ref{alg:SG-EM} at stepsize $h/2$ uses Gaussians $\xi_{2k-1}$ and $\xi_{2k}$ at steps $2k-1$ and $2k$, then the coupled chain at stepsize $h$ uses $(\xi_{2k-1}+\xi_{2k})/\sqrt{2}$ at step $k$. Such couplings allowed us to estimate biases more accurately as the variance of the differences becomes significantly smaller compared to using independent chains, while the expected value of the differences remains the same. 
The maximum stepsize below which stability issues occurred was $h_0=2\cdot 10^{-3}$.
Let $h_l=h_0\cdot 2^{-l}$ for $l\ge 0$. We have constructed couplings between chains with stepsizes $(h_0, h_1), (h_1, h_2),\ldots (h_5,h_6)$, run coupled chains for 
$400\cdot 2^l$ epochs for levels $(h_l, h_{l+1})$, discarded 20\% of the samples as burn-in, and used the rest for estimating the difference in expectations between stepsizes $h_l$ and $h_{l+1}$. We have estimated the full bias using the telescoping sum $\pi_{h_0}(g)-\pi_{h_6}(g)=(\pi_{h_0}(g)-\pi_{h_1}(g))+\ldots+ (\pi_{h_5}(g)-\pi_{h_6}(g))$, assuming that the bias $\pi_{h_6}$ is negligible. We have estimated the standard deviation of these bias estimates by breaking the total simulation run into 4 chunks of equal size, computing bias estimates based on each one separately, and then computing the standard deviations.
\subsection{Behaviour of the norm of Hessian at different points in the weight space}\label{sec:Hessiannorm}
Figure \ref{fig:Hessiannorm} shows the behaviour of the norm of the Hessian for our CNN-based neural network architecture for Fashion-MNIST at different points in the weight space. The norm can be efficiently computed using the power iteration method (\cite{mises1929praktische}) based on Hessian-vector products.
\begin{figure}[H]
\includegraphics[width=0.48\linewidth]{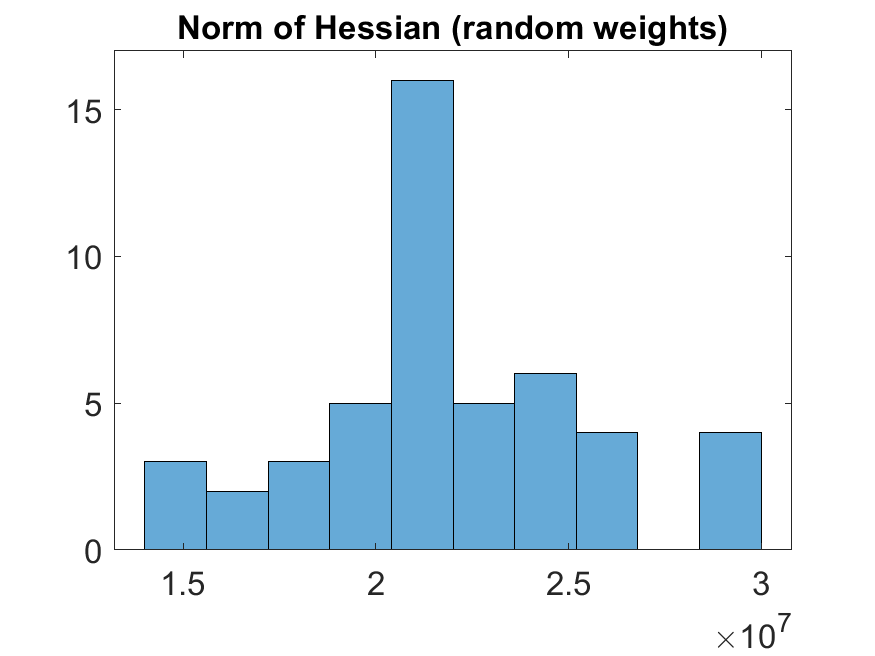}
\includegraphics[width=0.48\linewidth]{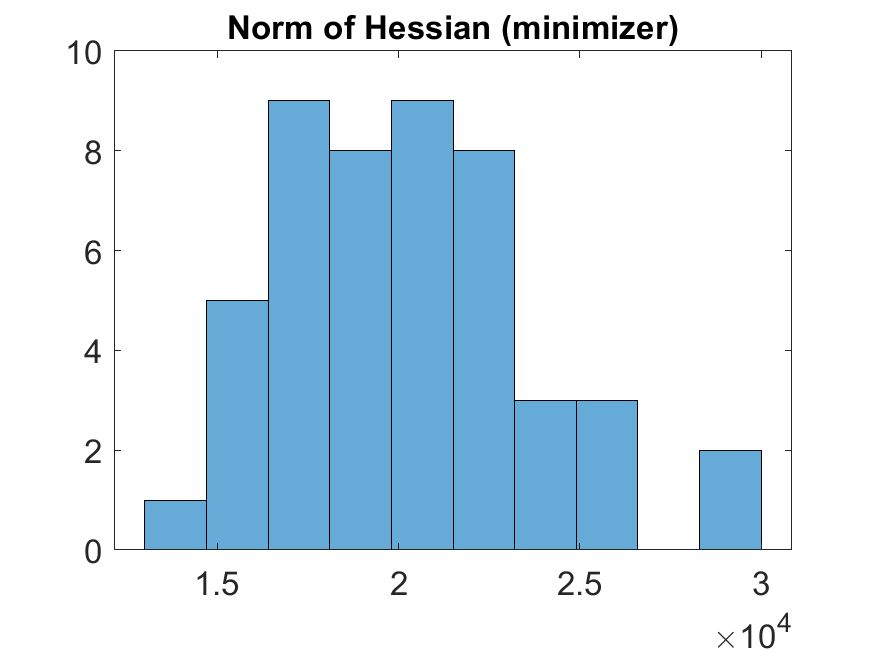}\\
\includegraphics[width=0.48\linewidth]{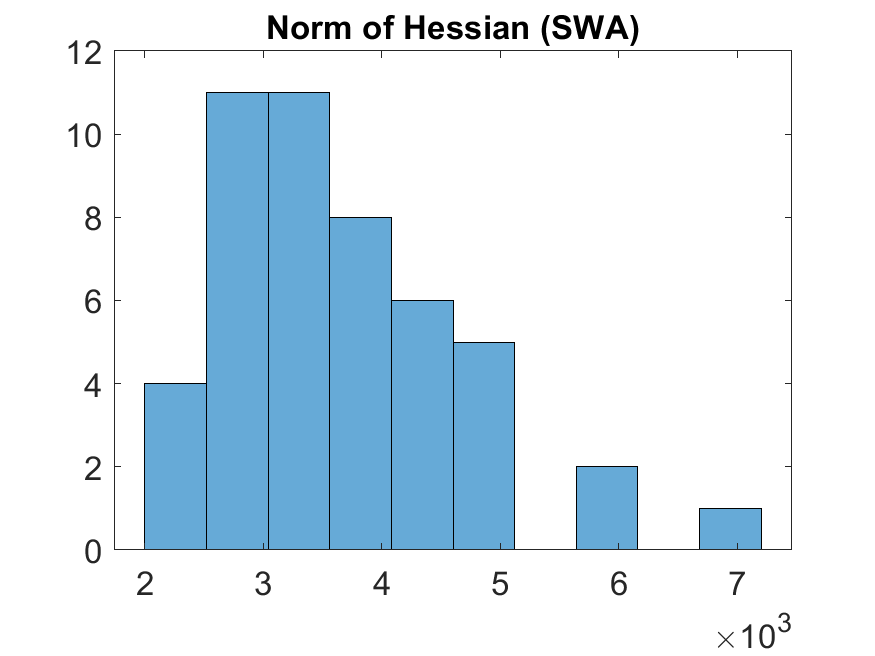}
\includegraphics[width=0.48\linewidth]{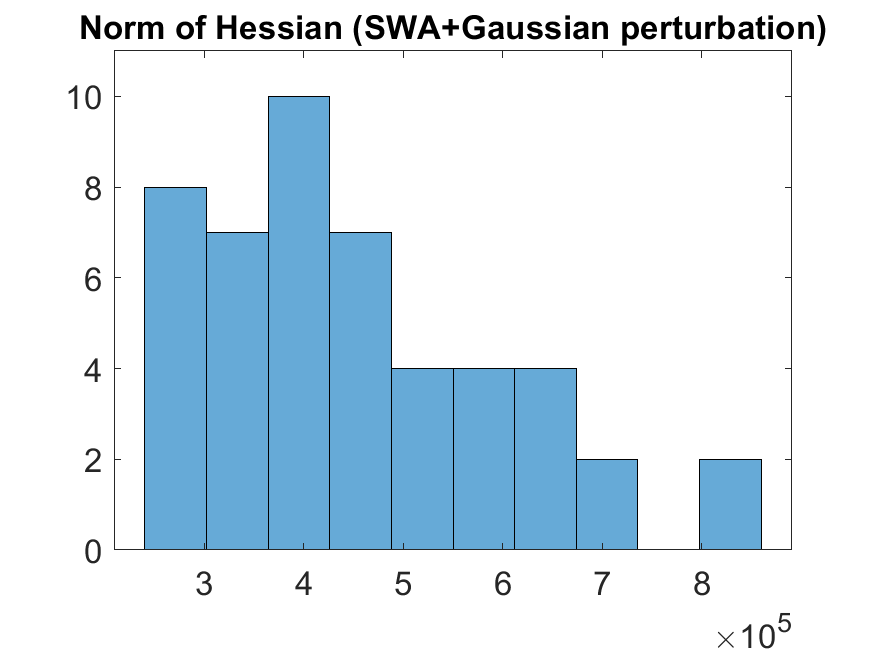}
\caption{Norm of the Hessian for a CNN-based network for Fashion-MNIST at different points in the weight space over 48 independent runs: random initialisation, minimizer after 15 epochs of training (ADAM), Stochastic Weight Averaging over 5 epochs, and Gaussian perturbation of the SWA network}
\label{fig:Hessiannorm}
\end{figure}
As we can see, the Hessian has a norm of up to $3 \cdot 10^7$ for a network with random weights, but it has much smaller norm near local minimizers, and especially at the SWA points. Adding some Gaussian noise to the SWA network does lead to a significant increase in the norm of the Hessian, but it is still small compared to using random weights. This observation, together with our theoretical understanding of convergence properties of  kinetic Langevin dynamics integrators, leads us to focus on exploring the relatively smooth area in the neighbourhood of the SWA network optima.

\section{Additional experiments for a wide neural network for CIFAR-10}
It has been argued in the literature (\cite{abe2022deep}) that the improvements due to the use deep ensembles can be replicated by using wider, larger models with more parameters. We have implemented a wide and large neural network for the CIFAR-10 dataset (\cite{krizhevsky2009learning}) (10 classes, 50000 train images, 10000 test images, resolution $32\times 32$, 3 channels). This network is also CNN-based, but it had a number of channels varying from 512 to 1024, which is 4-16 times wider than the networks we used in the other examples. In addition to increasing the width, we have also increased the number of convolutional layers from 6 to 12. The total number of parameters of this network was over 73 million. Pytorch code for this network is provided in Section \ref{sec:Pytorch.code.for.networks}.

We did not use any data augmentation in this example. Quadratic regularization with precision 1 was used. We have trained this network for 22 epochs, followed by 8 epochs of stochastic weight averaging (SWA) \cite{SWA}, and 60 epochs of SMS-UBU (5 epochs of burn-in, 55 epochs sampling). Our hyperparameter choices are stated in Table \ref{tab:expdetails3}.
\setcounter{table}{2}
\begin{table}[h!]
\caption{Choices of hyperparameters for CIFAR-10 experiment}
\centering
\begin{tabular}{|c|c|c|c|c|c|c|}
\hline
Batch size & Initial L.R. & SWA L.R. & Stepsize $h$ & $\rho$ & $\rho_{\max}$ & Friction $\gamma$\\
\hline
200 & $10^{-2}$ & $10^{-3}$ & $5 \cdot 10^{-4}$ & $50^{-1/2}$ & $6\cdot 50^{-1/2}$ & $\rho^{-1}$\\
\hline
\end{tabular}
\label{tab:expdetails3}
\end{table}

Due to the higher computational cost associated with a larger network, we only performed 4 independent runs with this network. We did not measure ensemble performance, but only the performance of a single model. Our results are stated in Table \ref{tab:cifar10results}. As we can see, our BNN implementation based on SMS-UBU improves slightly on the accuracy over SWA and standard training, and it significantly improves in calibration performance (ACE, NLL, and RPS). The accuracy result of $0.8979$ obtained using only 90 epochs of training is better than the $0.8964$ accuracy obtained using a different network trained with HMC over 60 million epochs in \cite{BNNpost}.
\begin{table}[h!]
\caption{Accuracy and calibration performance of wide neural network model for CIFAR-10, with standard deviations}
\centering
\begin{tabular}{|c|c|c|c|c|}
\hline
Training Method & Accuracy (test set) & ACE (test set) & NLL (test set) & RPS (test set)\\
\hline
ADAM & $0.8734 (\pm 0.0045)$ & $0.0150 (\pm 0.0011)$ & $0.5475(\pm 0.0352)$ & $0.0387 (\pm 0.0016)$\\
\hline
SWA & $0.8943 (\pm 0.0046)$ & $0.0138 (\pm 0.0007)$ & $0.5620 (\pm 0.0507)$ & $0.0340 (\pm 0.0016)$ \\
\hline
BNN (SMS-UBU) & $0.8979 (\pm 0.0020)$ & $0.0054 (\pm 0.0003)$ & $0.3086 (\pm 0.0061)$ & $0.0285 (\pm 0.0006)$\\
\hline
\end{tabular}
\label{tab:cifar10results}
\end{table}

\subsection{Pytorch code for networks}
\label{sec:Pytorch.code.for.networks}
Pytorch code for neural network of Fashion-MNIST example:
\begin{verbatim}
in_channels: int = 3
num_classes: int = 2
flattened_size: int = 16384
low_rank: int = 32

self.conv_layer = nn.Sequential(
    # Conv Layer block 1
    nn.Conv2d(in_channels=in_channels, out_channels=32, kernel_size=3, padding=1),
    nn.Softplus(beta=1.0),
    nn.BatchNorm2d(32,momentum=1.0),
    nn.Conv2d(in_channels=32, out_channels=64, kernel_size=3, padding=1),
    nn.Softplus(beta=1.0),
    nn.MaxPool2d(kernel_size=2, stride=2),
    nn.BatchNorm2d(64,momentum=1.0),
    # Conv Layer block 2
    nn.Conv2d(in_channels=64, out_channels=64, kernel_size=3, padding=1),
    nn.Softplus(beta=1.0),
    nn.BatchNorm2d(64,momentum=1.0),
    nn.Conv2d(in_channels=64, out_channels=128, kernel_size=3, padding=1),
    nn.Softplus(beta=1.0),
    nn.BatchNorm2d(128,momentum=1.0),
    nn.Conv2d(in_channels=128, out_channels=128, kernel_size=3, padding=1),
    nn.MaxPool2d(kernel_size=2, stride=2),
    # Conv Layer block 3
    nn.BatchNorm2d(128,momentum=1.0),
    nn.Conv2d(in_channels=128, out_channels=256, kernel_size=3, padding=1),
    nn.Softplus(beta=1.0),
    nn.BatchNorm2d(256,momentum=1.0),
    nn.Conv2d(in_channels=256, out_channels=256, kernel_size=3, padding=1),
    nn.Softplus(beta=1.0),
    nn.MaxPool2d(kernel_size=2, stride=2),
)

self.fc_layer = nn.Sequential(
    nn.Flatten() 
    nn.BatchNorm1d(flattened_size,momentum=1.0),
    nn.Linear(flattened_size, low_rank),
    nn.BatchNorm1d(low_rank,momentum=1.0),
    nn.Linear(low_rank,512),            
    nn.Softplus(beta=1.0),
    nn.BatchNorm1d(512,momentum=1.0),
)

self.last_layer=nn.Sequential(
    nn.Linear(512, num_classes)
)
\end{verbatim}

Pytorch code for neural network of hair colour (blonde/brown) classification on Celeb-A dataset:
\begin{verbatim}
in_channels: int = 3
num_classes: int = 2
flattened_size: int = 16384
low_rank: int = 16

self.conv_layer = nn.Sequential(
    # Conv Layer block 1
    nn.Conv2d(in_channels=in_channels, out_channels=32, kernel_size=3, padding=1),
    nn.Softplus(beta=1.0),
    nn.BatchNorm2d(32,momentum=1.0),
    nn.Conv2d(in_channels=32, out_channels=64, kernel_size=3, padding=1),
    nn.Softplus(beta=1.0),
    nn.MaxPool2d(kernel_size=2, stride=2),
    nn.BatchNorm2d(64,momentum=1.0),            
    # Conv Layer block 2
    nn.Conv2d(in_channels=64, out_channels=128, kernel_size=3, padding=1),
    nn.Softplus(beta=1.0),
    nn.BatchNorm2d(128,momentum=1.0),
    nn.Conv2d(in_channels=128, out_channels=128, kernel_size=3, padding=1),
    nn.Softplus(beta=1.0),
    nn.MaxPool2d(kernel_size=2, stride=2),
    nn.BatchNorm2d(128,momentum=1.0),
    # Conv Layer block 3
    nn.Conv2d(in_channels=128, out_channels=256, kernel_size=3, padding=1),
    nn.Softplus(beta=1.0),
    nn.BatchNorm2d(256,momentum=1.0),
    nn.Conv2d(in_channels=256, out_channels=256, kernel_size=3, padding=1),
    nn.Softplus(beta=1.0),            
    nn.MaxPool2d(kernel_size=2, stride=2),
)

self.fc_layer = nn.Sequential(

    nn.BatchNorm1d(flattened_size,momentum=1.0),
    nn.Linear(flattened_size, low_rank),
    nn.BatchNorm1d(low_rank,momentum=1.0),
    nn.Linear(low_rank, 4096),                        
    nn.Softplus(beta=1.0),
    nn.BatchNorm1d(4096,momentum=1.0),
    nn.Linear(4096, low_rank),
    nn.BatchNorm1d(low_rank,momentum=1.0),
    nn.Linear(low_rank, 1024),
    nn.Softplus(beta=1.0),
    nn.BatchNorm1d(1024,momentum=1.0),
    nn.Linear(1024, low_rank),
    nn.BatchNorm1d(low_rank,momentum=1.0),
    nn.Linear(low_rank, 512),            
    nn.Softplus(beta=1.0), 
    nn.BatchNorm1d(512,momentum=1.0),
)

self.last_layer=nn.Sequential(
    nn.Linear(512, num_classes)
)
\end{verbatim}

Pytorch code for neural network of Chest X-ray example:
\begin{verbatim}
in_channels: int = 1
num_classes: int = 2
flattened_size: int = 12544
low_rank: int = 32

self.conv_layer = nn.Sequential(
    nn.Conv2d(in_channels=in_channels, out_channels=32, kernel_size=3, stride=1, padding=1),
    nn.Softplus(beta=1.0),
    nn.MaxPool2d(kernel_size=2, stride=2, padding=1),
# Second Convolutional Block
    nn.BatchNorm2d(32,momentum=1.0),
    nn.Conv2d(in_channels=32, out_channels=64, kernel_size=3, stride=1, padding=1),
    nn.Softplus(beta=1.0),
    nn.MaxPool2d(kernel_size=2, stride=2, padding=1),        
# Third Convolutional Block
    nn.BatchNorm2d(64,momentum=1.0),
    nn.Conv2d(in_channels=64, out_channels=64, kernel_size=3, stride=1, padding=1),
    nn.Softplus(beta=1.0),
    nn.MaxPool2d(kernel_size=2, stride=2, padding=1),        
# Fourth Convolutional Block
    nn.BatchNorm2d(64,momentum=1.0),
    nn.Conv2d(in_channels=64, out_channels=128, kernel_size=3, stride=1, padding=1),
    nn.Softplus(beta=1.0),
    nn.MaxPool2d(kernel_size=2, stride=2, padding=1),      
# Fifth Convolutional Block
    nn.BatchNorm2d(128,momentum=1.0),
    nn.Conv2d(in_channels=128, out_channels=256, kernel_size=3, stride=1, padding=1),
    nn.Softplus(beta=1.0),
    nn.MaxPool2d(kernel_size=2, stride=2, padding=1),
)

self.fc_layer = nn.Sequential(
    nn.Flatten()
    nn.BatchNorm1d(flattened_size,momentum=1.0),
    nn.Linear(flattened_size, low_rank),
    nn.BatchNorm1d(low_rank,momentum=1.0),
    nn.Linear(low_rank, 512),            
    nn.Softplus(beta=1.0), 
    nn.BatchNorm1d(512,momentum=1.0),
)

self.last_layer=nn.Sequential(
    nn.Linear(512, num_classes)
)
\end{verbatim}

Pytorch code for wide neural network of CIFAR-10 example (73038346 parameters):
\begin{verbatim}
in_channels: int = 3
num_classes: int = 10
flattened_size: int = 16384
low_rank: int = 512

self.conv_layer = nn.Sequential(
nn.Sequential(
    nn.Conv2d(in_channels=in_channels, out_channels=512, kernel_size=3, padding=1),
    nn.Softplus(beta=1.0),
    nn.BatchNorm2d(512,momentum=1.0),
    nn.Conv2d(in_channels=512, out_channels=512, kernel_size=3, padding=1),
    nn.Softplus(beta=1.0),
    nn.BatchNorm2d(512,momentum=1.0),
    nn.Conv2d(in_channels=512, out_channels=512, kernel_size=3, padding=1),
    nn.Softplus(beta=1.0),            
    nn.BatchNorm2d(512,momentum=1.0),
    nn.Conv2d(in_channels=512, out_channels=512, kernel_size=3, padding=1),
    nn.Softplus(beta=1.0),
    nn.MaxPool2d(kernel_size=2, stride=2),
    nn.BatchNorm2d(512,momentum=1.0),


    # Conv Layer block 2
    nn.Conv2d(in_channels=512, out_channels=768, kernel_size=3, padding=1),
    nn.Softplus(beta=1.0),
    nn.BatchNorm2d(768,momentum=1.0),
    nn.Conv2d(in_channels=768, out_channels=768, kernel_size=3, padding=1),
    nn.Softplus(beta=1.0),
    nn.BatchNorm2d(768,momentum=1.0),
    nn.Conv2d(in_channels=768, out_channels=768, kernel_size=3, padding=1),
    nn.Softplus(beta=1.0),            
    nn.BatchNorm2d(768,momentum=1.0),
    nn.Conv2d(in_channels=768, out_channels=768, kernel_size=3, padding=1),
    nn.Softplus(beta=1.0),
    nn.MaxPool2d(kernel_size=2, stride=2),
    nn.BatchNorm2d(768,momentum=1.0),

    #Conv Layer block 3
    nn.Conv2d(in_channels=768, out_channels=1024, kernel_size=3, padding=1),
    nn.Softplus(beta=1.0),
    nn.BatchNorm2d(1024,momentum=1.0),
    nn.Conv2d(in_channels=1024, out_channels=1024, kernel_size=3, padding=1),
    nn.Softplus(beta=1.0),     
    nn.BatchNorm2d(1024,momentum=1.0),
    nn.Conv2d(in_channels=1024, out_channels=1024, kernel_size=3, padding=1),
    nn.Softplus(beta=1.0),
    nn.BatchNorm2d(1024,momentum=1.0),
    nn.Conv2d(in_channels=1024, out_channels=1024, kernel_size=3, padding=1),
    nn.Softplus(beta=1.0),            
    nn.MaxPool2d(kernel_size=2, stride=2),            
)

self.fc_layer = nn.Sequential(
    nn.Flatten() 
    nn.BatchNorm1d(flattened_size,momentum=1.0),
    nn.Linear(flattened_size, low_rank),
    nn.BatchNorm1d(low_rank,momentum=1.0),
    nn.Linear(low_rank,2048),            
    nn.Softplus(beta=1.0),
    nn.BatchNorm1d(2048,momentum=1.0),
    nn.Linear(2048,low_rank),            
    nn.Softplus(beta=1.0),
    nn.BatchNorm1d(low_rank,momentum=1.0),
    nn.Linear(low_rank,1024),            
    nn.Softplus(beta=1.0),
    nn.BatchNorm1d(1024),
)

self.last_layer=nn.Sequential(
    nn.Linear(1024, num_classes)
)
\end{verbatim}

\section{Comparison with Symmetric Minibatch Splitting HMC}
\label{sec:comparsionwithSMSHMC}
In \cite{HMCBNN}, a symmetric minibatch splitting version of HMC was proposed, including a Metropolis-Hastings accept/reject step. As a comparison, we have implemented Metropolised method this for the same multinomial logistic regression example of Section 3.4, using the same variance reduction scheme on the potentials, using the same batch size and dataset. We have also tried a slightly modified version of the scheme of  \cite{HMCBNN}, stated in Algorithm \ref{alg:SMS-GHMC}. This version is different from the original scheme in terms of the type of leapfrog steps used (first a half step in position update, followed by a full step in velocity update, and then another half step in position update), and the partial velocity refreshment step (as in Generalized HMC (\cite{horowitz1991generalized})).
At the same step sizes, we found that Algorithm \ref{alg:SMS-GHMC} had significantly higher acceptance rates compared to the original scheme in \cite{HMCBNN}. We did some tuning of the parameters, and obtained the best performance $h=10^{-5}$, $L=10$, and $\alpha=0.7$. With $K=1000$ iterations, the acceptance rate was $0.844$. We discarded $20\%$ of the samples as burn-in, and computed the calibration performance 
based on the remaining ones. Accuracy was $0.8420$, ACE was $0.0195$, NLL was $0.4464$, and RPS was $0.0391$.

Comparing with Figure 1 on page 7 of the submission, we can see that the calibration results are essentially the same as for SMS-UBU and the other unadjusted methods, confirming that both have very low bias. Nevertheless, the step size $h=10^{-5}$ ($\log_2 h\approx -16.61$) for SMS-GHMC is 100 times smaller than the step size $h=10^{-3}$ at which SMS-UBU was already offering essentially the same calibration performance. Hence the Metropolized method SMS-GHMC is much more computationally expensive in this example compared to the unadjusted ones based on second-order integrators.

\setcounter{algorithm}{6}
\begin{algorithm}[h!]
    \footnotesize
    \begin{algorithmic}
		\State Initialize $\left(x_{0},v_{0}\right) \in \mathbb{R}^{2d}$, stepsize $h > 0$, number of minibatches $N_{m}$, number of sweeps per accept/reject step $L$, partial refreshment parameter $\alpha\in [0,1)$, number of iterations $K$.
            \For {$k = 1,2,...,K$}
            \State Sample $\omega_{1},...,\omega_{N_{m}} \in [N_{D}]^{N_{b}}$ uniformly without replacement
            \State $(x,v)\to (x_{k-1},v_{k-1})$
            \For {$l=1,2,\ldots, L$}
            \State \textbf{\underline{Forward Sweep}}            
            \For {$b = 1,2,...,N_{m}$}                 
                \begin{itemize}                
                \item[] \hspace{0.55cm}$x \to x+\frac{h}{2} v$
                \item[] \hspace{0.55cm}$v \to v-h\mathcal{G}(x,\omega_{b})$
                \item[] \hspace{0.55cm}$x \to x+\frac{h}{2} v$               
                \end{itemize}
            \EndFor
            \State \textbf{\underline{Backward Sweep}}
            \For {$b = 1,2,...,N_{m}$}
                \begin{itemize}                
                \item[] \hspace{0.55cm}$x \to x+\frac{h}{2}v$
                \item[] \hspace{0.55cm}$v \to v-h\mathcal{G}(x,\omega_{N_m+1-b})$
                \item[] \hspace{0.55cm}$x \to x+\frac{h}{2}v$                
                \end{itemize}
            \EndFor
        \EndFor
        \State \textbf{\underline{Metropolis-Hastings Accept/Reject Step}}         
        \State Compute difference of Hamiltonians $H(x_{k-1},v_{k-1})-H(x,v)=f(x_{k-1})+\frac{\|v_{k-1}\|^2}{2}-f(x)-\frac{\|v\|^2}{2}$.
        \State Sample $U_k\sim \mathrm{Uniform}[0,1]$.
        \If{$\log(U_k)<H(x_{k-1},v_{k-1})-H(x,v)$}
        \State $(x_{k},v_{k})=(x,v)$
        \Else
        \State $(x_{k},v_{k})=(x_{k-1},-v_{k-1})$. 
        \EndIf
        \State \textbf{\underline{Velocity Refreshment Step}}            
        \State Sample $Z_k\sim \mathcal{N}(0_d,I_d)$
        \State Update $v_k\to \alpha v_k + \sqrt{1-\alpha^2} Z_k$
        \EndFor
        \State \textbf{Output}: Samples $(x_{k})^{K}_{k=0}$.
    \end{algorithmic}
	\caption{Symmetric Minibatch Splitting Generalized HMC (SMS-GHMC)}
	\label{alg:SMS-GHMC}
\end{algorithm}

\subsection{Further details for calibration experiments}\label{sec:hyperparameterchoices}
The dataset sizes for our experiments are stated below.
\begin{table}
\centering
\begin{tabular}{|c|c|c|c|c|c|}
\hline
Dataset & Training size & Test size & Resolution & Channels & \# of param.\\
\hline
Fashion-MNIST & 60000 & 10000 & $28\times 28$ & 1 & 1265258\\ 
\hline
Celeb-A& 55400 & 6073 & $64\times 64$ & 3 & 1612306\\ 
\hline
Chest X-ray & 5217 & 600 & $180\times 180$ & 1 & 870882\\ 
\hline
\end{tabular}
\caption{Details about the datasets and neural networks used in our experiments}
\label{tab:expdetails}
\end{table}
In the Fashion-MNIST and Celeb-A examples, we did not use any data augmentation. For the chest X-ray experiments, we have used random rotations, shifts, cropping, and brightness changes to augment the dataset size to 47400, and then used those images for training. 

Batch size was chosen as 200 in each case. In all datasets, quadratic regularisation with variance 1 was applied on all weights except biases. An important implementation detail is that the batch normalization layers were not using running means, but only the current batch. This was done to ensure a well-defined loss function. To ensure good performance in the prediction task, we have evaluated the network a single larger batch consisting of 20 individual batches (4000 images) before predicting with network weights set to evaluation mode. In the Bayesian setting, no temperature parameter was used, but the potential was simply chosen as the total cross-entropy loss plus the regulariser.

Our hyperparameter choices are stated in Table \ref{tab:expdetails2}.
\begin{table}[h!]
\caption{Details about the datasets and neural networks used in our experiments}
\centering
\begin{tabular}{|c|c|c|c|c|c|c|}
\hline
Dataset & Initial L.R. & SWA L.R. & Stepsize $h$ & $\rho$ & $\rho_{\max}$ & Friction $\gamma$\\
\hline
Fashion-MNIST & $10^{-2}$ & $10^{-3}$ & $2.5 \cdot 10^{-4}$ & $50^{-1/2}$ & $6\cdot 50^{-1/2}$ & $\rho^{-1}$\\ 
\hline
Celeb-A(blonde/brun.)& $10^{-2}$ & $10^{-3}$ & $2.5 \cdot 10^{-4}$ & $0.2$ & $6\cdot 0.2$ & $\rho^{-1}$\\ 
\hline
Chest X-ray & $10^{-2}$ & $10^{-3}$ & $5 \cdot 10^{-4}$ & $50^{-1/2}$ & $6\cdot 50^{-1/2}$ & $\rho^{-1}$\\ 
\hline
\end{tabular}
\label{tab:expdetails2}
\end{table}

%\section{FORMATTING INSTRUCTIONS}

%To prepare a supplementary pdf file, we ask the authors to use \texttt{aistats2024.sty} as a style file and to follow the same formatting instructions as in the main paper.
%The only difference is that the supplementary material must be in a \emph{single-column} format.
%You can use \texttt{supplement.tex} in our starter pack as a starting point, or append the supplementary content to the main paper and split the final PDF into two separate files.

%Note that reviewers are under no obligation to examine your supplementary material.

%\section{MISSING PROOFS}

%The supplementary materials may contain detailed proofs of the results that are missing in the main paper.

%\subsection{Proof of Lemma 3}

%\textit{In this section, we present the detailed proof of Lemma 3 and then [ ... ]}

%\section{ADDITIONAL EXPERIMENTS}

%If you have additional experimental results, you may include them in the supplementary materials.

%\subsection{The Effect of Regularization Parameter}

%\textit{Our algorithm depends on the regularization parameter $\lambda$. Figure 1 below illustrates the effect of this parameter on the performance of our algorithm. As we can see, [ ... ]}

\vfill

\end{document}